\documentclass[acmsmall,screen]{acmart}

\AtBeginDocument{%
  }

\setcopyright{acmlicensed}
\copyrightyear{2025}
\acmYear{2025}
\acmDOI{XXXXXXX.XXXXXXX}

\usepackage{graphicx}
\usepackage{diagbox}
\usepackage{multirow}
\usepackage{booktabs}
\usepackage{subcaption}
\usepackage{threeparttable, array, float}
\usepackage{colortbl}
\usepackage{tkz-graph}
\usepackage{tikz}
\usepackage{pgfplots}
\usetikzlibrary{calc}

\usepackage{textcomp}

\usepackage{xcolor}

\usepackage{hyperref}
\definecolor{mydarkblue}{rgb}{0,0.08,0.45}
\hypersetup{ %
    pdftitle={},
    pdfsubject={},
    pdfkeywords={},
    pdfborder=0 0 0,
    pdfpagemode=UseNone,
    colorlinks=true,
    linkcolor=mydarkblue,
    citecolor=mydarkblue,
    filecolor=mydarkblue,
    urlcolor=mydarkblue,
}

\usepackage{amsmath,amsfonts}
\usepackage{bm}
\usepackage{mathtools}
\usepackage{nicefrac}
\usepackage{physics}

\mathtoolsset{showonlyrefs,showmanualtags}

\makeatletter
\newcommand\addstarred[1]{%
    \expandafter\let\csname\string#1@nostar\endcsname#1%
    \edef#1{\noexpand\@ifstar\expandafter\noexpand\csname\string#1@star\endcsname\expandafter\noexpand\csname\string#1@nostar\endcsname}%
    \expandafter\newcommand\csname\string#1@star\endcsname%
}
\makeatother
\newcommand{\1}[1]{\mathbbm{1}{\{#1\}}}
\addstarred\1[1]{\mathbbm{1}{\left\{#1\right\}}}

\DeclarePairedDelimiter\ceil{\lceil}{\rceil}
\DeclarePairedDelimiter\floor{\lfloor}{\rfloor}

\renewcommand{\ge}{\geqslant}
\renewcommand{\le}{\leqslant}

\renewcommand{\leq}{\leqslant}

\newcommand{\upbra}[1]{^{(#1)}}

\newcommand{\type}[1]{Type-\uppercase\expandafter{\romannumeral#1}}

\usepackage{amsthm}
\newtheorem{theorem}{Theorem}
\newtheorem{lemma}[theorem]{Lemma}
\newtheorem{proposition}[theorem]{Proposition}

\newtheorem{definition}[theorem]{Definition}

\usepackage{algorithm}
\usepackage{algpseudocode}

\let\oldComment=\Comment
\renewcommand{\Comment}[1]{\oldComment{\texttt{#1}}}
\algnewcommand{\LeftComment}[1]{\Statex $\triangleright$ \texttt{#1}}
\algnewcommand{\RightComment}[1]{\Statex \leavevmode\hfill$\triangleright$ \texttt{#1}}

\algnewcommand\algorithmicinput{\textbf{Input:}}
\algnewcommand\Input{\item[\algorithmicinput]}%

\algnewcommand\algorithmicoutput{\textbf{Output:}}
\algnewcommand\Output{\item[\algorithmicoutput]}%

\algnewcommand\algorithmicinitial{\textbf{Initialize:}}
\algnewcommand\Initial{\item[\algorithmicinitial]}%

\usepackage{xspace}
\usepackage{comment}

\usepackage{fontawesome5}

\usepackage[colorinlistoftodos,prependcaption,textsize=tiny,textwidth=2cm,color=green]{todonotes}

\usepackage{blkarray}

\usepackage{tikz-cd}

\usepackage{wrapfig}
\usepackage{pdflscape}
 
\DeclareMathOperator{\ratio}{CR}

\newcommand{\overall}{\upbra{{\normalfont \text{OV}}}}
\newcommand{\statedep}{\upbra{{\normalfont \text{SD}}}}
\newcommand{\individual}{\upbra{{\normalfont \text{IND}}}}

\newcommand{\sdcr}{\ratio\upbra{\normalfont \text{SD}}}
\newcommand{\ovcr}{\ratio\upbra{\normalfont  \text{OV}}}
\newcommand{\indcr}{\ratio\upbra{\normalfont \text{IND}}}

\DeclareMathOperator{\opt}{OPT}
\DeclareMathOperator{\sdopt}{OPT\upbra{SD}}
\DeclareMathOperator{\ovopt}{OPT\upbra{OV}}

\DeclareMathOperator{\cost}{Cost}
\DeclareMathOperator{\rent}{\texttt{Rent}}
\DeclareMathOperator{\buy}{\texttt{Ind}}
\DeclareMathOperator{\group}{\texttt{Group}}
\DeclareMathOperator{\leave}{\texttt{Null}}

\newcommand{\masr}{{\normalfont \texttt{MA-SkiRental}}\xspace}

\begin{document}

\title[Competitive Algorithms for Multi-Agent Ski-Rental Problems]{Competitive Algorithms for Multi-Agent Ski-Rental Problems
}


\author{Xuchuang Wang}
\affiliation{%
  \institution{University of Massachusetts Amherst}
  \city{Amherst}
  \state{MA}
  \country{USA}}
\email{xuchuangw@gmail.com}

\author{Bo Sun}
\affiliation{%
  \institution{University of Ottawa}
  \city{Ottawa}
  \state{ON}
  \country{Canada}}
\email{bo.sun@uottawa.ca}

\author{Hedyeh Beyhaghi}
\affiliation{%
  \institution{University of Massachusetts Amherst}
  \city{Amherst}
  \state{MA}
  \country{USA}}
\email{hbeyhaghi@umass.edu}

\author{John C.S. Lui}
\affiliation{%
  \institution{The Chinese University of Hong Kong}
  \city{Hong Kong}
  \country{China}
}
\email{cslui@cse.cuhk.edu.hk}

\author{Mohammad Hajiesmaili}
\affiliation{%
  \institution{University of Massachusetts Amherst}
  \city{Amherst}
  \state{MA}
  \country{USA}}
\email{hajiesmaili@cs.umass.edu}

\author{Adam Wierman}
\affiliation{%
  \institution{California Institute of Technology}
  \city{Pasadena}
  \state{CA}
  \country{USA}}
\email{adamw@caltech.edu}









\renewcommand{\shortauthors}{}

\begin{abstract}
  This paper introduces a novel multi-agent ski-rental problem that generalizes the classical ski-rental dilemma to a group setting where agents incur individual and shared costs.
  In our model,
  each agent can either rent at a fixed daily cost,
  or purchase a pass at an individual cost,
  with an {additional third option} of a discounted \emph{group pass} available to all.
  We consider scenarios in which agents' active days differ, leading to \emph{dynamic states} as agents drop out of the decision process.
  To address this problem from different perspectives, we define three distinct competitive ratios: overall, state-dependent, and individual rational.
  For each objective, we design and analyze optimal deterministic and randomized policies.
  Our deterministic policies employ state-aware threshold functions that adapt to the dynamic states, while our randomized policies sample and resample thresholds from tailored state-aware distributions.
  The analysis reveals that symmetric policies, in which all agents use the same threshold, outperform asymmetric ones in most cases.
  Our results provide competitive ratio upper and lower bounds and extend classical ski-rental insights to multi-agent settings, highlighting both theoretical and practical implications for group decision-making under uncertainty.
\end{abstract}

\ccsdesc[500]{Theory of computation~Online algorithms}
\ccsdesc[500]{Theory of computation~Multi-agent learning}
\ccsdesc[500]{Theory of computation~Sequential decision making}

\keywords{Ski rental problem, competitive analysis, multi-agent, sequential decision making}


\maketitle

\section{Introduction}


The ski-rental problem~\citep{karlin1988competitive,karlin1994competitive} is a well-known online optimization problem~\citep{borodin2005online,shalev2012online},
where
for each active day, the skier can either rent skis for a unit cost of \(1\), or buy skis for a fixed cost of \(B\,(>1)\) and stop renting thereafter, where active days \(N\in\mathbb{N}^+\) to ski are unknown a prior.
The online algorithm needs to determine a threshold\footnote{This threshold sometimes is also known as the \emph{break-even point} or \emph{cutoff time} in the literature of ski-rental problems~\citep{karlin1988competitive}. For consistency, we use the term \emph{threshold} in this paper.} to decide when to stop renting and buy skis.
The ski-rental problem has many applications, such as online cloud computing~\citep{khanafer2013constrained}, snoopy caching~\citep{karlin1994competitive}, energy management~\citep{lee2017online,lee2021online,zhang2015peak}, etc.

However, many applications that deal with the rent or buy dilemma involve multiple agents.
For example, consider a group of researchers in an institute using a service (e.g., Overleaf) to write and revise papers, which requires a monthly payment of (short-term) subscription fees (\emph{rent}).
The service platform usually also provides a long-term (e.g., lifetime) subscription plan (\emph{buy}).
For each individual researcher, whether and when to choose the long-term subscription is a typical ski-rental problem.
Often, the service platform provides a third option, which is the institutional-level subscription (\emph{group buy}) plan, which allows all researchers in the same institute to share the subscription, and the expenses are often cheaper than the sum of individual subscriptions.
This third option of \emph{group buy}
is largely ignored in the literature of ski-rental problems.
For another example, consider a community solar project~\citep{huang2015methods,nolden2020community}:
to enjoy the benefits of renewable energy and compensate the electricity bill, a group of residents in a community can either subscribe to a community solar project (rent) or buy a share of the community solar project (buy)---a ski-rental problem for each resident.
However, the community solar programs provide a group option for the residents in the same community to buy a large share together (group buy), which is cheaper than the sum of individual shares due to the reduction of administrative and operational costs.
This group-buy option is common in many other applications as well~\citep{lu2012matching}, such as managing business cars in a company~\citep{graus2008principal},
financing electric vehicle charging station in a block~\citep{azarova2020potential}, etc.

To study algorithms that consider the group-buy option, we introduce the cooperative multi-agent ski-rental problem (\masr), consisting of \(M\in\mathbb{N}^+\) canonical ski-rental problems.
Each skier (agent) \(m\in \mathcal{M}\coloneqq \{1,2,\dots,M\}\) needs to solve a ski-rental problem with buy cost \(B\) and unknown active days \(N_m\).
Besides the individual rent and buy options, a (sub)group of \(M\) skiers can also buy a group pass with cost \(G\), which is evenly shared among them.
We develop \emph{optimal} deterministic and randomized policies from both \textit{group} and \textit{individual} perspectives.


When addressing multi-agent problems, a fundamental question arises: are \emph{symmetric} or \emph{asymmetric} policies more effective? Symmetric policies entail all agents adopting an identical strategy, whereas asymmetric policies allow each agent to employ a distinct strategy.
From a high-level perspective, symmetric policies offer advantages in scenarios involving group purchases, as they can leverage collective buying power to secure resources at a cost lower than the sum of individual purchases. Conversely, asymmetric policies, with their broader strategy space, enable agents to set varying purchasing thresholds, potentially confounding adversaries and enhancing strategic robustness.
In this paper, we demonstrate that the optimal policy for the multi-agent ski-rental problem is symmetric, except for the deterministic strategy in the special case where all agents are homogeneous, i.e., they have same active days. This finding suggests that in applications such as service subscriptions and community solar projects, participants should coordinate to make uniform decisions, regardless of the length of individual active days.


\subsection{Contributions}


In this paper, we study the multi-agent ski-rental problem (\masr) and propose optimal deterministic and randomized policies. The contributions of this paper are summarized as follows:

\(\blacktriangleright\)
In Section~\ref{sec:model-formulation}, we investigate offline optimal policies for \masr as benchmarks for our online policies.
From the group perspective, we assume all agents are \textit{cooperative} and aim to minimize the total cost of all agents.
We generalize the offline optimal policy for the canonical ski-rental problem to \masr, yielding the overall optimal offline policy.
Then, we highlight an intricate property of \masr---over the course of decision-making, the agents gradually become inactive as their active days pass, i.e., \(\{N_n\}_{n\le \ell}\) for \(\ell\) inactive agents, which alters the \emph{``state''} of the remaining active agents.
Taking the state into account, we propose the \emph{state-dependent optimal offline policy}, which can be interpreted as the overall optimal policy \emph{conditioned} on the state.
From the individual perspective, when agents are self-interested and aim to minimize their own cost, we propose the \emph{individual rational offline optimal policy} for each agent, which ensures the general-sum Nash equilibrium (NE) for the group.
According to these offline policies, we define three types of competitive ratios as objectives for \masr.




\(\blacktriangleright\) As a warm-up, we first study the homogeneous \masr in Section~\ref{sec:homogeneous-ma-ski-rental}.
The homogeneous \masr restricts the adversary to set the active days of all agents to the same value, which implies a potential advantage of the asymmetric policies (flexibility) over the symmetric ones (cheaper group buy).
Specifically, for deterministic strategies we prove that the optimal policy is whichever of the optimal symmetric or optimal asymmetric policies yields the smaller competitive ratio; which one prevails hinges on the relationship between the group-buy cost~\(G\) and the individual-buy cost~\(B\). In contrast, in the randomized setting the optimal symmetric policy always dominates its asymmetric counterpart in the homogeneous \masr.


\(\blacktriangleright\) We study deterministic policies in Section~\ref{sec:deterministic-heterogeneous}.
We show that the optimal deterministic policies should have a state-aware threshold \(T(\ell)\), a function of the state \(\{N_n\}_{n\le \ell}\) (simplified as \(\ell\), the number of inactive agents so far).
Specifically,
the policy keeps updating the threshold \(T(\ell)\) when agents become inactive and reveals their active days (see Figure~\ref{fig:algorithmic-idea}).
Within this framework, we propose two thresholds,
\(T\overall(\ell) \coloneqq \min\{ {(G - \sum_{n=1}^{\ell} N_n)} /{(M-\ell)}, B \}\)
and \(T\statedep(\ell) \coloneqq \min\left\{ {G}/{(M-\ell)}, B \right\}\),
and show that deploying \(\ceil{T\overall(\ell)}\) minimizes the overall competitive ratio, while deploying \(\ceil{T\statedep(\ell)}\) minimizes both state-dependent and individual rational competitive ratios.
The optimality of these thresholds is proved by showing that the competitive ratio upper bounds of the deterministic policies match our proposed competitive ratio lower bounds.

\(\blacktriangleright\) We study randomized policies in Section~\ref{sec:randomized-algorithm}.
We first show that, similar to the deterministic policies, randomized policies also need state-aware thresholds \(T(\ell)\), resampled from a state-dependent density function \(f_t(\ell)\) whenever the state changes.
We propose three randomized policies \(p_t\left(\ell, T\overall(\ell)\right)\), \(p_t\left(\ell, T\statedep(\ell)\right)\) and \(q_t\left(\ell, T\statedep(\ell)\right)\),
which can be interpreted as the randomized versions of the deterministic thresholds \(T\overall(\ell)\) and \(T\statedep(\ell)\), for each type of competitive ratio, respectively.
We prove that these three randomized policies are optimal by showing that the competitive ratio upper bounds of the randomized policies match our provided competitive ratio lower bounds, respectively.
We summarize the results of the deterministic and randomized policies in Table~\ref{tab:summary}.

\(\blacktriangleright\) Lastly, we reports an empirical study in Section~\ref{sec:experiments} that tests the performance of our proposed policies and validates our theoretical results.
From the worst-case competitive ratios perspective (in which sense our theoretical results hold), we show that for each type of competitive ratio,
the policy designed to minimize the particular competitive ratio indeed achieves the lowest value of that ratio among all deterministic (or, respectively, randomized) policies.
From the average competitive ratios perspective (average practical performance), we show that in a benign environment where the number of active days of agents is drawn from a stochastic distribution,
the deterministic policy optimized for the \emph{state-dependent} competitive ratio consistently outperforms every other policy, deterministic or randomized, that targets other types of competitive ratio.




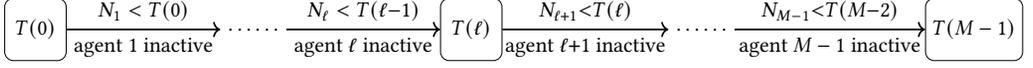
\begin{figure}[t]
    \centering
    \resizebox{0.98\linewidth}{!}{
        \begin{tikzpicture}
            \node[draw, rectangle, rounded corners, minimum width=1cm, minimum height=1cm] (rect1) {\(T(0)\)};

            \node[draw=white, rectangle, rounded corners, minimum width=0cm, minimum height=1cm, right=2.5cm of rect1] (rect2) {\(\ldots\ldots\)};

            \node[draw, rectangle, rounded corners, minimum width=1cm, minimum height=1cm, right=2.5cm of rect2] (rect3) {\(T(\ell)\)};

            \node[draw=white, rectangle, rounded corners, minimum width=0cm, minimum height=1cm, right=2.7cm of rect3] (rect4) {\(\ldots\ldots\)};

            \node[draw, rectangle, rounded corners, minimum width=1cm, minimum height=1cm, right=3.1cm of rect4] (rect5)  {\(T(M-1)\)};

            \draw[->, thick] (rect1.east) to node[above]{\(N_1 < T(0)\)}  node[below] {agent \(1\) inactive} (rect2.west);


            \draw[->, thick] (rect2.east) to node[above]{\(N_{\ell} < T({\ell{-}1})\)} node[below] {agent \(\ell\) inactive} (rect3.west);

            \draw[->, thick] (rect3.east) to node[above]{\(N_{\ell{+}1} {<} T({\ell})\)} node[below] {agent \(\ell{+}1\) inactive} (rect4.west);

            \draw[->, thick] (rect4.east) to node[above]{\(N_{M{-}1} {<} T(M{-}2)\)} node[below] {agent \(M-1\) inactive} (rect5.west);

        \end{tikzpicture}}
    \caption{Illustration for the state-aware thresholds. The sequence may terminate at any point when \(t=T(\ell)\) happens before \(t = N_{\ell+1} < T(\ell)\).}
    \label{fig:hereogeneous-deterministic-policy}
    \label{fig:algorithmic-idea}
\end{figure}

\begin{table}[t]
    \centering
    \caption{
        Summary of policies: distributions \(p_t(\cdot,\cdot)\) and \(q_t(\cdot,\cdot)\) are defined in~\eqref{eq:randomized-p-t} and~\eqref{eq:randomized-q-t}; denote \(T\upbra{\text{OV}}(\ell) = \min\left\{{\frac{G - \sum_{m=1}^\ell N_m}{M-\ell}}, B\right\} \), \(T\upbra{\text{SD}}(\ell) = \min\left\{{\frac{G}{M-\ell}}, B\right\}\) with \(\ell\) as the number of inactive agents. The detailed competitive ratios are deferred to Table~\ref{tab:summary-ratio} in Appendix.}\label{tab:summary}
    \begin{tabular}{|r|c|c|}
        \hline
         & Deterministic Policy (\S\ref{sec:deterministic-heterogeneous})
         & Randomized Policy (\S\ref{sec:randomized-algorithm})
        \\\hline
        Overall \(\ratio\)
         & \(\ceil*{T\upbra{\text{OV}}(\ell)}\) (Thm.~\ref{thm:deterministic-heterogeneous-policy-overall})
         & \(p_t \left(\ell, T\upbra{\text{OV}}(\ell)\right)\) (Thm.~\ref{thm:randomized-heterogeneous-policy-overall})
        \\
        \cline{1-3}
        State-dependent \(\ratio\)
         & \(\ceil*{T\upbra{\text{SD}}(\ell)}\) (Thm.~\ref{thm:deterministic-heterogeneous-policy-state-dependent})
         & \(p_t \left(\ell, T\upbra{\text{SD}}(\ell)\right)\) (Thm.~\ref{thm:randomized-heterogeneous-policy-state-dependent})
        \\
        \cline{1-3}
        Individual rational \(\ratio\)
         & \(\ceil*{T\upbra{\text{SD}}(\ell)}\) (Thm.~\ref{thm:deterministic-heterogeneous-policy-individual-rational})
         & \(q_t \left(\ell, T\upbra{\text{SD}}(\ell)\right)\) (Thm.~\ref{thm:randomized-heterogeneous-policy-individual-rational})
        \\
        \hline
    \end{tabular}
\end{table}

{\textbf{Overview of technical challenges.}}
The optimal policies' design and analysis for \masr consist of two key steps:
(i) prove that the optimal policies are symmetric, i.e., all agents sharing the same strategy is the best option, and
(ii) given a state, devise the corresponding best symmetric policy.
For (i), to rigorously prove that the optimal policies are symmetric, one needs to show that any asymmetric policy cannot outperform the symmetric ones, where
the total number of deterministic asymmetric policies can be as many as \(O(N^M)\) (say all \(M\) agents with the same active day \(N\)), let alone the randomized asymmetric policies.
This argument is highly non-trivial, and one needs to prove it for both deterministic and randomized policies and for all three types of competitive ratios.
For (ii),
even with the symmetric policy restrictions in place,
agents still can have heterogeneous active days, which makes the design of optimal policies intricate.
For one thing, the state of the group keeps changing as agents become inactive, and the change of state affects the optimal policies, which is not the case for the classical ski-rental problem.
For another, the given state implies the paid costs for these inactive agents, which in turn affects the optimal policies for the remaining active agents, and this impact differs in different types of competitive ratios. We provide a more detailed discussion on the challenges in Sections~\ref{subsec:determistic-algorithmic-idea} and~\ref{subsec:randomized-algorithmic-idea}.

\subsection{Related Works}\label{sec:related-works}

\textbf{Ski-rental problems.} The canonical ski-rental problem is one of the most fundamental online algorithm problems:
With an unknown number of active ski days \(N\in \mathbb{N}^+\),
the skier can either rent skis for a daily cost \(1\), or buy skis costing \(B \in \mathbb{N}^+\) and terminate the game.
The skier's goal is to minimize the total cost of skiing.
The optimal deterministic policy is to rent skis for the first \(B-1\) days and buy skis if the skier reaches the threshold day \(B\), with
a competitive ratio of \(2 - 1/B\)~\citep{karlin1988competitive}.
The optimal randomized policy picks a threshold day \(T\in\mathbb{N}^+\) from a specifically devised distribution with parameter \(1/B\),
achieving a competitive ratio of \(e/(e-1)\approx 1.58\)~\citep{karlin1994competitive}.

After the seminal works~\citep{karlin1988competitive,karlin1994competitive}, the ski-rental problem has been extended to various settings~\citep{lotker2008rent,ai2014multi,wang2020online,zhang2020combinatorial,dinitz2024controlling,purohit2018improving,wu2022competitive,wu2021competitive}.
Among them, the multi-slope ski-rental problem~\citep{lotker2008rent} generalizes the renting and buying to a smooth action: renting for a period with cost proportional to the length of the period.
The multi-shop ski-rental problem~\citep{ai2014multi,wang2020online}
considers the location-dependent costs of renting and buying,
and the skier accounts for shop locations when making decisions.
The combinatorial ski-rental problem~\citep{zhang2020combinatorial} considers the rental of multiple items,
where the cost functions of renting and buying are combinatorial regarding the picked items.
The two-level ski-rental problem~\citep{wu2021competitive} addresses scenarios involving multiple commodities, with three payment options: renting (pay-per-use), individual purchase (lifetime access for a single item), and combo purchase (lifetime access for all items).
While the combo purchase option resembles our group buy, it pertains to acquiring multiple distinct resources (e.g., skis and helmets), fulfilling a sequence of online requests for different resources, which is different from our multi-agent ski-rental problem.
Some new perspectives on the ski-rental problem have also been investigated, such as the tail risk of the randomized policy~\citep{dinitz2024controlling,christianson2024risk} and ski-rental problem with predictions~\citep{purohit2018improving,zeynali2021data,wei2020optimal}.

\textbf{Multi-agent online algorithms and online learning.}
Multi-agent online algorithms are largely underexplored in the literature. One paper that we are aware of is~\citet{istrate2024equilibria}, where the authors also study a multi-agent variant of the ski-rental problem.
Their formulation generalizes the individual buy option to the group license option,
in which multiple agents collaboratively buy a group license by \emph{pooling resources}. If the total pooled resources are insufficient to buy the group license, the agents must pay for renting.
\citet{istrate2024equilibria} focuses explicitly on the equilibrium of how each agent decides its pooling price with the {prediction} of others' actions.
This is a fundamentally different setting from our work, where we have both the individual and group-buy options, and given the cost of buying the group pass evenly shared among all buyers,
We focus on the optimal policies for the whole group of agents to optimize different types of competitive ratios.

\citet{meng2025group} studied the \emph{homogeneous} multi-agent ski-rental problem, which consists of a group of skiers with the same active days with \emph{only} the individual buy options (without the group-buy option).
They show that the optimal deterministic policy is asymmetric with competitive ratio \(e/(e-1)\) when the number of skiers is large.
This is because the asymmetric buying policies can ``trick'' the adversary who is restricted to set the same active days for all skiers.
This result was also mentioned as a special case in~\citet[Algorithm 3]{wu2022competitive}.
The problem studied in~\citet{meng2025group} is different from ours: (i) we study the general multi-agent ski-rental problem where agents may have different active days, and (ii) we consider the group-buy option, none of which is considered in~\citet{meng2025group}.
Later on, we will show that the optimal policies for our general multi-agent ski-rental problem (\masr) are symmetric and state-aware, which is different from the asymmetric policies in~\citet{meng2025group}.


Beyond the ski‑rental problem, several other classical online decision problems have been studied through a multi‑agent lens.  The canonical \(k\)-server problem~\citep{manasse1988competitive}, for example, can be interpreted as \(k\) agents that must coordinate the placement of mobile servers to service requests across a metric space.  Likewise, the metrical task system (MTS) problem~\citep{borodin1992optimal} has been extended to a collaborative setting---referred to as \emph{collective MTS}---in which multiple agents work together to attain a competitive ratio that improves upon the single‑agent benchmark~\citep{cosson2024barely}.
Outside the online algorithm literature, the multi-agent setting in other related fields, such as online bandits learning~\citep{boursier2019sic,wang2023achieving}, reinforcement learning~\citep{tan1993multi,nowe2012game}, has been well studied.
Our work is one of the first attempts to study the multi-agent online algorithms with competitive analysis.


\section{Model Formulation for Multi-Agent Ski-Rental}
\label{sec:model-formulation}

We consider the multi-agent ski-rental (\masr) problem with \(M\in\mathbb{N}^+\) skiers (agents) as a \emph{group}.
Each skier \(m \in \mathcal{M} \coloneqq \{1,2,\dots, M\}\) has \(N_m\in\mathbb{N}^+\) active days to ski.
To ski on each active day, each skier can either rent the ski for a day with cost \(1\), or (individually) buy an individual pass with cost \(B\in\mathbb{N}^+\) for the rest of the snow season (i.e., all its remaining active days).
In addition, agents can also (collectively) buy a \emph{``group''}  pass with cost \(G\in\mathbb{N}^+\) for the ski access of the whole group for the rest of the snow season, where the cost \(G\) is shared equally among all agents participating in the purchase.
To remove trivial cases, we assume \(B < G < MB\).
Because when \(G \ge MB\) (resp. \(G \le B\)), the group (resp. individual) ski pass is never beneficial.
For simplicity of notation, we label agents in an ascending order regarding their active days \(N_1 \le N_2 \le \cdots \le N_m\) (break the tie arbitrarily), unknown to online algorithms.
Especially, with the given \(\{M, B, G\}\), denote a \masr instance as \(\mathcal{I} \coloneqq \{N_1, \dots, N_M\}\) and the set of all possible instances as \(\mathfrak{I}\).


\textbf{Decisions and policies.}
Denote \(\mathcal{A} \coloneqq \{\rent, \buy, \group, \leave\}\) as the action space of an agent, \(\rent\) means that the agent rents for a day, \(\buy\) means that the agent buys pass individually, \(\group\) means that the agent buys the group ski pass (with others also taking \(\group\) actions),
and \(\leave\) means that the agent either is inactive or had already bought the pass.
Denote \(\mathcal{H}_t \coloneqq \{a_{m,s}\}_{m\in\mathcal{M}, s \le t}\) as the history of the decision process up to day \(t\), where \(a_{m,s} \in \mathcal{A}\) is the action taken by agent \(m\) on day \(s\).
Let \(\pi\) denote a \emph{deterministic} policy for all \(M\) agents which maps the history \(\mathcal{H}_t\) to one of the four actions of each agent \(m\in\mathcal{M}\), i.e., \(\pi: \mathcal{H}_t \to \mathcal{A}^M\).
Denote \(\Pi\) as the set of all deterministic policies for \masr.
Denote the probability density function \(f(\cdot)\)
over the set of deterministic policies in \(\Pi\) as a \emph{randomized} policy, which maps any deterministic policy \(\pi\) to a probability, i.e., \(f: \Pi \to \Delta^{\abs{\Pi}-1}\),
where \(\Delta^{\abs{\Pi}-1}\) denotes the probability simplex over the set of deterministic policies \(\Pi\).
This policy \(f(\cdot)\) means that at the beginning (or a specific point) of the online problem,
it samples a deterministic policy \(\pi\) from the distribution \(f(\cdot)\) and then follows it.

\textbf{Cost.}
The cost of a deterministic policy \(\pi\) is defined as the total cost of all agents in the group as
\begin{align}\label{eq:group-cost}
    \cost\left( \pi, \mathcal{I} \right)
    \coloneqq \sum_{m=1}^M  \cost_m\left( \pi, \mathcal{I} \right)
    = \sum_{m=1}^M  \sum_{t=1}^{N_m} c_m(\pi; \mathcal{H}_t),
\end{align}
where the \(\mathcal I\) is the instance of the \masr problem, containing agents' active days, and \(\pi_m(\mathcal{H})\) is the action taken by agent \(m\) under the policy \(\pi\) and history \(\mathcal{H}\),
and for agent \(m\), the per-action cost \(c_m(\pi;\mathcal H)\)  is defined as follows,
\begin{align}
    c_m(\pi; \mathcal{H}) \coloneqq
    \begin{cases}
        1
         & \text{if } \pi_m(\mathcal{H}) = \rent
        \\
        B
         & \text{if } \pi_m(\mathcal{H}) = \buy
        \\
        {G} / {\sum_{m=1}^M \mathbb{I}[\pi_m(\mathcal{H}) {=} \group]}
         & \text{if }
        \pi_m(\mathcal{H}) = \group
        \\
        0
         & \text{if } \pi_m(\mathcal{H}) = \leave
    \end{cases}.
\end{align}
For a randomized policy \(f\), the total cost is defined as \(
\cost\left( f, \mathcal{I} \right)
\coloneqq
\mathbb{E}_{\pi\sim f(\cdot)}[\cost\left( \pi, \mathcal{I} \right)] = \sum_{\pi \in \Pi} f(\pi) \cdot \cost\left( \pi, \mathcal{I} \right).
\)

\subsection{Benchmark Offline Policies and Competitive Ratios}
\label{subsec:benchmark-policies}





We provide three offline optimal policies for \masr, where the first two are from the group perspective (minimize the total cost of all agents) and the last is from the individual agent aspect.

\textbf{Overall offline optimal policy.}
Given the knowledge of the active days of all agents, the offline optimal policy is either to (1) collectively buy the group pass in the beginning,
or (2) individually buy the individual pass in the beginning or rent for all active days, whichever is smaller.
The policy's total cost is
\begin{align}\label{eq:ovopt}
    \ovopt\left( \mathcal{I} \right)
    \coloneqq \min\left\{G, \sum_{m=1}^M \min\{B, N_m\}\right\}.
\end{align}
While the overall offline optimal policy
is a natural extension of that in canonical ski-rental,
the multi-agent nature of \masr introduces peculiar novel \emph{intricacy}:
During the decision-making procedure,
the agents gradually become inactive, which alters the number of active agents in the group.
Below, we define the another offline policy taking this intricacy into account.

\textbf{State-dependent offline optimal policy.}
In general, agents may become inactive due to either (i)
their active days being passed,
or (ii) the agents buy the individual/group pass.
However, as we are considering the offline optimal policy (and the online optimal policy later on),
the agents becoming inactive \emph{before} the termination of the sequential procedure must be due to all their active days being passed
(i.e., (i)); otherwise, the policy  (i.e., buy early) would be suboptimal~\citep{karlin1988competitive}.
Therefore, we only consider the case where the agents become inactive without buying,
and the \emph{state} of the group as the \emph{revealed} active days of the agents that becomes inactive at the current day.
As the number of active days of the agents is revealed sequentially when the agents become inactive, we define the \emph{state} of the group at the certain point of the decision-making procedure as the revealed number of active days of the agents that have already become inactive.


\begin{definition}[State]\label{def:state}
    In \masr, suppose that \(\ell\) agents have already become inactive and revealed the number of their active days.
    The \emph{state} at this moment is
    \(
    \{N_n\}_{n\le \ell},
    \)
    where \(N_n\) denotes the number of active days reported by the \(n\)-th inactive agent. When appropriate, we simplify the state by \(\ell\) just to indicate the number of inactive agents.
\end{definition}

We define the \emph{state-dependent} offline optimal cost as follows,
\begin{align}\label{eq:sdopt}
    \sdopt(\mathcal{I}\vert \{N_n\}_{n\le \ell})
     & \coloneqq \sum_{n=1}^\ell N_n + \min\left\{G, \sum_{n=\ell+1}^M \min\{N_n, B\}\right\}.
\end{align}
where the second term in the RHS is the overall optimal cost for the remaining \(M-\ell\) agents (see~\eqref{eq:ovopt}).
This state-dependent policy can be interpreted as the overall offline optimal policy for the remaining active agents, \emph{conditioned} that the first \(\ell\) agents are inactive.


\textbf{Individual rational offline optimal policy.}
\emph{Individually rational} agents aim to minimize their own individual cost.
Denote \(\ell_* \coloneqq \min \left\{0\le \ell \le M-1:  N_{\ell+1} > \min\left\{ {G}/{(M-\ell)}, B \right\}\right\}\) as the number of agents better off to rent than buy the ski pass.
Then, for each agent \(m\), the individual offline optimal policy implies the cost for each agent as follows,
\begin{align}\label{eq:indopt}
    \opt_{m}\individual(\mathcal{I})
    \coloneqq
    \begin{cases}
        N_m
         & \text{ if } N_m \leq \min\left\{ \frac{G}{M-\ell_*}, B \right\}
        \\
        \min\left\{\frac{G}{M-\ell_*}, B\right\}
         & \text{ if } N_m >  \min\left\{ \frac{G}{M-\ell_*}, B \right\}
    \end{cases},
\end{align}
where the first case means that agent \(m\) is better off to rent for all active days, and the second case means that the agents \(\{\ell_*+1,\ell_*+2,\dots, M\}\) are better off to buy the ski pass---either group or individual, whichever is cheaper.
Proposition~\ref{prop:indopt} shows that the individual offline optimal policy achieves the Nash equilibrium for the group.
All proofs of this paper are deferred to the appendices.


\begin{proposition}\label{prop:indopt}
    The offline policy in~\eqref{eq:indopt} achieves the general-sum Nash equilibrium for all agents.
\end{proposition}

\textbf{Competitive ratios.}
Given the three types of offline optimal policies, we can define three competitive ratios (CR) for the multi-agent ski-rental problem as follows,
\begin{align}
    \label{eq:overall-cr}
    \ovcr(\pi\vert \{N_n\}_{n\le \ell})
     & \coloneqq \max_{\mathcal{I} \in \mathfrak{I}\vert \{N_n\}_{n\le \ell}} \frac{\cost(\pi, \mathcal{I})}{\ovopt(\mathcal{I})}, \,\forall \ell\in\mathcal{M},
    \\\label{eq:state-dependent-cr}
    \sdcr(\pi \vert \{N_n\}_{n\le \ell} )
     & \coloneqq \max_{\mathcal{I} \in \mathfrak{I}\vert \{N_n\}_{n\le \ell}}  \frac{\cost(\pi, \mathcal{I})}{\sdopt(\mathcal{I}\vert \{N_n\}_{n\le \ell})}, \,\forall \ell\in\mathcal{M},
    \\\label{eq:individual-cr}
    \ratio\upbra{\text{IND}}_m (\pi\vert\{N_n\}_{n\le \ell})
     & \coloneqq \max_{\mathcal{I} \in \mathfrak{I}\vert \{N_n\}_{n\le \ell}}  \frac{\cost_m(\pi, \mathcal{I})}{\opt_m\upbra{\text{IND}}(\mathcal{I})}, \,\forall m, \ell\in\mathcal{M},
\end{align}
where the \(\max_{\mathcal{I} \in \mathfrak{I}\vert \{N_n\}_{n\le \ell}}\) considers the worst-case scenario of the cost of the policy \(\pi\) over all possible instances of \masr where first \(\ell\) agents have active days \(\{N_n\}_{n\le \ell}\).
Among all these competitive ratio definitions, the overall competitive ratio \(\ovcr(\pi\vert \emptyset)\) is the counterpart of that in canonical ski-rental,
while all others are novel and particularly defined for \masr.

\section{Warm-up: Policies for Homogeneous Multi-Agent Ski-Rental}
\label{sec:homogeneous-ma-ski-rental}
\label{sec:homogeneous-ski-rental}



In this warm-up section, we examine the \emph{homogeneous} multi-agent ski-rental problem, in which every agent is active for the same number of days, i.e., \(N_1 = N_2 = \dots = N_M = N\).
Because the state \(\{N_n\}_{n \le \ell}\) does not change over the decision procedure, there is not need to separately consider the three competitive-ratio notions introduced in Section~\ref{subsec:benchmark-policies}; we therefore focus exclusively on the overall competitive ratio \(\ratio\) defined in~\eqref{eq:overall-cr} in this section.

Although the homogeneous case is a special instance of the broader heterogeneous \masr\ studied later, it restricts the adversary to assign identical active-day horizons to all agents, which can give \emph{asymmetric} policies an edge.
We show in Section~\ref{sec:homogeneous-deterministic-policy} that, among deterministic strategies, certain asymmetric policies \emph{may} achieve better competitive ratios than any symmetric alternative.
In contrast, Section~\ref{sec:homogeneous-randomized-policy} proves that the optimal {randomized} policy remains symmetric.

The main body of the paper (Sections~\ref{sec:deterministic-heterogeneous} and~\ref{sec:randomized-algorithm}) returns to the general heterogeneous setting and demonstrates that optimal \emph{symmetric} policies---augmented by the cheaper group-buy option---always outperform asymmetric ones.
Thus, asymmetry only helps in the homogeneous deterministic case but offers no advantage in the heterogeneous problem.

\subsection{Deterministic Policy}\label{sec:homogeneous-deterministic-policy}

Without the group-buy option, \citet[Section 4]{meng2025group} and~\citet[Algorithm 3]{wu2022competitive} show that the optimal deterministic policy is \emph{asymmetric} (see~\citet[Eq.(1)]{meng2025group} for the non-closed-form solution), whose competitive ratio reaches to \(\frac{e}{e-1}\, (< 2)\) when the number of agents \(M\) is large enough~\citep[Proposition 4]{meng2025group}.
This is strictly better than the \(2\) competitive ratio of the single-agent symmetric policy~\citep{karlin1988competitive}.
With the group-buy option, the symmetric deterministic policy also achieves a competitive ratio better than \(2\). We present the optimal symmetric policy as follows.
\begin{proposition}\label{thm:homogeneous-deterministic-policy}
    For homogeneous \masr, the optimal symmetric deterministic policy is to set the threshold \(T\) as \(\ceil*{\frac{G}{M}}\) for all agents, and the competitive ratio is \(2 - \frac{M}{G}\) if \(\frac{G}{M}\) is an integer.\footnote{The competitive ratio becomes \(1 + \frac{M}{G}\floor*{\frac{G}{M}}\) if \(\frac{G}{M}\) is not an integer, which is still strictly less than \(2\).}
\end{proposition}
\noindent This proposition is derived by casting the multi-agent ski-rental problem as a single-agent problem with a daily rent cost \(M\), a total \(N\) active days, and the buy cost \(G\), which corresponding optimal threshold is \(\ceil*{\frac{G}{M}}\)~\citep{karlin1988competitive}.

Then, the optimal deterministic policy for homogeneous \masr is the one with the smaller competitive ratio between the optimal asymmetric and symmetric policies.
Since the competitive ratio of the optimal asymmetric policy does not have a closed-form expression, we cannot give an analytical criterion to determine which is better.
However, note that the competitive ratio of asymmetric policy decreases in terms of the number of agents \(M\) and converges to \(\frac{e}{e-1}\) as \(M\to\infty\)~\citep[Proposition 4]{meng2025group}.
Hence, when the group-buy option has a relative large discount (formally, \(\frac{G}{M} > \frac{e-1}{e-2} \approx 2.39\), which is usually true in practices) and the number of agents \(M\) is large,
the optimal symmetric policy is better.

We conduct a numerical experiment to compare the competitive ratios of the optimal asymmetric and symmetric policies for different values of \(M\) and \(G\) in Figure~\ref{fig:homogeneous-deterministic-policy-cr}.
We see that when the group-buy cost \(G\) is relatively large (e.g., \(B=0.5G\) and \(B=0.55G\)), the optimal policy is asymmetric, while when the group-buy cost \(G\) is relatively small (e.g., \(B=0.6G\) and \(B=0.7G\)), the optimal policy is symmetric.

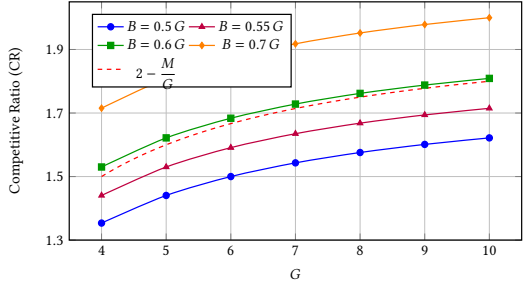
\begin{figure}[tbhp]
    \centering
    \resizebox{0.5\textwidth}{!}{
        \begin{tikzpicture}
            \begin{axis}[
                    xlabel={$G$},
                    ylabel={Competitive Ratio (CR)},
                    xmin=3.5, xmax=10.5,
                    ymin=1.30, ymax=2.05,
                    xtick={4,5,6,7,8,9,10},
                    ytick={1.30,1.50,1.70,1.90,2.10},
                    grid=both,
                    width=12cm,
                    height=7cm,
                    thick,
                    legend style={at={(0.05,0.95)},anchor=north west,legend columns=2},
                ]
                \addplot[
                    color=blue,
                    mark=*,
                    smooth,
                ] coordinates {
                        (4,  1.3536)
                        (5,  1.4405)
                        (6,  1.5000)
                        (7,  1.5431)
                        (8,  1.5757)
                        (9,  1.6012)
                        (10, 1.6217)
                    };
                \addlegendentry{$B=0.5\,G$}

                \addplot[
                    color=purple,
                    mark=triangle*,
                    smooth,
                ] coordinates {
                        (4,  1.4405)
                        (5,  1.5301)
                        (6,  1.5910)
                        (7,  1.6349)
                        (8,  1.6681)
                        (9,  1.6940)
                        (10, 1.7148)
                    };
                \addlegendentry{$B=0.55\,G$}

                \addplot[
                    color=green!60!black,
                    mark=square*,
                    smooth,
                ] coordinates {
                        (4,  1.5301)
                        (5,  1.6217)
                        (6,  1.6836)
                        (7,  1.7282)
                        (8,  1.7618)
                        (9,  1.7880)
                        (10, 1.8090)
                    };
                \addlegendentry{$B=0.6\,G$}

                \addplot[
                    color=orange,
                    mark=diamond*,
                    smooth,
                ] coordinates {
                        (4,  1.7153)
                        (5,  1.8090)
                        (6,  1.8724)
                        (7,  1.9178)
                        (8,  1.9520)
                        (9,  1.9786)
                        (10, 2.0000)
                    };
                \addlegendentry{$B=0.7\,G$}

                \addplot[
                    color=red,
                    domain=4:10,
                    samples=100,
                    dashed,
                ] {2 - 2/x};
                \addlegendentry{$2 - \dfrac{M}{G}$}
            \end{axis}
        \end{tikzpicture}}
    \caption{Competitive ratio comparison of the optimal asymmetric and  symmetric policies for \(M=2\) agents: all the solid lines corresponds to the competitive ratio of the optimal asymmetric policies for different relations between individual-buy cost \(B\) and group-buy cost \(G\), while the red dashed line for \(2- \frac{M}{G}\) corresponds to the competitive ratio of the optimal symmetric policy.}\label{fig:homogeneous-deterministic-policy-cr}
\end{figure}

\subsection{Randomized Policy}\label{sec:homogeneous-randomized-policy}

Without the group-buy option, \citet[Section 3]{meng2025group} study the optimal randomized policy the homogeneous \masr problem.
They show that the optimal randomized policy can be \emph{symmetric} among all agents, and the optimal threshold distribution is given by
\(
p_t = \frac{1}{B}\left( 1 - \frac{1}{B} \right)^{\frac{1}{B} - t}, \forall t\in \left\{1,2,\dots, B\right\},
\)
and the optimal competitive ratio is still \(\frac{e}{e-1}\) as the single-agent ski-rental problem.

As the group-buy option only reduces the cost when agents buy the pass together, the optimal randomized policy for the homogeneous \masr problem with group-buy option is still symmetric, which is summarized as follows.

\begin{theorem}\label{thm:homogeneous-randomized-policy}
    For homogeneous \masr, the optimal randomized policy is to randomly pick a threshold day \(T\) from the following distribution for buying a group pass if the group reaches the sampled threshold,
    where the probability density function is, \begin{align}\label{eq:homogeneous-randomized-policy}
        r_t \coloneqq \frac{M}{G}\left( 1 - \frac{M}{G} \right)^{\frac{G}{M} - t}, \forall t\in \left\{1,2,\dots, \ceil*{\frac{G}{M}}\right\}.
    \end{align}
    Especially, the policy has a competitive ratio  \(\ratio = {1} \big/ \left( 1 - \left( 1 - \frac{M}{G} \right)^{\frac{G}{M}} \right)\).
\end{theorem}

Different from the asymmetric and symmetric comparisons in the deterministic case, the optimal randomized policy is symmetric for sure.
This competitive ratio is asymptotically \(\frac{e}{e-1}\) when \(\frac{G}{M} \to\infty\), the same as that of the canonical ski-rental~\citep{karlin1994competitive}.
Similar to the optimal deterministic symmetric policy in Theorem~\ref{thm:homogeneous-deterministic-policy}, the optimal randomized policy can also be interpreted as the optimal policy for a single-agent ski-rental problem.


\section{Deterministic Policies for Multi-Agent Ski-Rental}
\label{sec:deterministic-heterogeneous}

In this section, we develop optimal deterministic policies.
We first present high-level algorithmic ideas in Section~\ref{subsec:determistic-algorithmic-idea},
a discussion on the superiority of the symmetric policy in the general \masr in Section~\ref{subsec:symetric-superiority-deterministic},
and specifically propose a policy for minimizing the overall \(\ratio\) (Section~\ref{subsec:deterministic-heterogeneous-policy-overall}), state-dependent \(\ratio\) (Section~\ref{subsec:deterministic-heterogeneous-state-dependent}), and individual rational \(\ratio\) (Section~\ref{subsec:deterministic-heterogeneous-policy-individual-rational}), respectively.

\subsection{Algorithm Design Challenge and Key Ideas}\label{subsec:determistic-algorithmic-idea}

\textbf{Algorithm design challenge.}
As mentioned in Section~\ref{subsec:benchmark-policies},
\masr has one crucial intricacy: the \emph{state} of the group changes over time, where the state
is the revealed active days of the \(\ell\) agents that has been inactive so far, i.e., \(\{N_n\}_{n \le \ell}\),
often simplified as an \(\ell\) alone.
As the decision proceeds, agents gradually become inactive and reveal their active days. The accumulation of these revealed active days spurs the online algorithms (with different \(\ratio\) objectives) to make corresponding alterations.
This kind of state-aware decision-making does not present in canonical ski-rental, and one needs to design \emph{state-aware policies} to adapt to the state change.
Specifically, given a state, the online algorithm needs to decide
when to buy the pass (i.e., determine the thresholds).

\textbf{Key design ideas.}
First, to choose between the same threshold for the group pass (symmetric policy) or different thresholds for the individual passes (asymmetric policy), we rigorously show that
the symmetric policy is better than the asymmetric policy.
Second, given the state \(\{N_n\}_{n\le \ell}\), the symmetric policy picks the same threshold for all {active} agents, which we denote as \(T(\{N_n\}_{n\le \ell})\), often abbreviated as \(T(\ell)\).
Whenever the state changes, say from \(\{N_n\}_{n\le \ell-1}\) to \(\{N_n\}_{n\le \ell}\) (i.e., agent \(\ell\) becomes inactive, or formally {\(N_\ell < T(\ell - 1)\)}), the algorithm needs to alter the threshold for the remaining agents.
The new given state \(\{N_n\}_{n\le \ell}\) implies paid costs for these inactive agents---the sum of the revealed active days of the inactive agents \(\sum_{n=1}^\ell N_n\).
Depending on how a specific objective (i.e., overall, state-dependent, and individual rational) accounts for this paid cost,
we design the corresponding threshold functions \(T(\ell)\).
Figure~\ref{fig:hereogeneous-deterministic-policy} illustrates how the state-aware threshold \(T(\ell)\) changes when the state changes.


\textbf{Deterministic policy {framework}.} With the above key ideas, we propose a general framework (Algorithm~\ref{alg:deterministic-optimal-policy}) of the deterministic policy.
Besides the prior knowledge of \((M, B, G)\), the algorithm takes a threshold function \(T(\ell)\) as input,
which is specifically devised for each type of competitive ratio in the following subsections.
For each active day \(t\) (with at least one agent active at the day), the algorithm checks whether the threshold \(T(\ell)\) is reached.
If so (i.e., \(t=T(\ell)\)), the algorithm decides to buy either the group pass (if \(T(\ell) < B\)) or the individual pass (otherwise) and terminates (Lines~\ref{line:deterministic-buy}-\ref{line:deterministic-terminate}).
Otherwise, the algorithm rents for all active agents (Line~\ref{line:deterministic-rent}).
After each active day, the agents with \(N_m=t\) reveal (i.e., become inactive), and the algorithm updates the number of inactive agents \(\ell\) and records the revealed active days of these agents \(\{N_n\}_{n\le \ell}\) (Lines~\ref{line:deterministic-begin-update-inactive}-\ref{line:deterministic-end-update-inactive}).
\begin{algorithm}[tbhp]
    \caption{Framework for deterministic symmetric policy for \masr}
    \label{alg:deterministic-optimal-policy}
    \begin{algorithmic}[1]
        \Input number of agents \(M\), group pass cost \(G\), individual pass cost \(B\), threshold function \(T(\cdot)\)
        \Initial number of inactive agents \(\ell \gets 0\)
        \For{each active day \(t\) (with active agent(s))}
        \If{\(t = T(\ell)\)}
        \State \textbf{Buy} \(\begin{cases}
            \text{Group pass}      & \text{if } T(\ell) < {B}
            \\
            \text{Individual pass} & \text{otherwise}
        \end{cases}\) \label{line:deterministic-buy}
        \State \textbf{Terminate} \label{line:deterministic-terminate}
        \Else
        \State \textbf{Rent} for all active agents \label{line:deterministic-rent}
        \EndIf

        \State Reveal \(\mathcal{L}_t\): set of agents become inactive at the end of day \(t\) \label{line:deterministic-begin-update-inactive}
        \If{\(\mathcal{L}_t \neq \emptyset\)}
        \State \(\ell \gets \ell + \abs{\mathcal{L}_t}\)
        \State \(N_m \gets t\) for all \(m \in \mathcal{L}_t\)
        \EndIf\label{line:deterministic-end-update-inactive}

        \EndFor
    \end{algorithmic}
\end{algorithm}

\textbf{Technical analysis challenge.}
As the key ideas above illustrate, we need to rigorously show that the optimal deterministic policy is symmetric, where one needs to prove that for any asymmetric policy, there exists a symmetric policy that achieves a better competitive ratio.
Given the heterogeneous active days of agents, the number of possible asymmetric policies is large, making the analysis highly non-trivial.
For another thing, as the state changes over time, the algorithm needs to adaptively change the threshold as well. This iterative (and time-varying) decision-making process makes the analysis of the future performance of an online algorithm challenging.
Whenever the state changes, to minimizing the corresponding (worst-case) competitive ratio, the next threshold needs to take account all the possible active days for the remaining active agents in the future, i.e., the instance space \(\mathfrak{I}\vert \{N_n\}_{n\le \ell}\) conditioned on the current state \(\{N_n\}_{n\le \ell}\),
which is again a large space of possible instances.

\subsection{Dominance of Symmetric Deterministic Policies}\label{subsec:sym-det}\label{subsec:symetric-superiority-deterministic}

This subsection establishes a structural result that guides the design of all deterministic algorithms analyzed later.  Intuitively, when every active agent follows the same threshold rule---a \emph{symmetric} policy---the group can never perform worse than under any different assignment of thresholds (\emph{asymmetric} policy).  The next lemma formalizes this intuition.

\begin{lemma}[Symmetry beats asymmetry]\label{lemma:symmetric-better-than-asymmetric-deterministic}
    Fix an arbitrary state \(\{N_n\}_{n\le \ell}\).  For every deterministic \emph{asymmetric} policy there exists a deterministic \emph{symmetric} policy whose competitive ratio---under each of the three definitions introduced earlier---is strictly smaller.
\end{lemma}

Lemma~\ref{lemma:symmetric-better-than-asymmetric-deterministic} permits us to restrict attention to symmetric strategies in the remainder of the section.  Its proof, deferred to Appendix~\ref{sec:symmetric-better-than-asymmetric-deterministic-proof}, iteratively adjusts any asymmetric policy: at each step the adjustment reduces the ``asymmetricity'' of the policy, which decreases the competitive ratio; the process ends when all thresholds coincide, yielding a symmetric policy.

Once symmetry is enforced from the perspective of the online algorithm---i.e., all active agents adopt the same threshold \(T\)---the adversary,s worst‑case instance is to assign an identical number of active days, \(N_m = T\) for every still‑active agent.  Any other choice would only \emph{lower} the competitive ratio, contradicting the adversary’s maximizing objective \citep{karlin1988competitive}.  Consequently, when constructing optimal deterministic policies, we only need to consider those instances in which all active agents share the same number of active days.

\subsection{Optimal Deterministic Policy for Overall Competitive Ratio}
\label{subsec:deterministic-heterogeneous-policy-overall}

The overall offline optimal cost in~\eqref{eq:overall-cr} is the minimal between the group-buy cost \(G\)
and the sum of the minimum among the individual-buy cost and the active days of all agents, i.e., \(\sum_{m=1}^M \min\{N_m, B\}\).
To match the overall optimal,
these paid costs \(\sum_{n=1}^\ell N_n\) are considered as \emph{irrevocable losses}.
When the group pass is cheaper, the policy needs to be more aggressive in buying it, given the paid cost \(\sum_{n=1}^\ell N_n\).
Specifically,
the \emph{``equivalent''} cost of buying the group pass,
averaged by all \(M-\ell\) remaining active agents,
would become \(\frac{G - \sum_{n=1}^\ell N_n}{M-\ell}\).
Considering the individual pass cost \(B\), the policy should choose the group pass when the equivalent average group cost is smaller than \(B\) and otherwise, the individual pass.
Formally, the threshold for the overall competitive ratio at state \(\ell\) is
\begin{align}\label{eq:threshold-overall}
    T\overall(\ell)  = T\overall(\{N_n\}_{n\le \ell})  \coloneqq \min\left\{ {\frac{G - \sum_{n=1}^{\ell} N_n}{M-\ell}}, B \right\}.
\end{align}


\begin{theorem}\label{thm:deterministic-heterogeneous-policy-overall}
    The deterministic policy of Algorithm~\ref{alg:deterministic-optimal-policy} with input threshold \(\ceil*{T\overall(\ell)}\) in~\eqref{eq:threshold-overall} is optimal for minimizing overall competitive ratio, denoted as \(\pi_*\overall\).
    When \(T\overall(\ell)\) is an integer,\footnote{
        For the sake of simplicity in presentation and discussion, this section only presents the ratios for integer thresholds \(T(\ell)\).
        The non-integer counterparts can be easily derived with careful calculations but introduce unnecessary complications in understanding the results.}
    the competitive ratio can be expressed as follows,
    \begin{align}\label{eq:deterministic-overall-competitive-ratio}
         & \ovcr(\pi_*\overall \vert \{N_n\}_{n\le\ell})
        \!=\!
        \begin{cases}
            2 - \frac{M-\ell}{G}
             & \text{if } G \le (M-\ell)B
            \\
            1 + \frac{(M-\ell)(B-1)}{G}
             & \text{if } (M{-}\ell) B < G \le \sum_{n=1}^{\ell}N_n {+} (M{-}\ell) B
            \\
            1 + \frac{(M-\ell)(B-1)}{\sum_{n=1}^{\ell} N_n + (M-\ell)B}
             & \text{if } G > \sum_{n=1}^{\ell} N_n + (M-\ell) B
        \end{cases}
    \end{align}
\end{theorem}

All optimality statements (as in Theorem~\ref{thm:deterministic-heterogeneous-policy-overall}) for deterministic policies in this section consist of two arguments: (i) the devised policy is optimal for any possible instance with the identical active days of agents, detailed in the proof in Appendix~\ref{sec:deterministic-heterogeneous-policy-overall-proof}, and
(ii) any other instances with different active days of agents will lead to a smaller competitive ratio (i.e., not the worst case), illustrated in Section~\ref{subsec:symetric-superiority-deterministic}.
The same holds for Theorems~\ref{thm:deterministic-heterogeneous-policy-state-dependent} and ~\ref{thm:deterministic-heterogeneous-policy-individual} for state-dependent and individual rational competitive ratios.
We plot the curve of the competitive ratio \(\ovcr(\pi_*\overall \vert \ell)\) as a function of the number of inactive agents \(\ell\) as the \textcolor{red}{red} line in Figure~\ref{subfig:deterministic-CR-trend-overall} under a specific instance of \(M=10, B=10, G=60\) and active days are \(\mathcal{I} = (1,2,\dots,10)\).
As the state \(\ell\) increases, the \(\ovcr(\pi_*\overall \vert \ell)\) increases under the first condition \(G \le (M-\ell)B\) and decreases afterwards.

The first condition \(G \le (M-\ell) B\) (Region A) of the piece-wise ratio in~\eqref{eq:deterministic-overall-competitive-ratio} implies that in the current state, the group pass is still cheaper than the individual pass for the remaining active agents.
In that case, the competitive ratio \(2 - \frac{M-\ell}{G}\) can be regarded as the generalization of the competitive ratio \(2 - \frac{M}{G}\) in the homogeneous case in Theorem~\ref{thm:homogeneous-deterministic-policy}, where the number of active agents is \(M-\ell\) instead of \(M\).
This increase is because in Region A, the group pass is still cheaper than the individual pass for the remaining active agents.
However, the later the threshold is reached in Region A, the more agents become inactive, yielding larger paid costs \(\sum_{n=1}^{\ell} N_n\) upon the unavoidable group pass cost \(G\) and thus worse competitive ratios.
Once the group pass is not preferred (Regions B and C), i.e., \(G > (M-\ell) B\), the competitive ratio decreases regarding the number of inactive agents, as expected.
These two regions (second and third pieces of~\eqref{eq:deterministic-overall-competitive-ratio}) are separated by \(\sum_{n=1}^\ell N_n + (M-\ell) B\) due to the fact that the overall \(\ratio\) takes account the paid costs \(\sum_{n=1}^\ell N_n\) as irrevocable losses.
Hence, whether the group pass is better or not actually depends on the sum of the paid costs and that of individual pass of the remaining active agents \((M-\ell)B\).


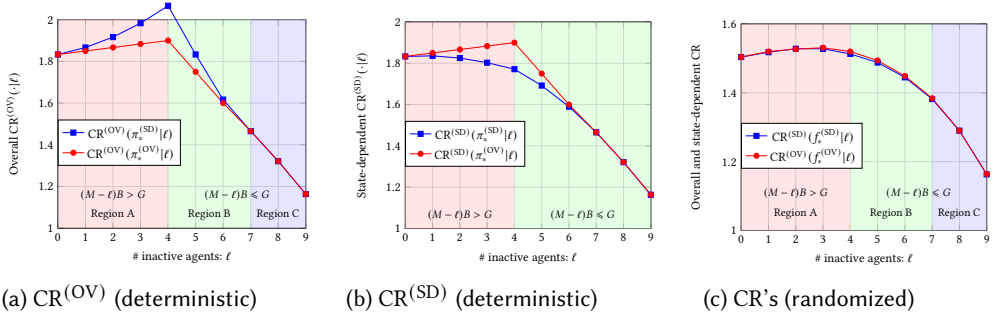
\begin{figure}[t]
    \centering
    \begin{subfigure}{0.25\textwidth}
        \resizebox{!}{\textwidth}{
            \begin{tikzpicture}
                \begin{axis}[
                        xlabel={\# inactive agents: $\ell$},
                        ylabel={Overall \(\ovcr(\cdot\vert \ell)\)},
                        xtick={0, 1,2,3,4,5,6,7,8,9},
                        grid=both,
                        ymin=1, ymax=2, 
                        xmin=0, xmax=9, 
                        clip=false,      
                        legend style={at={(0,0.3)}, anchor=south west},
                    ]
                    \draw[fill=red!20, fill opacity=0.5, draw=none]
                    (axis cs:0,1) rectangle (axis cs:4,2);
                    \node[anchor=south, font=\small] at (axis cs:2,1.02) {Region A};
                    \node[anchor=south, font=\small] at (axis cs:2,1.11) {\((M-\ell)B > G\)};

                    \draw[fill=green!20, fill opacity=0.5, draw=none]
                    (axis cs:4,1) rectangle (axis cs:7,2);
                    \node[anchor=south, font=\small] at (axis cs:5.5,1.02) {Region B};

                    \draw[fill=blue!20, fill opacity=0.5, draw=none]
                    (axis cs:7,1) rectangle (axis cs:9,2);
                    \node[anchor=south, font=\small] at (axis cs:8,1.02) {Region C};

                    \node[anchor=south, font=\small] at (axis cs:6.5,1.11) {\((M-\ell)B \le G\)};

                    \addplot[
                        color=blue,
                        mark=square*,
                        thick,
                    ] coordinates {
                            (0, {11 / 6)})
                            (1, {28/15})
                            (2, {23/12})
                            (3, {119/60})
                            (4, {31/15})
                            (5, {11/6})
                            (6, {97/60})
                            (7, {85/58})
                            (8, {37/28})
                            (9, {64/55})
                        };
                    \addlegendentry{\(\ovcr(\pi_*\statedep\vert \ell)\)}

                    \addplot[
                        color=red,
                        mark=*,
                        thick,
                    ] coordinates {
                            (0, {2 - ((10 - 1 + 1) / 60)})
                            (1, {2 - ((10 - 2 + 1) / 60)})
                            (2, {2 - ((10 - 3 + 1) / 60)})
                            (3, {2 - ((10 - 4 + 1) / 60)})
                            (4, {2 - ((10 - 5 + 1) / 60)})
                            (5, {1 + ((10 - 6 + 1) * (10 - 1) / 60)})
                            (6, {1 + ((10 - 7 + 1) * (10 - 1) / 60)})
                            (7, {1 + ((10 - 8 + 1) * (10 - 1) / 58)})
                            (8, {1 + ((10 - 9 + 1) * (10 - 1) / 56)})
                            (9, {1 + ((10 - 10 + 1) * (10 - 1) / 55)})
                        };
                    \addlegendentry{\(\ovcr(\pi_*\overall\vert \ell)\)}
                \end{axis}
            \end{tikzpicture}}
        \caption{\(\ratio\overall\) (deterministic)}
        \label{subfig:deterministic-CR-trend-overall}
    \end{subfigure}
    \hspace{0.07\textwidth}
    \begin{subfigure}{0.24\textwidth}
        \centering
        \resizebox{!}{\textwidth}{
            \begin{tikzpicture}
                \begin{axis}[
                        xlabel={\# inactive agents: $\ell$},
                        ylabel={State-dependent \(\ratio\statedep(\cdot|\ell)\)},
                        xtick={0,1,2,3,4,5,6,7,8,9},
                        grid=both,
                        ymin=1, ymax=2, 
                        xmin=0, xmax=9, 
                        clip=false,      
                        legend style={at={(0,0.3)}, anchor=south west},
                    ]
                    \draw[fill=red!20, fill opacity=0.5, draw=none]
                    (axis cs:0,1) rectangle (axis cs:4,2);
                    \node[anchor=south, font=\small] at (axis cs:2,1.02) {\((M-\ell)B>G\)};

                    \draw[fill=green!20, fill opacity=0.5, draw=none]
                    (axis cs:4,1) rectangle (axis cs:9,2);
                    \node[anchor=south, font=\small] at (axis cs:6.5,1.02) {\((M-\ell)B\le G\)};

                    \addplot[
                        color=blue,
                        mark=square*,
                        thick,
                    ] coordinates {
                            (0, {11 / 6})
                            (1, {112 / 61})
                            (2, {115 / 63})
                            (3, {119 / 66})
                            (4, {124 / 70})
                            (5, {110 / 65})
                            (6, {97 / 61})
                            (7, {1 + ((10 - 8 + 1) * (10 - 1) / 58)})
                            (8, {1 + ((10 - 9 + 1) * (10 - 1) / 56)})
                            (9, {1 + ((10 - 10 + 1) * (10 - 1) / 55)})
                        };
                    \addlegendentry{\(\ratio\statedep(\pi_*\statedep \vert \ell)\)}
                    \addplot[
                        color=red,
                        mark=*,
                        thick,
                    ] coordinates {
                            (0, {2 - ((10 - 1 + 1) / 60)})
                            (1, {2 - ((10 - 2 + 1) / 60)})
                            (2, {2 - ((10 - 3 + 1) / 60)})
                            (3, {2 - ((10 - 4 + 1) / 60)})
                            (4, {2 - ((10 - 5 + 1) / 60)})
                            (5, {1 + ((10 - 6 + 1) * (10 - 1) / 60)})
                            (6, {1 + ((10 - 7 + 1) * (10 - 1) / 60)})
                            (7, {1 + ((10 - 8 + 1) * (10 - 1) / 58)})
                            (8, {1 + ((10 - 9 + 1) * (10 - 1) / 56)})
                            (9, {1 + ((10 - 10 + 1) * (10 - 1) / 55)})
                        };
                    \addlegendentry{\(\ratio\statedep(\pi_*\overall \vert \ell)\)}
                \end{axis}
            \end{tikzpicture}}
        \caption{\(\ratio\statedep\) (deterministic)}
        \label{subfig:deterministic-CR-trend-state-dependent}
    \end{subfigure}
    \hspace{0.07\textwidth}
    \begin{subfigure}{0.24\textwidth}
        \centering
        \resizebox{!}{\textwidth}{
            \begin{tikzpicture}
                \begin{axis}[
                        xlabel={\# inactive agents: $\ell$},
                        ylabel={Overall  and state-dependent CR},
                        xtick={0, 1,2,3,4,5,6,7,8,9},
                        grid=both,
                        ymin=1, ymax=1.6, 
                        xmin=0, xmax=9, 
                        clip=false,      
                        legend style={at={(0,0.3)}, anchor=south west},
                    ]
                    \draw[fill=red!20, fill opacity=0.5, draw=none]
                    (axis cs:0,1) rectangle (axis cs:4,1.6);
                    \node[anchor=south, font=\small] at (axis cs:2,1.08) {\((M-\ell)B > G\)};

                    \node[anchor=south, font=\small] at (axis cs:2,1.02) {Region A};

                    \draw[fill=green!20, fill opacity=0.5, draw=none]
                    (axis cs:4,1) rectangle (axis cs:7,1.6);
                    \node[anchor=south, font=\small] at (axis cs:5.5,1.02) {Region B};

                    \draw[fill=blue!20, fill opacity=0.5, draw=none]
                    (axis cs:7,1) rectangle (axis cs:9,1.6);
                    \node[anchor=south, font=\small] at (axis cs:8,1.02) {Region C};

                    \node[anchor=south, font=\small] at (axis cs:6.5,1.08) {\((M-\ell)B \le G\)};

                    \addplot[
                        color=blue,
                        mark=square*,
                        thick,
                    ] coordinates {
                            (0, 1.504)
                            (1,1.518)
                            (2,1.5283)
                            (3,1.5279)
                            (4,1.513)
                            (5,1.488)
                            (6,1.445)
                            (7,1.382)
                            (8,1.290)
                            (9,1.163)
                        };
                    \addlegendentry{\(\sdcr(f_*\statedep\vert \ell)\)}

                    \addplot[
                        color=red,
                        mark=*,
                        thick,
                    ] coordinates {
                            (0, 1.504)
                            (1, 1.520)
                            (2, 1.527)
                            (3, 1.531)
                            (4, 1.520)
                            (5, 1.494)
                            (6, 1.449)
                            (7, 1.384)
                            (8, 1.290)
                            (9, 1.164)
                        };
                    \addlegendentry{\(\ovcr(f_*\overall\vert \ell)\)}

                \end{axis}
            \end{tikzpicture}}
        \caption{\(\ratio\)'s (randomized)}
        \label{subfig:randomized-policy-comparison}
    \end{subfigure}
    \caption{Compare policies in terms of group competitive ratios,
        \(\ovcr\) and \(\sdcr\), under the scenario that \(G = 60, B = 10\) and \(M=10\) agents with active days \(\mathcal{I} = (1, 2, 3, \dots, 10)\). More discussions in Appendix~\ref{app:deterministic-CR-trend}.
    }\label{fig:deterministic-CR-trend}
\end{figure}

\subsection{Optimal Deterministic Policy for State-Dependent Competitive Ratio}
\label{subsec:deterministic-heterogeneous-state-dependent}
\label{subsec:deterministic-individual-rational}

The offline optimal state-dependent cost in~\eqref{eq:state-dependent-cr} takes the paid cost \(\sum_{n=1}^\ell N_n\) as \emph{sunk costs} for the remaining active agents.
Hence, the policy does not account for the paid cost by inactive agents (who have already left the group).
From this perspective, the policy considers the group cost averaged over the remaining active agents, i.e., \(\frac{G}{M-\ell}\), and compares it with the individual pass cost \(B\).
This yields the threshold function for the state-dependent competitive ratio as follows,
\begin{align}\label{eq:threshold-state-dependent}
    T\statedep (\ell) \coloneqq \min\left\{ {\frac{G}{M-\ell}}, B \right\}.
\end{align}

\begin{theorem}\label{thm:deterministic-heterogeneous-policy-state-dependent}
    The deterministic policy of Algorithm~\ref{alg:deterministic-optimal-policy} with input threshold \(\ceil{T\statedep(\ell)}\) in~\eqref{eq:threshold-state-dependent} is optimal for minimizing state-dependent
    competitive ratio, denoted as \(\pi_*\statedep\).
    When \(T\statedep(\ell)\) is an integer, the competitive ratio can be expressed as follows,
    \begin{align}\label{eq:deterministic-state-dependent-competitive-ratio}
         & \sdcr(\pi_*\statedep \vert \{N_n\}_{n\le\ell})
        =
        \begin{cases}
            1 +
            \frac{G - (M-\ell)}{\sum_{n=1}^{\ell} N_n + G}
             & \text{ if } G \le (M-\ell)B
            \\
            1 +
            \frac{(M-\ell)(B-1)}{\sum_{n=1}^{\ell} N_n + (M-\ell)B}
             & \text{ if } G > (M-\ell)B
        \end{cases}.
    \end{align}
\end{theorem}

The state-dependent \(\ratio\) in~\eqref{eq:deterministic-state-dependent-competitive-ratio} has only two cases.
Because the state-dependent \(\ratio\) does not account the sunk costs \(\sum_{n=1}^\ell N_n\) for future decisions, and only whether the group pass is cheaper than the individual pass, i.e., \(G\) vs. \((M-\ell) B\), matters.
Figure~\ref{fig:deterministic-CR-trend} reports a numerical comparison of the \(\pi_*\overall\) and \(\pi_*\statedep\) policies regarding the overall and state-dependent \(\ratio\)s. Under the overall \(\ratio\) in Figure~\ref{subfig:deterministic-CR-trend-overall}, \(\pi_*\overall\) enjoys a lower ratio than that of \(\pi_*\statedep\),
and under the state-dependent \(\ratio\) in Figure~\ref{subfig:deterministic-CR-trend-state-dependent},
\(\pi_*\statedep\) outperforms \(\pi_*\overall\).
This observation corroborates the optimality stated in Theorems~\ref{thm:deterministic-heterogeneous-policy-overall} and~\ref{thm:deterministic-heterogeneous-policy-state-dependent}.
Empirical comparisons are reported in Section~\ref{sec:experiments}.


\subsection{Optimal Deterministic Policy for Individual Rational Competitive Ratio}
\label{subsec:deterministic-heterogeneous-policy-individual-rational}

As agents aim to minimize their individual costs for \(\ratio\individual\), the paid costs \(\sum_{n=1}^\ell N_n\) (for other agents) is also considered as {sunk costs} for the remaining active ones, as for the state-dependent \(\ratio\). Hence, the threshold function is also the same as that for the state-dependent \(\ratio\) in~\eqref{eq:threshold-state-dependent}.

Recall that \(\ell_*\) defined for~\eqref{eq:indopt} is the number of agents that are better off to rent in all active days regarding the individual rational perspective.
We notice \(T\statedep(\ell_*) = \opt_{m}\individual\) for any agent \(m> \ell_*\) that is better off to buy the pass,
where the threshold for the state of \(\ell_*\) inactive agents (in parenthesis at LHS) is exactly deployed to the
\(M-\ell_*\) agents (in subscript at RHS).
In other words, with Algorithm~\ref{alg:deterministic-optimal-policy} following the state-dependent threshold  \(T\statedep(\ell)\), the agents \(\{1,\dots, \ell_*\}\)---supposed to rent in all active days from the offline perspective---indeed rent in all active days (since \(N_m < T(\ell_*) \le T(m)\)),
while agents \(\{\ell_*{+}1,\dots, M\}\)---supposed to buy the pass at the cost of \(\opt_{m}\individual {=} T(\ell_*)\)---rent until the threshold \(T(\ell_*)\) is reached.
Therefore, agents \(m\, (\le \ell_*)\) have \(\ratio\)s exactly \(1\), while agents \(m\, (> \ell_*)\) have \(\ratio\)s strictly less than \(2\), summarized in the following theorem.

\begin{theorem}\label{thm:deterministic-heterogeneous-policy-individual}
    \label{thm:deterministic-heterogeneous-policy-individual-rational}
    The deterministic policy of Algorithm~\ref{alg:deterministic-optimal-policy} with input threshold \(\ceil*{T\statedep(\ell)}\) in~\eqref{eq:threshold-state-dependent} is optimal for minimizing individual rational competitive ratio, denoted as \(\pi_*\individual\).
    When \(T\statedep(\ell)\) is an integer, the competitive ratio can be expressed as follows,
    \begin{align}\label{eq:deterministic-individual-rational-competitive-ratio}
        \ratio\individual_m(\pi_*\individual \vert \{N_{n}\}_{n\le m-1}) =
        \begin{cases}
            1
             & \text{ if } m \le \ell_*
            \\
            2 - \frac{1}{T\statedep(\ell_*)}
             & \text{ if } m > \ell_*
        \end{cases}
    \end{align}
    where the individual rational competitive ratio for agent \(m\) is calculated under the state of \(m-1\) inactive agents, i.e., \(\{N_{n}\}_{n\le m-1}\), since agents reveal their active days in an increasing order.
\end{theorem}

Comparing to the single agent ski-rental problem,
the smaller individual rational \(\ratio\) in~\eqref{eq:deterministic-individual-rational-competitive-ratio} highlights the advantage of taking the group pass option into account.
Agents \(m\le \ell_*\) achieves the competitive ratio \(1\), their best possible ratio.
Even for agents \(m > \ell_*\), their competitive ratio \(2 - ({1}/{T\statedep(\ell_*)})\) is no worse than the competitive ratio \(2 - ({1}/{B})\) of the canonical ski-rental problem,
and is strictly better when the group pass is cheaper than the individual pass, i.e., \(T\statedep(\ell_*) < B\).


\section{Randomized Policies for Multi-Agent Ski-Rental}
\label{sec:randomized-algorithm}

This section investigates the optimal randomized policies for \masr.
We start with the policy framework in Section~\ref{subsec:randomized-algorithmic-idea}, discuss the superiority of symmetric policies and lower bound in Section~\ref{subsec:symmetric-and-lower-bound}, and then devise specific policies for three types of competitive ratios: overall (Section~\ref{subsec:randomized-overall}), state-dependent (Section~\ref{subsec:randomized-state-dependent}), and individual rational (Section~\ref{subsec:randomized-individual-rational}).

\subsection{Randomized Policy Framework}
\label{subsec:randomized-algorithmic-idea}

\textbf{Randomized policy framework.}
As explained in Section~\ref{subsec:determistic-algorithmic-idea}, policies for \masr should be state-aware.
While in the deterministic case, the state-awareness is implemented by picking state-dependent thresholds as in~\eqref{eq:threshold-overall} and~\eqref{eq:threshold-state-dependent},
for the randomized case, the thresholds need to be sampled and resampled from state-dependent distributions \(f(\ell)\) over the course of the algorithm.
That is, the randomized policies also have a structure similar to Figure~\ref{fig:hereogeneous-deterministic-policy},
where the fixed \(T(\ell)\)  should be replaced by sampled \(T(\ell) \sim f(\ell)\).
Due to the limited space, we defer the detailed algorithmic framework (Algorithm~\ref{alg:randomized-optimal-policy}) to the appendix.
The key differences of the randomized framework from the deterministic one (Algorithm~\ref{alg:deterministic-optimal-policy}) are:
(i) The input threshold \(T(\cdot)\) should be replaced by density function \(f(\cdot)\), and
(ii) whenever agents become inactive (including the initialization of \(T(0)\)), the threshold \(T(\ell)\) is sampled from \(f(\cdot)\).
We devise the specific functions \(f(\cdot)\) for all three \(\ratio\)s in the following subsections.


\textbf{Technical challenges for analyzing randomized policies.}
While the above randomized policy structure is clear after the investigation of the deterministic policies in Section~\ref{sec:deterministic-heterogeneous}, the competitive ratio analysis of the randomized policies is much more challenging than the deterministic case.
To derive the probability density functions \(f(\cdot)\) for each type of competitive ratios, one needs to address the following challenges:
(i) Although the design in Algorithm~\ref{alg:randomized-optimal-policy} implicitly assumes that all agents shall sample the same threshold \(T(\ell)\) at the same state, which is better than agents individually sampling their thresholds, this assumption needs to be justified by the competitive ratio analysis for randomized policies.
Especially, the state \(\{N_n\}_{n\le \ell}\) changes over time, and one needs to justify this for any number of remaining \(M-\ell\) active agents.
(ii) For a given state \(\{N_n\}_{n\le \ell}\), the paid cost \(\sum_{n=1}^{\ell} N_n\) introduces non-trivial complexity for the derivation of the probability \(f(\cdot)\) for the randomized policy.
Specifically, the original randomized policy for canonical ski-rental is often derived by solving a linear programming (LP), where the LHS of the linear constraints corresponds to the expected cost of the randomized policy, and the RHS is the cost of the optimal offline policy multiplied by \(\ratio\),
and the solution of this LP relies on the peculiar format of the constraints.
However, with the introduction of the \(\sum_{n=1}^{\ell} N_n\) term, at least one side of the linear constraints needs renovations (depending on the type of \(\ratio\)), and without the peculiar format, solving the new LPs is challenging.

\subsection{Dominance of Symmetric Randomized Policies and Lower Bounds}
\label{subsec:symmetric-and-lower-bound}

This section establishes the superiority of symmetric randomized policies over asymmetric ones and derives lower bounds for the competitive ratios.
While this is a parallel section to Section~\ref{subsec:symetric-superiority-deterministic} for deterministic policies, the proof techniques for randomized policies are much more challenging.
On the other hand, while the optimality of deterministic policies are supported by their optimal construction,
the optimality of randomized policies is supported by the asymptotic lower bounds for the competitive ratios,
both of which are common practices in the ski-rental literature~\cite{karlin1988competitive,karlin1994competitive}.

The following lemma establishes the superiority of symmetric randomized policies over asymmetric ones, which addresses the first technical challenge mentioned above.
\begin{lemma}\label{lemma:symmetric-better-than-asymmetric-randomized}
    Under any given state \(\{N_n\}_{n\le \ell}\), it is a symmetric randomized policy that minimizing the competitive ratio (of all three types of definitions).
\end{lemma}

\begin{proof}[Proof sketch of Lemma~\ref{lemma:symmetric-better-than-asymmetric-randomized}]
The full proof appears in Appendix~\ref{sec:symmetric-better-than-asymmetric-randomized}; here we outline the main ideas.  
The argument proceeds by \emph{strong induction} (the second principle of induction) on the pair \((M,\ell)\) shown in Figure~\ref{fig:randomized-policy-induction}.  
Each node \((M,T(\ell))\) represents an instance with \(M\) agents and current state \(\ell\) (i.e., \(\ell\) agents have become inactive).  
Three ingredients drive the induction:

\textbf{(i) Base case \(\boldsymbol{M=2}\).}  
When only two agents remain active, we formulate the problem as a linear program, simplify it, and show that the optimal randomized rule assigns identical thresholds to both agents—hence symmetry is optimal in this smallest non‑trivial case.

\textbf{(ii) Diagonal links (fixed number of active agents).}  
Diagonal arrows in Figure~\ref{fig:randomized-policy-induction} connect states with the same number of still‑active agents.  
Because the underlying optimization is identical along any such diagonal, the optimal threshold distribution computed for one node (say, \((1,T(0))\)) continues to be optimal—up to a relabeling—for all other nodes on the same diagonal (e.g., \((2,T(1))\), \((3,T(2))\), and so on).

\textbf{(iii) Horizontal links (incrementing} \(\boldsymbol{M}\)\textbf{).}  
Horizontal arrows connect instances in which the total number of active agents increases by one while the state resets to \(\ell=0\).  
Assume by the induction hypothesis that symmetry is optimal for every instance \((m,T(0))\) with \(m\le M\).  
To handle \((M+1,T(0))\), partition the agents into two non‑empty subgroups of sizes \(M_1\) and \(M_2\) with \(M_1+M_2=M+1\).  
By the hypothesis, each subgroup is governed optimally by a symmetric randomized policy.  
It remains to show that the optimal joint policy for two subgroups is also symmetric, which can be proved by similar approaches as in (i) above.  
\end{proof}

\begin{figure}[tb]
    \centering
            \begin{tikzcd}
            [
                column sep={8em,between origins},
                row sep   ={2em,between origins}
            ]
            (1, T(0)) \arrow[r, black] \arrow[dr]
            & [-2em] \textcolor{black}{\boxed{(2, T(0))}} \arrow[r] \arrow[dr]
            & [-2em] (3, T(0)) \arrow[r] \arrow[dr]
            & [-3em] \dots \arrow[r] \arrow[dr]
            &  [-2em] (M-1, T(0)) \arrow[r] \arrow[dr]
            & (M, T(0))
            \\
            {}
            & (2, T(1)) \arrow[dr]
            & (3, T(1)) \arrow[dr]
            & \dots \arrow[dr]
            & (M-1, T(1)) \arrow[dr]
            & (M, T(1))
            \\
            {}
            & {}
            & (3, T(2)) \arrow[dr]
            & \dots \arrow[dr]
            & (M-1, T(2)) \arrow[dr]
            & (M, T(2))
            \\
            {}
            & {}
            & {}
            & \ddots \arrow[dr]
            & \vdots \arrow[dr]
            & \vdots
            \\
            {}
            & {}
            & {}
            & {}
            & (M-1, T(M-2))  \arrow[dr]
            & (M, T(M-2))
            \\
            {}
            & {}
            & {}
            & {}
            & {}
            & (M, T(M-1))
        \end{tikzcd}
    \caption{Induction diagram for Lemma~\ref{lemma:symmetric-better-than-asymmetric-randomized}.  
    Each node \((M,T(\ell))\) denotes an instance with \(M\) agents and state \(\ell\).  
    The boxed node is the base case \(\,M=2\) (step (i)).  
    Diagonal arrows link nodes with the same number of active agents (used in step (ii)), while horizontal arrows increase \(M\) with \(\ell=0\) (step (iii)).}
    \label{fig:randomized-policy-induction}
\end{figure}
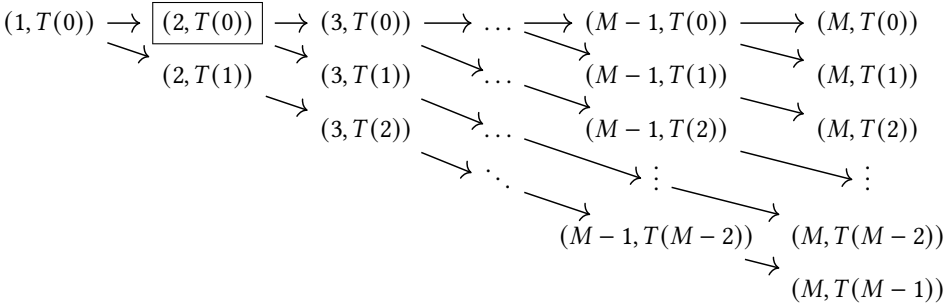

From Lemma~\ref{lemma:symmetric-better-than-asymmetric-randomized}, we derive the competitive ratio lower bound for randomized policies as follows,
\begin{lemma}\label{lemma:randomized-competitive-ratio-lower-bound}
    For any randomized policy for the \masr problem and any of the three types of competitive ratios, the competitive ratio is lower-bounded by \(\frac{e}{e-1}\) when the individual-buy cost \(B\) (hence the group-buy cost \(G\) as well) is sufficiently large.
\end{lemma}


The proof of Lemma~\ref{lemma:randomized-competitive-ratio-lower-bound} utilizes Yao’s principle~\cite{yao1977probabilistic} and a carefully crafted distribution supported on instances in which all agents have an identical number of active days.  Restricting attention to this subclass suffices for the lower‑bound analysis, for the same reason outlined in Section~\ref{subsec:symetric-superiority-deterministic} for deterministic policies.  The full proof is deferred to Appendix~\ref{sec:randomized-competitive-ratio-lower-bound-proof}.

\subsection{Optimal Randomized Policy for Overall Competitive Ratio}
\label{subsec:randomized-overall}

As illustrated in Section~\ref{subsec:deterministic-heterogeneous-policy-overall}, the overall \(\ratio\) accounts the paid cost \(\sum_{n=1}^{\ell} N_n\) as irrevocable losses, which leads to a threshold \(T\overall(\ell) = \min\{ {(G - \sum_{n=1}^\ell N_n)}/{(M-\ell)}, B \}\) for the deterministic policy.
This threshold also plays a key role for the randomized policy.
Specifically, by solving its corresponding linear programming problem, we can derive the probability density function \(f(\cdot)\) for the randomized policy as \(f_*\overall(t; \{N_n\}_{n\le \ell}) \coloneqq p_t(\{N_n\}_{n\le \ell}, T\overall(\ell))\), where the density \(p_t\) is defined as follows,
\begin{align}\label{eq:randomized-p-t}
     & p_t(\{N_n\}_{n\le \ell}, T(\ell))
    \! \coloneqq \!\begin{cases}
                       \frac{\left(\sum_{n=1}^\ell N_n + (M{-}\ell)(N_\ell {+} 1)\right)g}{\sum_{n=1}^\ell N_n + (M{-}\ell) (N_\ell {+} T(\ell))} \left( 1 {-} \frac{1}{T(\ell)} \right)^{T(\ell) {-} t}
                        & \!\!\!\!\! t = N_\ell + 1
                       \\
                       \frac{g}{T(\ell)} \left( 1 {-} \frac{1}{T(\ell)} \right)^{T(\ell) {-} t}
                        & \!\!\!\!\! t \in \left\{N_\ell {+} 2, \dots, \ceil*{T(\ell)}\right\}
                   \end{cases}
    \\\label{eq:competitive-ratio-group}
     & \text{where }\quad g = g(\{N_n\}_{n\le \ell}, T(\ell))
    =
    \frac{1}{1 - \frac{(M-\ell)(T(\ell) - 1)}{\sum_{n=1}^\ell N_n + (M-\ell) (N_\ell + T(\ell))}\left( 1 - \frac{1}{T(\ell)} \right)^{T(\ell) - N_\ell - 1 }}.
\end{align}

\begin{theorem}\label{thm:randomized-heterogeneous-policy-overall}
    The randomized policy of Algorithm~\ref{alg:randomized-optimal-policy} with input probability density \(f_*\overall(t; \{N_n\}_{n\le \ell})\) is optimal for minimizing overall competitive ratio,
    where the competitive ratio can be expressed as follows,
    \(
    \ovcr(f_*\overall \vert \{N_n\}_{n\le\ell})
    = g(\{N_n\}_{n\le \ell}, T\overall(\ell))
    \) defined in~\eqref{eq:competitive-ratio-group}.
\end{theorem}


The probability distribution and competitive ratio stated in Theorem~\ref{thm:randomized-heterogeneous-policy-overall} are obtained by solving a linear program that minimizes the overall competitive ratio, thereby certifying the optimality from construction.  
While the LP formulation follows the classical framework for devising optimal randomized ski‑rental policies~\citep{karlin1988competitive}, solving this new program requires a novel reduction that diverges from the standard approach; full details are given in Appendix~\ref{subsec:randomized-heterogeneous-policy-overall-proof}.  
Moreover, as the individual‑buy cost \(B\) grows to infinity, the competitive ratio in Theorem~\ref{thm:randomized-heterogeneous-policy-overall} converges to \(\tfrac{e}{e-1}\), matching the asymptotic lower bound in Lemma~\ref{lemma:randomized-competitive-ratio-lower-bound} and thereby providing an additional confirmation of optimality.

\subsection{Optimal Randomized Policy for State-Dependent Competitive Ratio}\label{subsec:randomized-state-dependent}

The key differences between overall and state-dependent \(\ratio\)s as shown in the deterministic case (Section~\ref{subsec:deterministic-heterogeneous-state-dependent}) is that the state-dependent \(\ratio\) accounts the paid cost \(\sum_{n=1}^{\ell} N_n\) as \emph{sunk costs}, yielding a threshold \(T\statedep(\ell) = \min\left\{ {G}/{(M-\ell)}, B \right\}\) independent of the paid cost.
With a similar derivation as in the overall case, we derive the probability density function for the randomized policy \(f_*\statedep(t; \{N_n\}_{n\le \ell}) \coloneqq p_t(\{N_n\}_{n\le \ell}, T\statedep(\ell))\), where the density \(p_t\) is defined in~\eqref{eq:randomized-p-t}.

\begin{theorem}\label{thm:randomized-heterogeneous-policy-state-dependent}
    The randomized policy of Algorithm~\ref{alg:randomized-optimal-policy} with input probability density \(f_*\statedep(t; \{N_n\}_{n\le \ell})\) is optimal for minimizing state-dependent competitive ratio,
    where the competitive ratio can be expressed as follows,
    \(
    \ovcr(f_*\statedep \vert \{N_n\}_{n\le\ell})
    = g(\{N_n\}_{n\le \ell}, T\statedep(\ell))
    \) defined in~\eqref{eq:competitive-ratio-group}.
\end{theorem}

The optimality of Theorem~\ref{thm:randomized-heterogeneous-policy-state-dependent} is guaranteed in the same way as Theorem~\ref{thm:randomized-heterogeneous-policy-overall}: the policy is optimal by construction (detailed in Appendix~\ref{subapp:proof-randomized-heterogeneous-policy-state-dependent}), and its ratio matches the \(\frac{e}{e-1}\) lower bound asymptotically.
The randomized policy for \(\ratio\statedep\) in Theorem~\ref{thm:randomized-heterogeneous-policy-state-dependent} is similar to that of \(\ratio\overall\) in Theorem~\ref{thm:randomized-heterogeneous-policy-overall}, with the only difference on the input thresholds.
This difference aligns with that of the deterministic case.
Notice that \(T\overall(0) = T\statedep(0) = \frac{G}{M}\)
and thus \(p_t(\emptyset, \frac{G}{M})\) is exactly the density function for
the randomized policy for the homogeneous \masr cases in Theorem~\ref{thm:homogeneous-randomized-policy}.
From this perspective, both policies in Theorems~\ref{thm:randomized-heterogeneous-policy-overall} and~\ref{thm:randomized-heterogeneous-policy-state-dependent} can be viewed as generalizations of the homogeneous case, regarding two different objectives.
We present numerical and empirical comparisons of both policies in Figures~\ref{subfig:randomized-policy-comparison} and~\ref{fig:hist-comparison-group}, respectively, demonstrating that both  competitive ratios are close but still distinct from each other.

\subsection{Optimal Randomized Policy for Individual Rational Competitive Ratio}
\label{subsec:randomized-individual-rational}

Solving the linear programming for optimizing individual rational \(\ratio\) leads to the probability density function
\(f_*\individual(t; \{N_n\}_{n\le \ell}) = q_t(\{N_n\}_{n\le \ell}, T\statedep(\ell))\), where the density \(q_t\) is defined as follows,
\begin{align}\label{eq:randomized-q-t}
     & q_t(\{N_n\}_{n\le \ell}, T(\ell))
    \!\coloneqq\!\!
    \begin{cases}
        \frac{h}{T(\ell)} \!\!\left(\!\! \frac{(N_\ell {+} 1)(T(\ell) {-} 1)}{N_\ell {+} T(\ell)} \!\!\left(\! 1 {-} \frac{1}{T(\ell)} \right)^{T(\ell) {-} N_\ell {-} 2}
        \!\!\!\!
        - \frac{N_\ell(T(\ell) {-} 1)}{N_\ell {+} T(\ell)} \!\!\right)
         & t = N_\ell + 1,
        \\
        \frac{h}{T(\ell)} \left(1\!-\!\frac{1}{T(\ell)} \right)^{T(\ell) - t}
         & \!\!\!\!\!\!\!\!\!\!\!\!\!\!\!\!\!\!\!\!\!\!\!\!\!\!\!\!\!\!\!\! t \in \left\{N_\ell {+} 2, \dots, \ceil*{T(\ell)}\right\},
    \end{cases}
    \\\label{eq:competitive-ratio-individual}
     & \text{where  }\quad h = h(\{N_n\}_{n\le \ell}, T(\ell))
    =
    \frac{1}{\frac{(T(\ell))^2 + N_\ell}{T(\ell) (T(\ell) + N_\ell)} - \frac{T(\ell) - 1}{T(\ell) + N_\ell}\left( 1 - \frac{1}{T(\ell)} \right)^{T(\ell) - N_\ell - 1 }}.
\end{align}

\begin{theorem}\label{thm:randomized-heterogeneous-policy-individual-rational}\label{thm:randomized-heterogeneous-policy-individual}
    The randomized policy of Algorithm~\ref{alg:randomized-optimal-policy} with input probability density \(f_*\individual\) is optimal for minimizing individual rational competitive ratio,
    where the competitive ratio for agent \(m\) can be expressed as
    \(
    \indcr_m(f_*\individual \vert \{N_n\}_{n\le m-1})
    = h(\{N_n\}_{n\le m-1}, T\statedep(m-1))
    \) defined in~\eqref{eq:competitive-ratio-individual}.
\end{theorem}

The optimality of Theorem~\ref{thm:randomized-heterogeneous-policy-individual-rational} is by its LP construction (detailed in Appendix~\ref{subapp:proof-randomized-heterogeneous-policy-individual}), and its competitive ratio also asymptotically matches the \(\frac{e}{e-1}\) ratio lower bound in Lemma~\ref{lemma:randomized-competitive-ratio-lower-bound}.
The specific density function \(q_t\) for \(f_*\individual\) is different from the \(p_t\) for \(f_*\statedep\) and \(f_*\overall\).
This difference comes from their different LP formulations: the LP for the individual rational \(\ratio\) is for any individual active agent, while that for the state-dependent \(\ratio\) is for the group of all remaining agents.
The policy \(f_*\individual\) also depends on the state-dependent threshold \(T\statedep(\ell)\) on which the deterministic policy \(\pi_*\individual\) in Section~\ref{subsec:deterministic-individual-rational} also depends.
This identical threshold \(T\statedep(\ell)\) between \(f_*\individual\) and \(f_*\statedep\) is due to both considering the paid cost as sunk costs.



\section{Experiments}
\label{sec:experiments}

In this section, we report the empirical performance of our proposed algorithms on the multi-agent ski-rental problem.

\textbf{Algorithms:}
We denote the proposed algorithms by the label, ``ALG Ratio'', where ``ALG'' is the type of the algorithm where ``Det'' stands for deterministic and ``Rand'' stands for randomized, and ``Ratio'' is the type of competitive ratio where ``OV'' stands for overall, ``SD'' stands for state-dependent, and ``IND'' stands for individual rational.
Besides the proposed six policies, we also consider two baseline algorithms with fixed buy thresholds (i.e., ``state-unaware'', in comparison to our ``state-aware'' policies): ``Det Fix'' with a buy threshold of \(\frac{G}{M}\) to buy the group pass,
and ``Rand Fix'' with a group pass buy threshold sample from the truncated distribution with rate \(\frac{G}{M}\).
We recall that ``Det SD'' and ``Det IND'' refer to the same policy, corresponding to the state-dependent threshold \(T\statedep\), and hence their performance is the same (with overlapped curves).

\begin{figure}[t]
    \centering
    \begin{subfigure}{0.99\textwidth}
        \centering
        \includegraphics[width=\linewidth]{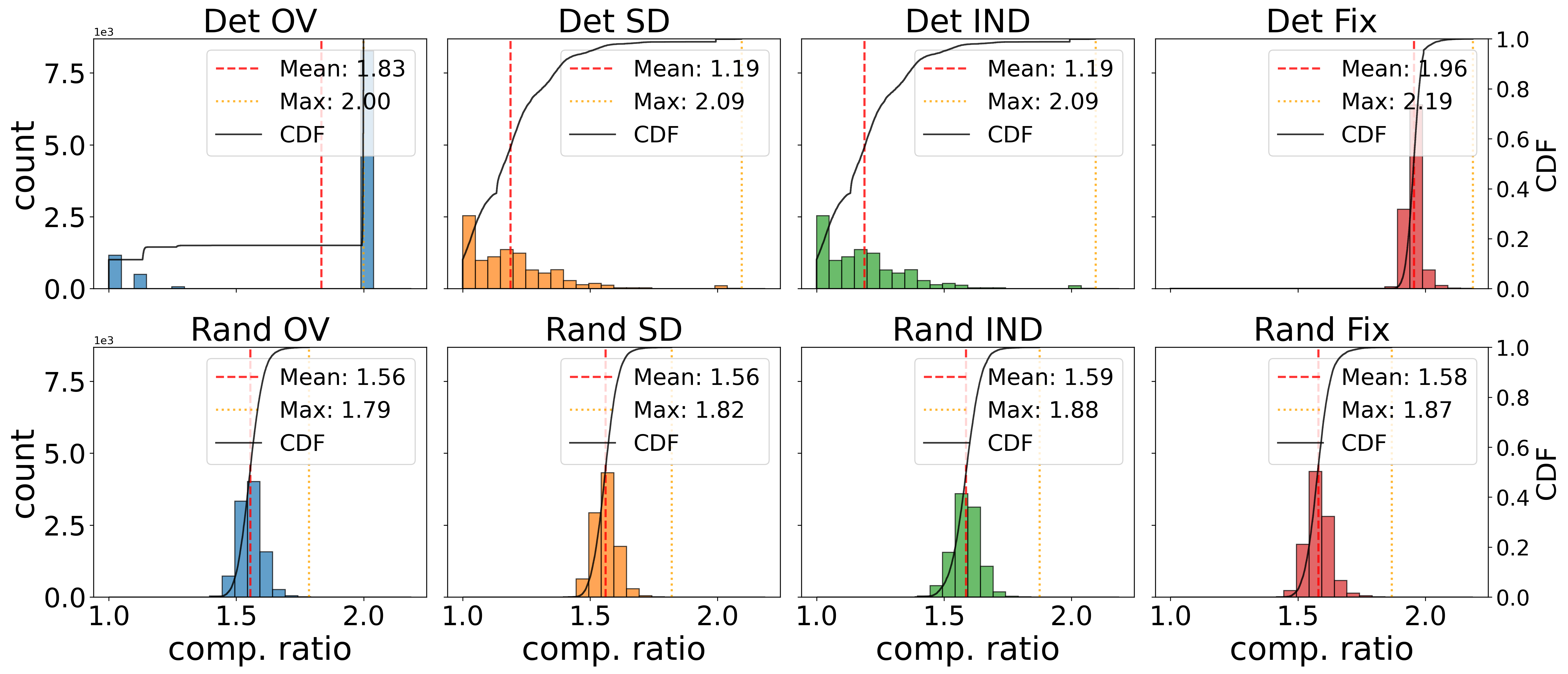}
        \caption{For overall competitive ratio \(\ovcr\)}
        \label{subfig:hist-ov}
    \end{subfigure}
    \hspace{0.2em}
    \\
    \begin{subfigure}{0.99\textwidth}
        \centering
        \includegraphics[width=\linewidth]{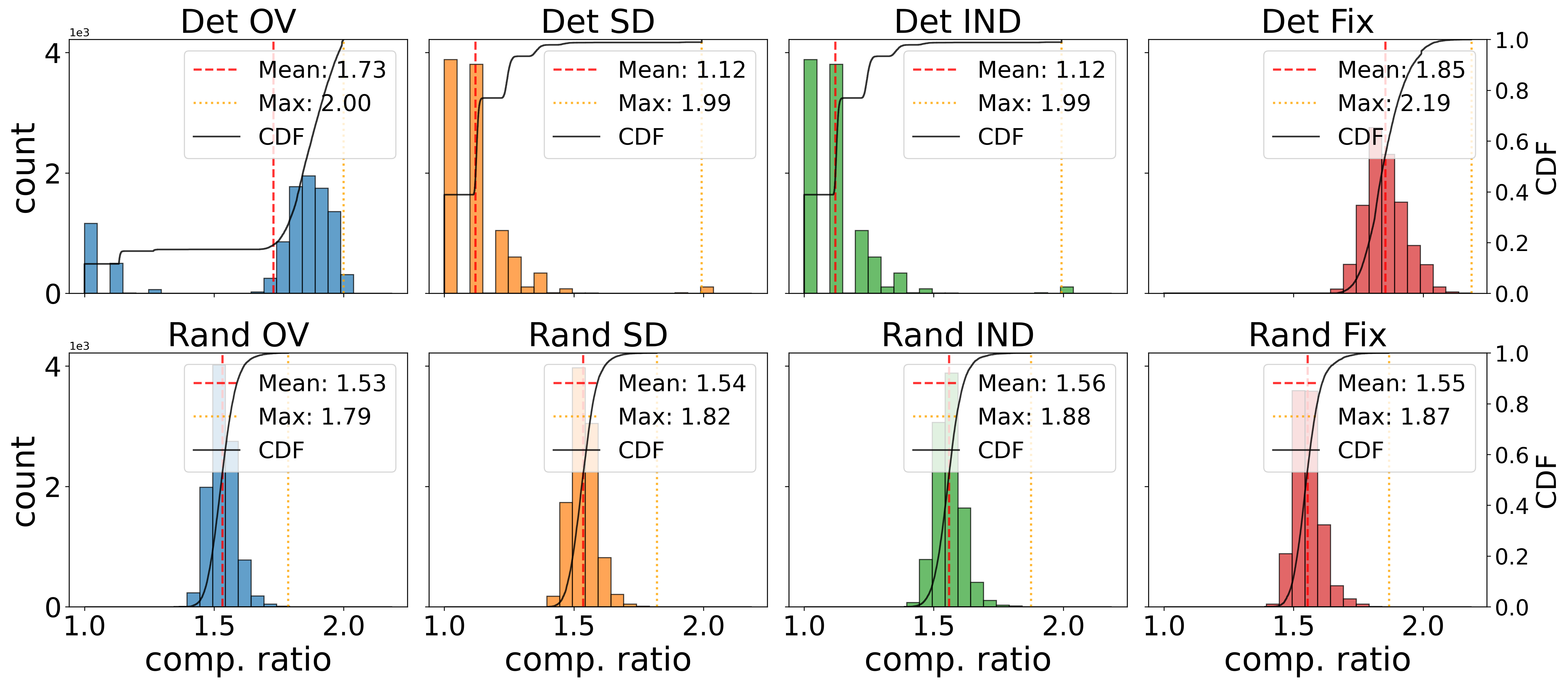}
        \caption{For state-dependent competitive ratio \(\sdcr\)}
        \label{subfig:hist-sd}
    \end{subfigure}
    \caption{Histogram and empirical cumulative distribution functions (CDFs) of two group competitive ratios: ``Max'' in the legend refers to the worst-case competitive ratio, while ``Mean'' refers to average performance.}\label{fig:hist-comparison-group}
\end{figure}

\begin{figure}[t]
    \centering
    \begin{subfigure}{0.99\textwidth}
        \centering
        \includegraphics[width=\linewidth]{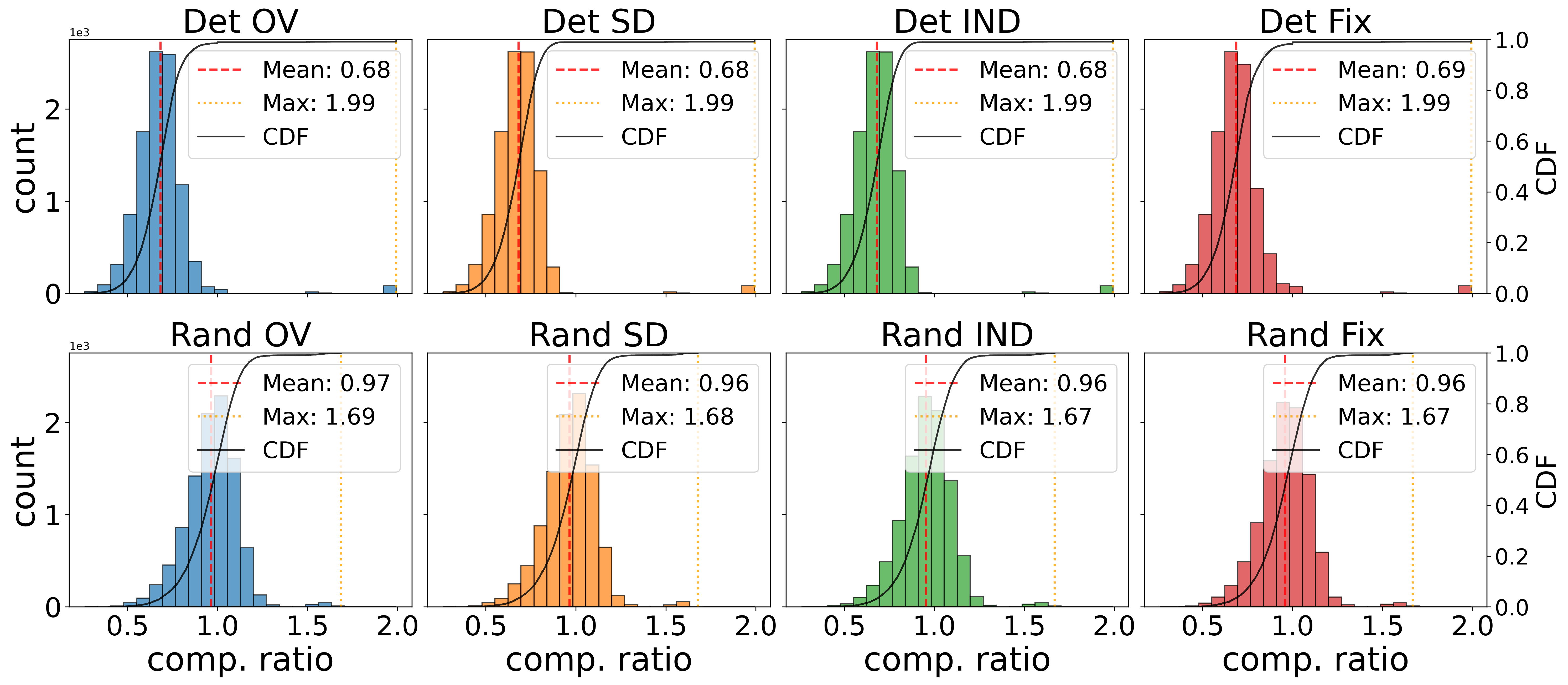}
        \caption{For minimum individual rational competitive ratio \(\indcr_{\min}\)}
        \label{subfig:hist-ind-min}
    \end{subfigure}
    \\
    \begin{subfigure}{0.99\textwidth}
        \centering
        \includegraphics[width=\linewidth]{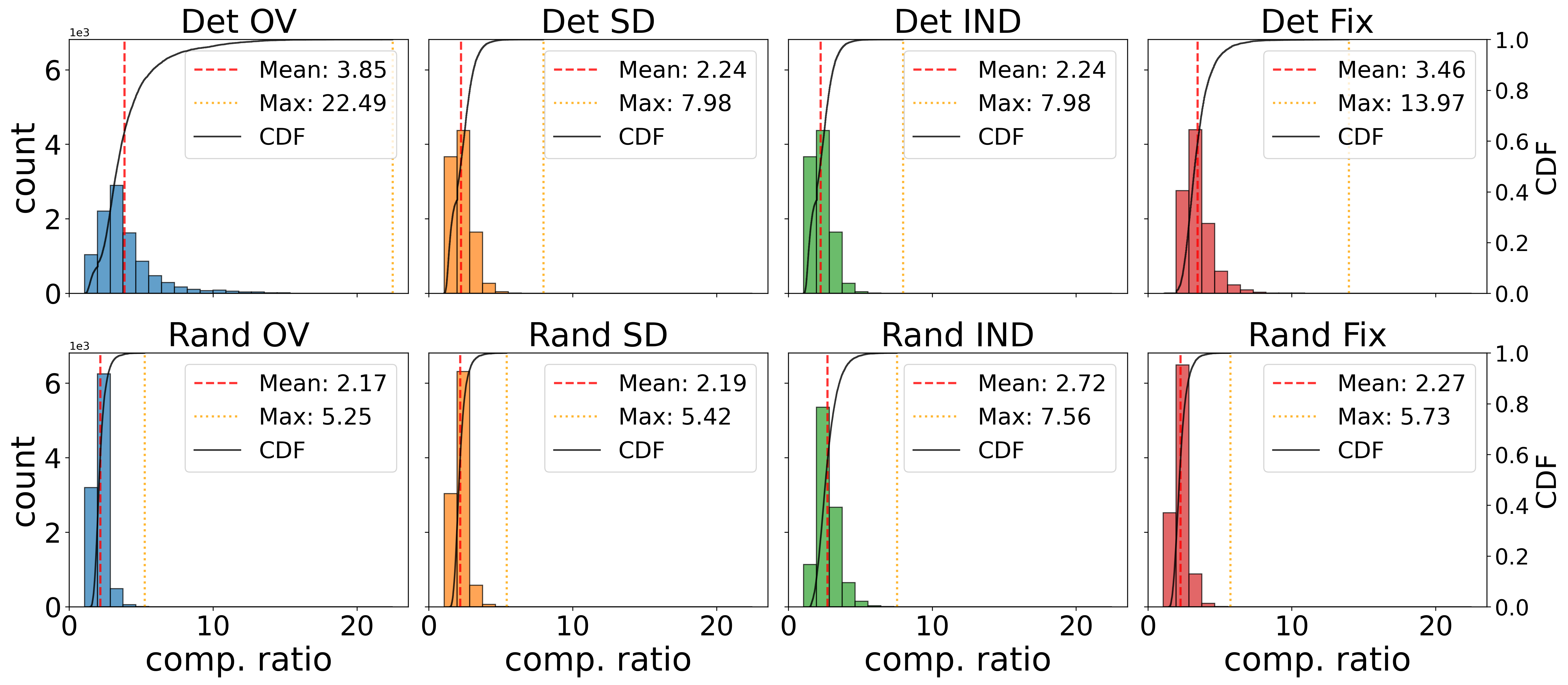}
        \caption{For maximum individual rational competitive ratio \(\indcr_{\max}\)}
        \label{subfig:hist-ind-max}
    \end{subfigure}
    \caption{Histogram and empirical cumulative distribution functions (CDFs) of individual rational competitive ratios: ``Max'' in the legend refers to the worst-case competitive ratio, ``Mean'' refers to average performance.}\label{fig:hist-comparison-ind}
\end{figure}

\textbf{Setup:} We consider a \masr problem consisting of \(M=10\) agents, each with a rental cost of \(1\), an individual-buy cost of \(B = 200\), and the group-buy cost is \(G = 1500\).
We randomly sample the active days of the agents from a normal distribution with mean \(160\) and standard deviation \(\sigma = 30\), truncated to the range \([50, 300]\).
The experimental results for other distribution configurations are provided in Appendix~\ref{app:additional-experiments}.
Each of the experiments is repeated \(10,\!000\) times. The average costs and the empirical competitive ratios of these algorithms are reported.

\textbf{Metrics:}
While the overall and state-dependent competitive ratios are in terms of the whole group,
the individual rational competitive ratio \(\indcr_{m}\) is defined for each agent individually, and we report the maximum and minimum among all agents, i.e., \(\indcr_{\max}\coloneqq \max_{m\in\mathcal{M}}\indcr_m\) and \(\indcr_{\min}\coloneqq \min_{m\in\mathcal{M}}\indcr_m\).
We note that our proposed algorithms are optimal in the sense of minimizing the worst-case competitive ratios, which corresponds to the highest realized competitive ratio (Max), \emph{instead of the mean of them.}
Nevertheless, the mean can reflect the average performance of the proposed algorithms in a relatively benign environment.
Besides the three types of competitive ratios, we also report the average accumulated group cost of these algorithms.

\textbf{Experimental Results:}
Figures~\ref{fig:hist-comparison-group} and~\ref{fig:hist-comparison-ind} present the histograms and cumulative distribution functions (CDFs) of the empirical competitive ratios for all eight evaluated algorithms. The results are shown for four metrics: the overall competitive ratio \(\ovcr\), the state-dependent competitive ratio \(\sdcr\), and the individual rational competitive ratios \(\indcr_{\min}\) and \(\indcr_{\max}\).

In Figure~\ref{subfig:hist-ov}, which illustrates the overall competitive ratio, we observe that ``Det OV'' and ``Rand OV'' achieve the lowest worst-case (Max) competitive ratios among the deterministic and randomized algorithms, respectively. This supports the optimality of the proposed algorithms concerning the overall competitive ratio.
In Figure~\ref{subfig:hist-sd}, which shows the state-dependent competitive ratio, ``Det SD'' attains the lowest worst-case competitive ratio among all deterministic algorithms. Among the randomized algorithms, ``Rand SD'' and ``Rand OV'' perform the best, with ``Rand OV'' showing slightly better results due to limited sampling, since the empirical maximum of the competitive ratio is only an approximation of the actual worst-case performance.
In terms of average performance (Mean), ``Det SD''---which is equivalent to ``Det IND''---achieves the lowest average overall and state-dependent competitive ratios among deterministic algorithms. Meanwhile, ``Rand OV'' yields the lowest averages among randomized algorithms.

Figure~\ref{fig:hist-comparison-ind} presents the histogram and CDF of the minimum and maximum individual rational competitive ratios. Recall from~\eqref{eq:indopt} that the individual rational benchmark is tailored to minimize the cost for agents with fewer active days. Therefore, the minimum individual rational competitive ratio \(\indcr_{\min}\) is the most relevant metric for evaluating algorithmic performance under individual rationality, while \(\indcr_{\max}\) reflects the worst-case trade-off required to satisfy this benchmark.
Figure~\ref{subfig:hist-ind-min} shows that ``Rand IND'' achieves the lowest worst-case (Max) value of \(\indcr_{\min}\) among the randomized algorithms, while all deterministic algorithms exhibit similar worst-case performance. This validates the optimality of the proposed algorithms in minimizing the individual rational competitive ratio.
Finally, Figure~\ref{subfig:hist-ind-max} reveals that ``Rand IND'' incurs the highest worst-case (Max) value of \(\indcr_{\max}\) among the randomized algorithms, illustrating the cost incurred by agents with longer activity durations to maintain individual rationality.

\begin{figure}[t]
    \centering
    \begin{subfigure}{0.6\textwidth}
        \centering
        \includegraphics[width=\linewidth]{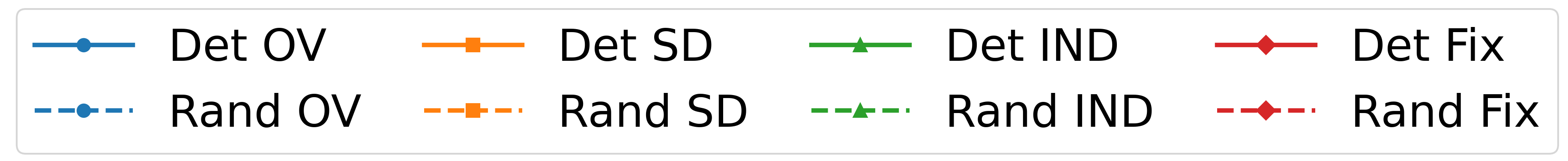}
    \end{subfigure}
    \\
    \begin{subfigure}{0.245\textwidth}
        \centering
        \includegraphics[width=\linewidth]{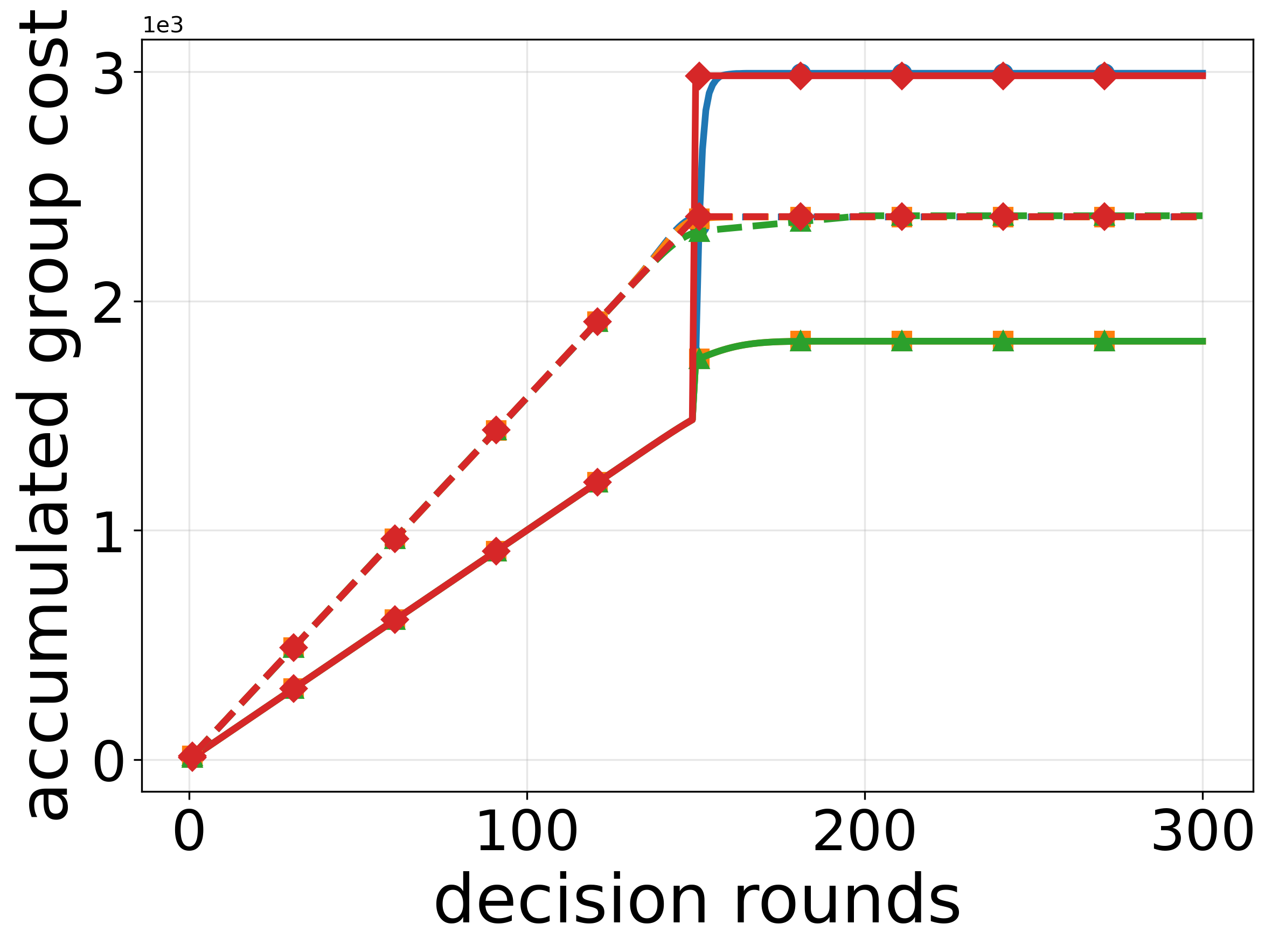}
        \caption{\(\sigma=10\)}
        \label{subfig:group-cost-10}
    \end{subfigure}
    \hfill
    \begin{subfigure}{0.245\textwidth}
        \centering
        \includegraphics[width=\linewidth]{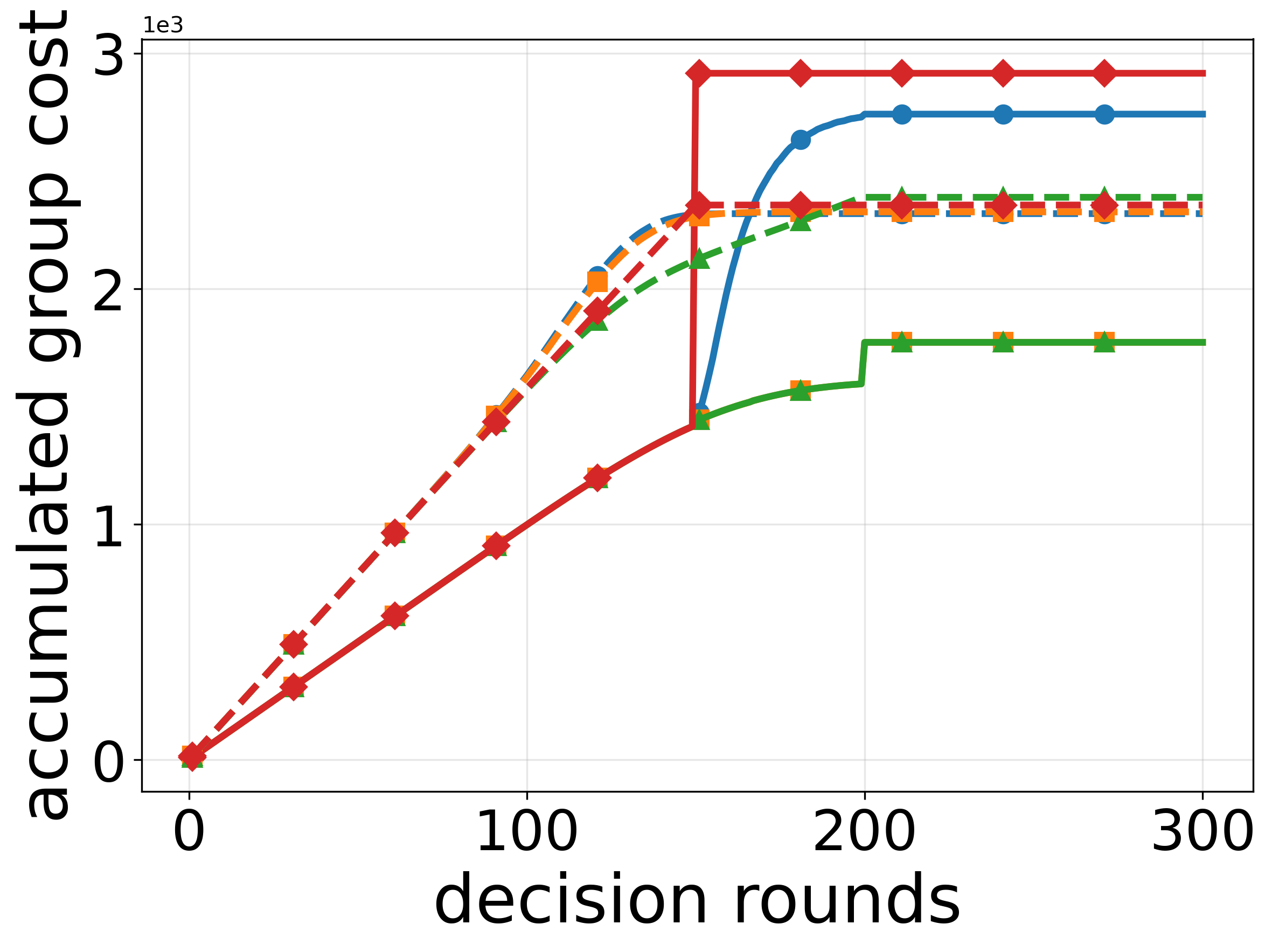}
        \caption{\(\sigma=30\)}
        \label{subfig:group-cost-30}
    \end{subfigure}
    \hfill
    \begin{subfigure}{0.245\textwidth}
        \centering
        \includegraphics[width=\linewidth]{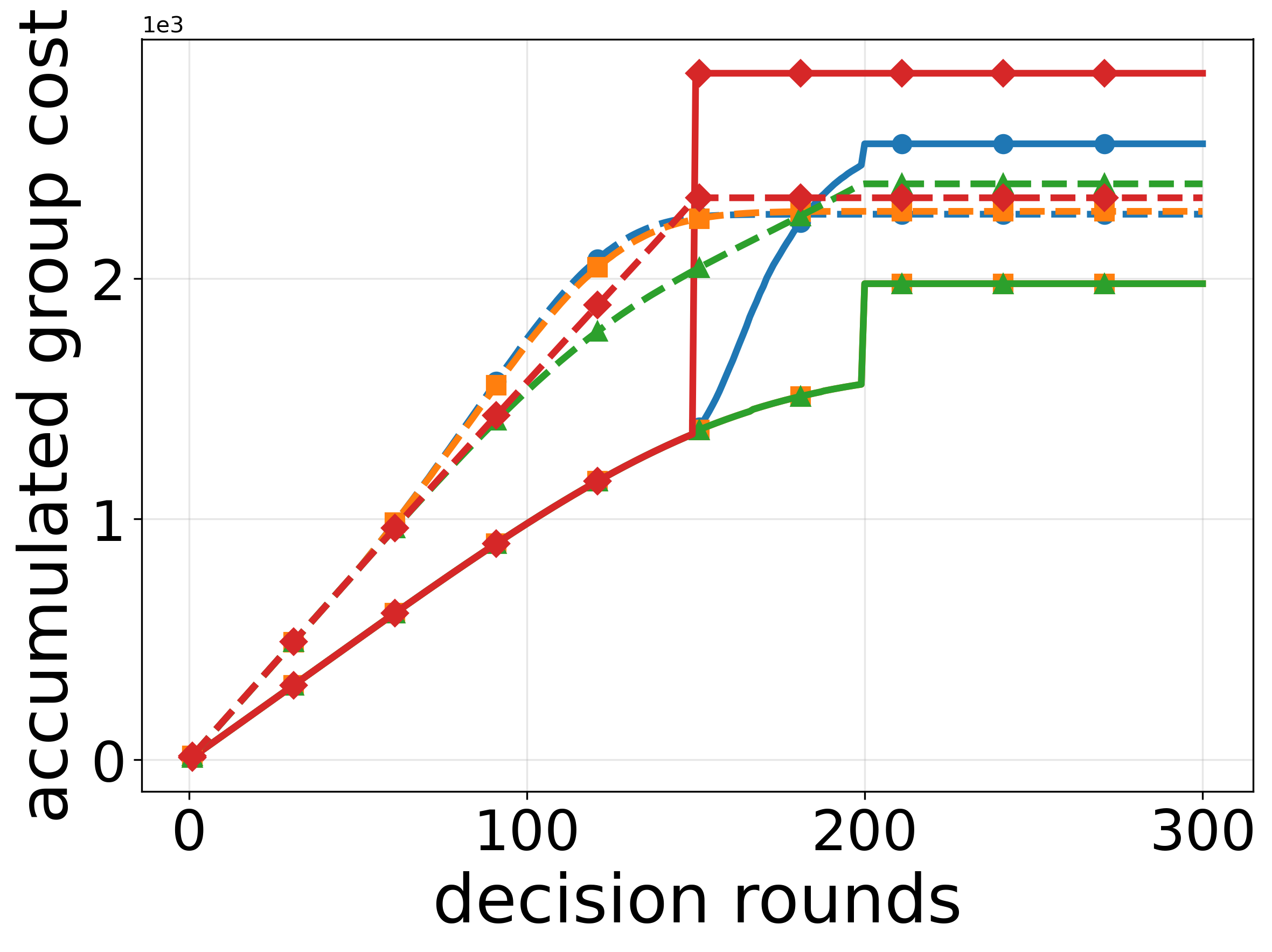}
        \caption{\(\sigma=50\)}
        \label{subfig:group-cost-50}
    \end{subfigure}
    \hfill
    \begin{subfigure}{0.245\textwidth}
        \centering
        \includegraphics[width=\linewidth]{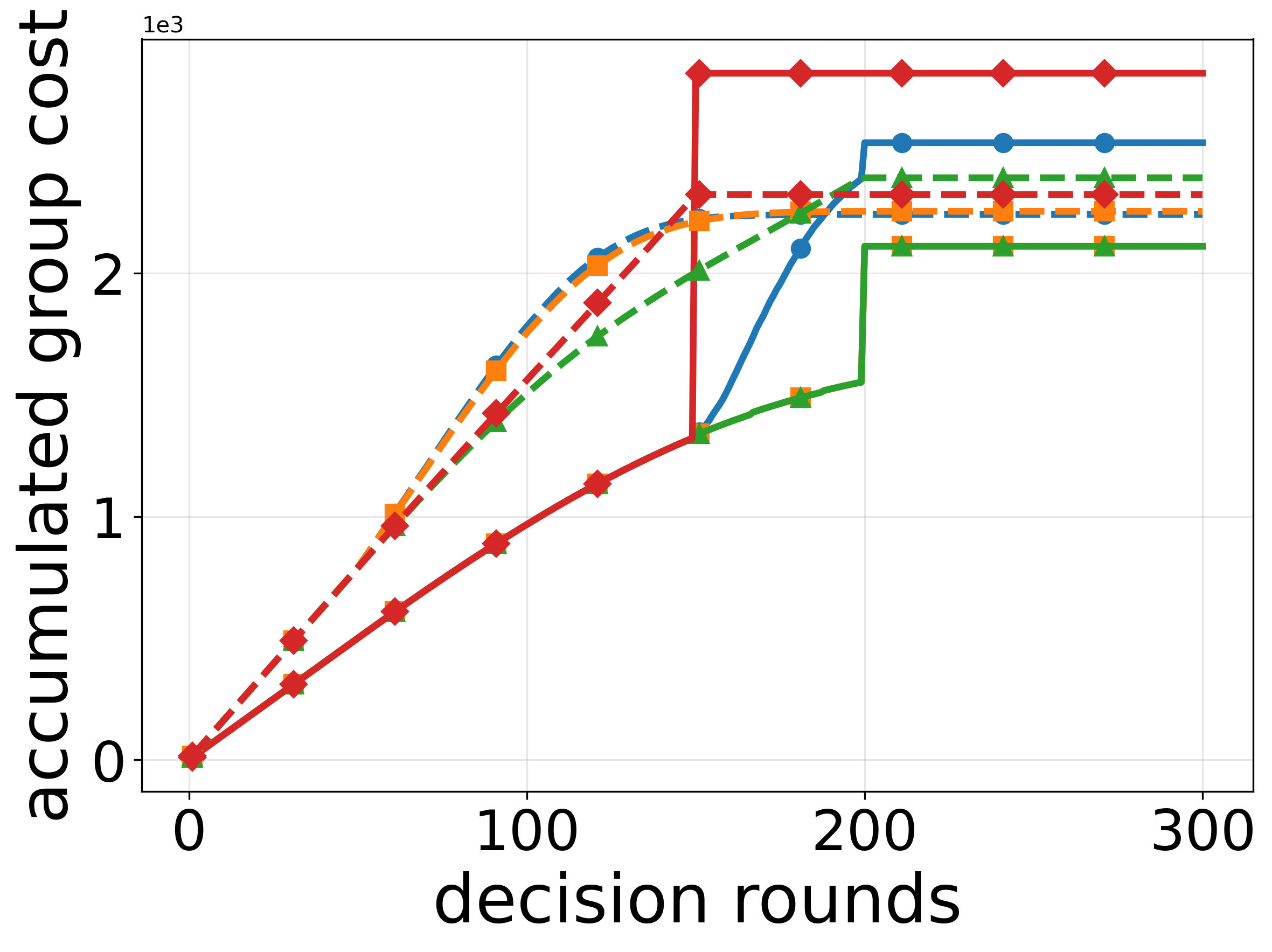}
        \caption{\(\sigma=70\)}
        \label{subfig:group-cost-70}
    \end{subfigure}
    \caption{Accumulative cost of all agents of algorithms for a normal distribution of active days with different standard deviations \(\sigma\).}\label{fig:group-cost}
\end{figure}

Figure~\ref{fig:group-cost} shows the cumulative group cost, averaged over \(10{,}000\) trials, for all algorithms at four heterogeneity levels \(\sigma\in\{10,30,50,70\}\) (all other parameters remain as in the setup).
Across every setting, the proposed ``Det~SD'' policy incurs the lowest cost, echoing its best mean competitive ratio in Figure~\ref{fig:hist-comparison-group} and underscoring its practical appeal.
When activity durations are nearly uniform (\(\sigma=10\); Figure~\ref{subfig:group-cost-10}), the baseline policies ``Det~Fix'' and ``Rand~Fix'' perform almost as well as the adaptive designs, since little benefit arises from state awareness.
As heterogeneity increases (\(\sigma=50,70\)), the adaptive policies—``Det~OV'', ``Det~SD'', and ``Rand~OV''—increasingly outstrip the baselines.
Among randomized methods, \texttt{Rand~IND} has the highest cost, consistent with its focus on the individual‑rational competitive ratio rather than group‑level efficiency.

\section{Conclusion}


We have introduced the multi-agent ski-rental problem (\masr), a natural generalization of the classical ski-rental problem that incorporates a discounted \emph{group-purchase} option prevalent in real-world scenarios.
Three performance metrics were examined---\emph{overall}, \emph{state-dependent}, and \emph{individually rational} competitive ratios---of which only the first reduces to the single-agent standard competitive ratio, while the latter two capture nuances unique to multi-agent interaction.
For each metric, we devised and analyzed optimal deterministic and randomized policies.
Our results reveal that symmetry is almost always optimal: in heterogeneous settings and for all randomized policies, the best strategy assigns identical thresholds to every active agent.
The lone exception arises in the deterministic homogeneous case, where agents share the same activity horizon, and a carefully tuned asymmetric rule can match or surpass its symmetric counterpart.
Central to the analysis is the notion of \emph{state}---the vector of active-day counts disclosed by the $\ell$ agents that have already become inactive.
All optimal policies are necessarily \emph{state-aware}: whenever the state updates, a deterministic policy must revise its purchase threshold, whereas a randomized policy must resample that threshold from a freshly adjusted distribution.
These findings extend the foundational insights of ski-rental to cooperative online decision-making and provide a blueprint for tackling more intricate rent-buy trade-offs in future multi-agent systems.


\textbf{Future directions.}
Although this paper provides a comprehensive treatment of the \masr\ problem, several compelling research avenues remain open.
\emph{Cost sharing.}
Our model assumes that agents who participate in a group-buy pass divide the group-buy cost evenly.
While this assumption leaves the overall and state-dependent competitive ratios unchanged, it obscures questions of fairness at the individual level.
Designing principled cost-sharing mechanisms---drawing, for instance, on proportional cost splitting, Shapley-value ideas, or incentive-compatible mechanisms---would make the framework more realistic and robust.
\emph{Richer rent–buy trade-offs.}
Many practical scenarios exhibit rent–buy structures more intricate than classical ski-rental.
A paradigmatic example is the \emph{Bahncard} problem~\citep{fleischer2001bahncard}, in which purchasing a pass confers only temporary discounts.
Extending the concept of group-buy to such settings and analyzing the resulting multi-agent online algorithms is a natural next step.
\emph{Learning augmentation.}
Recent work on learning-augmented algorithms has yielded concrete performance gains in online optimization, including ski-rental with predictions~\citep{wei2020optimal,purohit2018improving}.
Investigating how similar predictive signals can be leveraged to enhance performance in \masr\ constitutes another promising direction.







\bibliographystyle{ACM-Reference-Format}
\bibliography{bibliography}

\appendix

\newpage
\appendix

\section*{Appendix}
The appendix is organized as follows, and Table~\ref{tab:summary-ratio} summarizes the competitive ratios of all algorithms discussed in this paper.
\begin{itemize}
    \item In Appendix~\ref{app:individual-rational-offline-policy-proof}, we prove the individual rational offline policy in Section~\ref{sec:model-formulation} (Proposition~\ref{prop:indopt}).
    \item In Appendix~\ref{app:homogeneous-randomized-policy}, we prove the optimal homogeneous randomized policy (Theorem~\ref{thm:homogeneous-randomized-policy}).
    \item In Appendix~\ref{app:heterogeneous-determinisitic-policy-proof}, we present the proofs for the optimal deterministic policies for all three competitive ratios. Specifically,
          \begin{itemize}
              \item Appendix~\ref{sec:symmetric-better-than-asymmetric-deterministic-proof} proves there exists a symmetric deterministic policy is better than any given asymmetric deterministic policy (Lemma~\ref{lemma:symmetric-better-than-asymmetric-deterministic})
              \item Appendix~\ref{subapp:proof-determinisitic-overall-competitive-ratio} proves the overall competitive ratio (Theorem~\ref{thm:deterministic-heterogeneous-policy-overall})
              \item Appendix~\ref{subapp:proof-determinisitic-state-dependent-competitive-ratio} proves the state-dependent competitive ratio (Theorem~\ref{thm:deterministic-heterogeneous-policy-state-dependent})
              \item Appendix~\ref{subapp:proof-determinisitic-individual-rational-competitive-ratio} proves the individual rational competitive ratio (Theorem~\ref{thm:deterministic-heterogeneous-policy-individual-rational})
          \end{itemize}
    \item In Appendix~\ref{app:heterogeneous-randomized-policy-proof}, we present the proofs for the optimal  randomized policies for all three competitive ratios. Specifically,
          \begin{itemize}
              \item Appendix~\ref{subapp:basic-proof-for-randomized-policy} reviews the detailed procedures of the linear programming (LP) construction for the homogeneous \masr (Theorem~\ref{thm:homogeneous-randomized-policy}) as basics for the proofs in the rest of this appendix.
              \item Appendix~\ref{sec:symmetric-better-than-asymmetric-randomized} proves there exists a symmetric randomized policy is better than any given asymmetric randomized policy (Lemma~\ref{lemma:symmetric-better-than-asymmetric-randomized})

              \item Appendix~\ref{subapp:proof-randomized-overall-competitive-ratio} proves the overall competitive ratio (Theorem~\ref{thm:randomized-heterogeneous-policy-overall})
              \item Appendix~\ref{subapp:proof-randomized-state-dependent-competitive-ratio} proves the state-dependent competitive ratio (Theorem~\ref{thm:randomized-heterogeneous-policy-state-dependent})
              \item Appendix~\ref{subapp:proof-randomized-individual-rational-competitive-ratio} proves the individual rational competitive ratio (Theorem~\ref{thm:randomized-heterogeneous-policy-individual-rational})
          \end{itemize}
    \item In Appendix~\ref{app:lower-bound}, we prove the competitive ratio lower bounds for \masr. \begin{itemize}
              \item Appendix~\ref{app:deterministic-lower-bound} discusses the lower bounds for deterministic policies from the perspective of optimality by construction.
              \item Appendix~\ref{app:randomized-lower-bound} proves the lower bounds for randomized policies (Lemma~\ref{lemma:randomized-competitive-ratio-lower-bound}).
          \end{itemize}
    \item In Appendix~\ref{app:deterministic-CR-trend}, we provide more discussions and details on the numerical results in Figure~\ref{fig:deterministic-CR-trend}.
    \item In Appendix~\ref{app:additional-experiments}, we provide additional experiments to validate the effectiveness of our proposed algorithms.
\end{itemize}

\begin{table}[tbh]
    \centering
    \caption{Competitive ratios:
        \(T\upbra{\text{OV}}(\ell) = \min\left\{{\frac{G - \sum_{m=1}^\ell N_m}{M-\ell}}, B\right\} \) and \(T\upbra{\text{SD}}(\ell) = \min\left\{{\frac{G}{M-\ell}}, B\right\}\)}\label{tab:summary-ratio}
    \resizebox{\textwidth}{!}{
        \begin{tabular}{|c|l|c|}
            \hline
            \multicolumn{2}{|c|}{Algorithm Type}
             & \textbf{Deterministic}
            \\
            \hline
            \multicolumn{2}{|c|}{Single-agent}
             & \(\displaystyle 1 + \frac{B-1}{B}\)
            \\
            \hline
            \multicolumn{2}{|c|}{Homogeneous}
             & \(\displaystyle 1 + \frac{G-M}{G}\) (Thm.~\ref{thm:homogeneous-deterministic-policy})
            \\\hline
            \multirow[c]{3}{*}[-0.03in]{Hetero.}
             & Overall
             & \(\displaystyle  1 +
            \frac{\min\left\{
                G, (M-\ell) B
                \right\}
                - (M-\ell)
            }{
                \min\left\{
                G, \sum_{n=1}^{\ell} N_n + (M-\ell) B
                \right\}
            }\) (Thm.~\ref{thm:deterministic-heterogeneous-policy-overall})
            \\\cline{2-3}
             & St.-dep.
             & {\(\displaystyle 1 + \frac{\min\{G, (M-\ell)B\} - (M-\ell)
                    }{
                        \sum_{n=1}^{\ell} N_n
                        + \min\{G, (M-\ell)B\}
                    }\) (Thm.~\ref{thm:deterministic-heterogeneous-policy-state-dependent})}
            \\\cline{2-3}
             & Ind.
             & {\(\displaystyle
                    \begin{cases}
                        1
                         & \text{ if } m \le \ell_*
                        \\
                        2 - \frac{1}{T\statedep(\ell_*)}
                         & \text{ if } m > \ell_*
                    \end{cases}\) (Thm.~\ref{thm:deterministic-heterogeneous-policy-individual})}
            \\
            \hline
            \multicolumn{2}{|c|}{Algorithm Type}
             & \textbf{Randomized}
            \\
            \hline
            \multicolumn{2}{|c|}{Single-agent}
             & \(\left( 1 - \left( 1 - \frac{1}{B} \right)^{B} \right)^{-1}\)
            \\
            \hline
            \multicolumn{2}{|c|}{Homogeneous}
             & \(\left( 1 - \left( 1 - \frac{M}{G} \right)^{\frac{G}{M}} \right)^{-1}\)  (Thm.~\ref{thm:homogeneous-randomized-policy})
            \\\hline
            \multirow[c]{3}{*}[-0.03in]{Hetero.}
             & Overall
             & \(\displaystyle \left( 1 - \frac{(M-\ell)(T\overall(\ell) - 1)}{\sum_{m=1}^\ell N_m + (M-\ell) (N_\ell + T(\ell))}\left( 1 - \frac{1}{T\overall(\ell)} \right)^{T\overall(\ell) - N_\ell - 1 } \right)^{-1}\) (Thm.~\ref{thm:randomized-heterogeneous-policy-overall})
            \\\cline{2-3}
             & St.-dep.
             & \(\displaystyle \left( 1 - \frac{(M-\ell)(T\statedep(\ell) - 1)}{\sum_{m=1}^\ell N_m + (M-\ell) (N_\ell + T\statedep(\ell))}\left( 1 - \frac{1}{T\statedep(\ell)} \right)^{T\statedep(\ell) - N_\ell - 1 } \right)^{-1}\) (Thm.~\ref{thm:randomized-heterogeneous-policy-state-dependent})
            \\\cline{2-3}
             & Ind.
             & \(\displaystyle \left( \frac{(T\statedep(m))^2 + N_m}{T\statedep(m) (T\statedep(m) + N_m)} - \frac{T\statedep(m) - 1}{T(m) + N_m}\left( 1 - \frac{1}{T\statedep(m)} \right)^{T(m) - N_m - 1 } \right)^{-1}\) (Thm.~\ref{thm:randomized-heterogeneous-policy-individual-rational})
            \\
            \hline
        \end{tabular}
    }
\end{table}

\section{Proof for Individual Rational Offline Policy in Section~\ref{sec:model-formulation}: Proposition~\ref{prop:indopt}} \label{app:individual-rational-offline-policy-proof}

This Nash equilibrium is derived iteratively.
First from the perspective of agent \(1\) with the minimum active days, it is better off to rent for all active days if \(N_1 \le \min \left\{{G}/{(M-0)}, B\right\}\), and no other agents' decisions can affect agent \(1\)'s.
This follows a standard argument in general-sum game for Nash equilibrium~\citep[Chapter~4]{karlin2017game}. Here we take two-agent ski-rental with active days \(N_1 < N_2\) as an example, but the same derivation applies to any number of agents. In offline setting (know \(N_1, N_2\)) and at the beginning of this game, each of the agents \(1\) and \(2\) can choose from rent, buy individually, or buy as a group,
yielding the payoff matrix for the two agents as shown in Table~\ref{tab:payoff-matrix}.
From this payoff matrix, we can see no matter what agent \(2\) chooses, agent \(1\) is always better off to rent if \(N_1 \le \min \left\{ {G}/2, B \right\}\).
\begin{table}[htpb]
    \centering
    \caption{Payoff matrix for two agents with the active days \(N_1<N_2\).}\label{tab:payoff-matrix}
    \begin{tabular}{lccc}
        \hline
         & \(\text{Rent}_2\)
         & \(\text{Individual Buy}_2\)
         & \(\text{Group Buy}_2\)
        \\
        \hline
        \(\text{Rent}_1\)
         & \((N_1, N_2)\)
         & \((N_1, B)\)
         & \((N_1, G)\)
        \\
        \(\text{Individual Buy}_1\)
         & \((B, N_2)\)
         & \((B, B)\)
         & \((B, G)\)
        \\
        \(\text{Group Buy}_1\)
         & \((G, N_2)\)
         & \((G, B)\)
         & \(\left( \frac{G}{2}, \frac{G}{2} \right)\)
        \\\hline
    \end{tabular}
\end{table}

Then, we can iteratively apply the same reasoning to the agent \(2\) with the second minimum active days \(N_2\), and so on,
until reaching the agent \(\ell^*+1\) who is better off to buy the ski pass (either with the remaining agents for the group pass or separately for the individual pass), i.e., \(N_{\ell^*+1} > \min \left\{ {G}/{(M-\ell^*)}, B \right\}\).
That is, for all remaining agents \(m>\ell^*\), they are better off to buy the ski pass, which corresponds to the second case in~\eqref{eq:indopt}.

\section{Proofs for Homogeneous Multi-Agent Ski-Rental: Randomized Policy in Theorem~\ref{thm:homogeneous-randomized-policy}}
\label{app:homogeneous-randomized-policy}
\textbf{Technical novelty of Theorem~\ref{thm:homogeneous-randomized-policy}'s proof.} Although we intuitively illustrate that the asymmetric randomized policy is suboptimal, the proof for Theorem~\ref{thm:homogeneous-randomized-policy} needs to formally take both the symmetric and asymmetric policies into consideration.
To address the challenge of comparing the symmetric policy with exponentially many asymmetric policies, we introduce two new techniques to analyze the randomized policies:
(1) For the case of \(M=2\) agents (base hypothesis), we transform the problem to solving an underdetermined linear programming (LP) where the underdeterminedness is due to the additional asymmetric policies.
Then we utilize the property of LP that the optimal solution for an LP must be reached at the vertex of the feasible region and by comparing all feasible vertices, we show that the optimal policy is symmetric for the case of \(2\) agents.
(2) We use the second principle of induction to extend the base hypothesis to general \(M\) agents.
To bridge the general \(M\) case to the base hypotheses in the induction, we introduce a \emph{group-split} technique which splits the large group into two smaller groups.
Inside each smaller group, as it falls into the hypotheses of the induction, the optimal policy is symmetric.
Across these two symmetric groups, one can use the first technique above to show that the optimal policy among them is also symmetric.



\begin{proof}[Proof of Theorem~\ref{thm:homogeneous-randomized-policy}]
    As illustrated above, we use the second principle of induction to prove the theorem.

    \textbf{Step 1: Base hypothesis.} We first prove that for the case of \(M=2\) agents, the optimal randomized policy is symmetric.
    Let's first consider the case for \(M=2\) \emph{without group-buy option}, whose solution can be obtained from solving the following LP,
    \begin{align}
        \begin{blockarray}{cccccccccccc}
            N & p_{1,1} & p_{2,2} & p_{3,3} & p_{4,4} & p_{1,2} & p_{1,3}
            & p_{1,4} & p_{2,3}  & p_{2,4} & p_{3,4}
            & 2N
            \\
            \begin{block}{c[cccc|cccccc]c}
                1 & 2B & 2 & 2   & 2 & B+1 & B+1  & B+1 & 2 & 2 & 2 & 2 c \\
                2 & 2B & 2B+2 & 4   & 4 & 2B+1 &  B+2   & B+2 & B+3 & B+3  & 4 & 4c\\
                3 & 2B & 2B+2 & 2B+4   & 6 & 2B+1 & 2B+2   & B+3 & 2B+3  & B+4 & B+5 & 6c \\
                4 & 2B & 2B+2 & 2B+4   & 2B+6  & 2B+1 & 2B+2 & 2B+3 & 2B+3 & 2B+4 & 2B+5 & 8c\\
            \end{block}
        \end{blockarray}.
    \end{align}
    The derivation of the LP for devising randomized algorithms is detailed in Appendix~\ref{subapp:basic-proof-for-randomized-policy}.

    Then, note that the row of \(p_{i,j}\) for \(i,j\) can be reconstructed by averaging the rows of \(p_{i,i}\) and \(p_{j,j}\), that is,
    \begin{align}\label{eq:homogeneous-randomized-policy-row-average}
        \text{row}_{i,j} = \frac{1}{2}(\text{row}_{i,i} + \text{row}_{j,j})
    \end{align}
    where \(\text{row}_{i,j}\) is the row of \(p_{i,j}\) in the above table.
    From this observation, we can see that any randomized asymmetric policy can be represented as a linear combination of the symmetric policies, which means that the optimal randomized policy \emph{can be} symmetric.

    Next, for the homogeneous \masr \emph{with group-buy option}, the above linear programming is modified as follows,
    \begin{align}
        \begin{blockarray}{cccccccccccc}
            N & p_{1,1} & p_{2,2} & p_{3,3} & p_{4,4} & p_{1,2} & p_{1,3}
            & p_{1,4} & p_{2,3}  & p_{2,4} & p_{3,4}
            & \min\{2N, G\}
            \\
            \begin{block}{c[cccc|cccccc]c}
                1 & G & 2 & 2   & 2 & B+1 & B+1  & B+1 & 2 & 2 & 2 & 2 c \\
                2 & G & G+2 & 4   & 4 & 2B+1 &  B+2   & B+2 & B+3 & B+3  & 4 & 4c\\
                3 & G & G+2 & G+4   & 6 & 2B+1 & 2B+2   & B+3 & 2B+3  & B+4 & B+5 & 6c \\
                4 & G & G+2 & G+4   & G+6  & 2B+1 & 2B+2 & 2B+3 & 2B+3 & 2B+4 & 2B+5 & 8c\\
            \end{block}
        \end{blockarray},
    \end{align}
    where the \(2B\) costs of the symmetric policies are replaced by the group-buy option \(G\).

    Then relation of~\eqref{eq:homogeneous-randomized-policy-row-average} becomes
    \begin{align}\label{eq:homogeneous-randomized-policy-row-average-revised}
        \text{row}'_{i,j} = \text{row}_{i,j} = \frac{1}{2}(\text{row}_{i,i} + \text{row}_{j,j}) \ge \frac{1}{2}(\text{row}'_{i,i} + \text{row}'_{j,j}),
    \end{align}
    where \(\text{row}'_{i,j}\) is the row of \(p_{i,j}\) in the above matrix. From this observation, we can see that for any randomized asymmetric policy, there is a randomized symmetric policy that dominates it.

    \textbf{Step 2: Induction.} Assume that the optimal randomized policy is symmetric for any number of \(m \le M\) agents, and we will show that it is also symmetric for \(M+1\) agents.
    The key idea is to split the \(M+1\) agents into two groups, one with \(M_1\) agents and the other with \(M-M_1\) agents. No matter how the group-buy option is set, the optimal policy for each group is symmetric by the induction hypothesis.
    Then, we can use the same technique as in Step 1 to show that the optimal randomized policy for the two groups of total \(M+1\) agents is also symmetric.

    Lastly, as the optimal randomized policy is symmetric, the specific distribution of the threshold day \(T\) can be derived from the optimal randomized policy for the single-agent ski-rental problem, which is given in Appendix~\ref{subapp:basic-proof-for-randomized-policy}.
\end{proof}


\section{Proofs for Heterogeneous Multi-Agent Ski-Rental: Deterministic Policy}
\label{app:heterogeneous-determinisitic-policy-proof}

We present the proofs for the optimal deterministic policies for all three competitive ratios in the following subsections.

\subsection{Symmetric Deterministic Policies are Better than Asymmetric Policies}\label{sec:symmetric-better-than-asymmetric-deterministic}
\label{sec:symmetric-better-than-asymmetric-deterministic-proof}

\begin{proof}[Proof of Lemma~\ref{lemma:symmetric-better-than-asymmetric-deterministic}]
    We first consider the group competitive ratios: the overall and state-dependent ones.
    Without loss of generality, we consider the state that all agents are active, i.e., \(\ell=0\).
    Given any policy \(\pi = (T(0), T(1), \dots, T(M-1))\) (purchase thresholds in ascending order), we split the sets as two parts for some \(k<M-1\), \(\mathcal{T}_1 \coloneqq \{T(0), T(1), \dots, T(k)\}\) and \(\mathcal{T}_2 \coloneqq \{T(k+1), \dots, T(M-1)\}\) such that \(T(k) \neq T(k + 1)\) and \(T(k+1) = \dots = T(M-1)\).
    Then, if thresholds in \(\mathcal{T}_1\) are reached, the agents buy the individual pass, and if the thresholds in \(\mathcal{T}_2\) are reached, the agents buy either the individual pass or the group pass together, whichever is cheaper.

    To maximize the competitive ratio, the adversary needs to set the active days of agents exactly the same as the thresholds in \(\mathcal{T}_1\) and \(\mathcal{T}_2\), i.e., \(N_m = T(m-1)\) for all \(m\in \{1, \dots, M\}\).
    Then, the overall competitive ratio of the policy is
    \begin{align}
         & \quad \,\frac{\sum_{m=1}^M (T(m-1) - 1) + \ell B + \min\{(M-\ell) B, G\}}{\min\left\{ \sum_{m=1}^M \min\{B, T(m-1)\}, G \right\}}
        =
        \frac{\sum_{m=1}^M (T(m-1) - 1) + \min\{MB, G + \ell B\}}{\min\left\{ \sum_{m=1}^M \min\{B, T(m-1)\}, G \right\}},
    \end{align}
    which is minimized when \(k = 0\) (i.e., all active agents should have the same threshold, belonging to \(\mathcal{T}_2\)), yielding a symmetric policy. As the above derivation only depends on the nominator (i.e., the total cost) instead of the offline benchmark in the denominator, same statement holds for the state-dependent competitive ratio as well.

    For the individual rational competitive ratio, whenever the group-buy option is cheaper than the individual pass (i.e., \((M-\ell) B > G\) for \(\ell\) inactive agents), the symmetric policy is better than the asymmetric policy, while when the group-buy option is more expensive than the individual pass, the problem reduces to the single-agent ski-rental problem for each of the agent, which is also symmetric.
\end{proof}

\subsection{Proof of Minimizing Overall Competitive Ratio: Theorem~\ref{thm:deterministic-heterogeneous-policy-overall}}
\label{sec:deterministic-heterogeneous-policy-overall-proof}
\label{subapp:proof-determinisitic-overall-competitive-ratio}

After events \(N_1 < T(0)\), \(N_2 < T(1)\), \(\ldots\), \(N_{\ell} < T(\ell-1)\) happened, one needs to decide the threshold \(T(\ell)\) for the remaining \(M-\ell\) agents (see Figure~\ref{fig:algorithmic-idea} for an illustration).

If \(N_m \ge T(\ell)\) for all agents \(m \in \{\ell + 1, \ell +2, \dots, M\}\), overall CR becomes
\begin{align}
    \frac{\sum_{n=1}^{\ell} N_n + (M-\ell)(T(\ell) - 1) + \min\{G, (M-\ell)B\}}{
        \min\left\{
        \sum_{n=1}^{\ell} N_n + \sum_{n=\ell+1}^M \min\{N_n, B\},  G
        \right\}
    }.
\end{align}
To maximize the CR from the perspective of adversary, one needs to set \(N_m = T(\ell)\) for all agents \(m \in \{\ell + 1, \ell +2, \dots, M\}\), which yields the CR as follows,
\begin{align}
    \frac{\sum_{n=1}^{\ell} N_n + (M-\ell)(T(\ell) - 1) + \min\{G, (M-\ell)B\}}{
        \min\left\{
        \sum_{n=1}^{\ell} N_n + (M-\ell) \min\{T(\ell), B\},  G
        \right\}
    }.
\end{align}
This CR is minimized when \(T(\ell) = \min\left\{
\frac{G - \sum_{n=1}^{\ell} N_n}{M - \ell}, B
\right\},\) where the CR becomes \begin{align}
    \ovcr(\pi\upbra{\text{OV}}\vert \ell)
     & = \frac{\sum_{n=1}^{\ell} N_n + (M-\ell) \left( \min\left\{
        \frac{G - \sum_{n=1}^{\ell} N_n}{M - \ell}, B
        \right\} - 1 \right) + \min\{G, (M-\ell)B\}}{
        \min\left\{
        \sum_{n=1}^{\ell} N_n + (M-\ell) \min\left\{
        \frac{G - \sum_{n=1}^{\ell} N_n}{M - \ell}, B
        \right\},  G
        \right\}
    }
    \\
     & =
    1 +
    \frac{\min\left\{
        G, (M-\ell) B
        \right\}
        - (M-\ell)
    }{
        \min\left\{
        G, \sum_{n=1}^{\ell} N_n + (M-\ell) B
        \right\}
    }
    \\
     & =
    \begin{cases}
        2 - \frac{M-\ell}{G}
         & \text{ if } G \le (M-\ell)B
        \\
        1 + \frac{(M-\ell)(B-1)}{G}
         & \text{ if }(M-\ell)B < G \le \sum_{n=1}^{\ell} N_n + (M-\ell) B
        \\
        1 + \frac{(M-\ell)(B-1)}{\sum_{n=1}^{\ell} N_n + (M-\ell)B}
         & \text{ if } G > \sum_{n=1}^{\ell} N_n + (M-\ell) B
    \end{cases}.
\end{align}
If \(N_{\ell+1} < T(\ell)\), then it triggers the state transition.
Following the induction, the \(T(\ell+1)\) for the remaining \(M-(\ell+1)\) agents can be analyzed as above as well. \qed

\subsection{Proof of Minimizing State-Dependent Competitive Ratio: Theorem~\ref{thm:deterministic-heterogeneous-policy-state-dependent}}
\label{subapp:proof-determinisitic-state-dependent-competitive-ratio}

After events \(N_1 < T(0)\), \(N_2 < T(1)\), \(\ldots\), \(N_{\ell} < T(\ell-1)\) happened, one needs to decide the threshold \(T(\ell)\) for the remaining \(M-\ell\) agents.
If \(N_m \ge T(\ell)\) for all agents \(m \in \{\ell + 1, \ell +2, \dots, M\}\), the state-dependent CR becomes
\begin{align}
    \frac{\sum_{n=1}^{\ell} N_n + (M-\ell)(T(\ell) - 1) + \min\{G, (M-\ell)B\}
    }{
        \sum_{n=1}^{\ell} N_n
        + \min\left\{
        \sum_{n=\ell+1}^M \min\{N_n, B\},  G
        \right\}
    }
\end{align}
To maximize the CR from the perspective of adversary, one needs to set \(N_m = T(\ell)\) for all \(m \in \{\ell + 1, \ell +2, \dots, M\}\), which yields the CR as follows,
\begin{align}
    \frac{\sum_{n=1}^{\ell} N_n + (M-\ell)(T(\ell) - 1) + \min\{G, (M-\ell)B\}
    }{
        \sum_{n=1}^{\ell} N_n
        + \min\left\{
        \sum_{n=\ell+1}^M \min\{T(\ell), B\},  G
        \right\}
    }
\end{align}
This CR is \emph{minimized} when \(T(\ell) = \min\left\{
\frac{G}{M - m}, B
\right\},\) where the CR becomes \begin{align}
    \sdcr_{m}(\pi\upbra{\text{SD}}\vert \ell)
     & =
    \frac{\sum_{n=1}^{\ell} N_n + (M-\ell)(T(\ell) - 1) + \min\{G, (M-\ell)B\}
    }{
        \sum_{n=1}^{\ell} N_n
        + \min\left\{
        \sum_{n=\ell+1}^M \min\{T(\ell), B\},  G
        \right\}
    }
    \\
     & = \frac{\sum_{n=1}^{\ell} N_n - (M-\ell) + 2\min\{G, (M-\ell)B\}
    }{
        \sum_{n=1}^{\ell} N_n
        + \min\{G, (M-\ell)B\}
    }
    \\
     & =
    1 + \frac{\min\{G, (M-\ell)B\} - (M-\ell)
    }{
        \sum_{n=1}^{\ell} N_n
        + \min\{G, (M-\ell)B\}
    }
    \\
     & =
    \begin{cases}
        1 +
        \frac{M-\ell - G}{\sum_{n=1}^{\ell} N_n + G}
         & \text{ if } G \le (M-\ell)B
        \\
        1 +
        \frac{(M-\ell)(B-1)}{\sum_{n=1}^{\ell} N_n + (M-\ell)B}
         & \text{ if } G > (M-\ell)B
    \end{cases}.
\end{align}

If \(N_{\ell+1} < T(\ell)\), then it triggers the state transition.
Following the induction, the \(T(\ell+1)\) for the remaining \(M-(\ell+1)\) agents can be analyzed as above as well. \qed

\subsection{Proof of Minimizing Individual Rational Competitive Ratio: Theorem~\ref{thm:deterministic-heterogeneous-policy-individual}}
\label{subapp:proof-determinisitic-individual-rational-competitive-ratio}

Recall the offline optimal policy for the individual rational is as follows,
\begin{align}\label{eq:indopt-re}
    \opt_{m}\individual(\mathcal{I})
    \coloneqq
    \begin{cases}
        N_m
         & \text{ if } N_m \leq \min\left\{ \frac{G}{M-\ell_*}, B \right\}
        \\
        \min\left\{\frac{G}{M-\ell_*}, B\right\}
         & \text{ if } N_m >  \min\left\{ \frac{G}{M-\ell_*}, B \right\}
    \end{cases}.
\end{align}
As illustrated in Section~\ref{subsec:deterministic-heterogeneous-policy-individual-rational}, following the threshold \(T\statedep(\ell)\), the agents with \(m\in\{1,\dots,\ell_*\}\) will rent for all active days and hence, enjoy the ideal \(1\)-competitive ratio.
For agents with \(m\in\{\ell_*+1,\dots,M\}\), they will rent for all active days until the threshold \(T\statedep(\ell_*)\) is reached with the competitive ratio \(2 - \frac{1}{T\statedep(\ell_*)}\).
Recall that for single-agent ski-rental, without knowledge of the active days \(D\), the optimal deterministic policy is also to rent until the buy cost \(B\) is reached, which is the same as the threshold \(T\statedep(\ell_*)\) in our case.
Therefore, following threshold \(T\statedep(\ell)\) is the optimal deterministic policy for the individual rational competitive ratio as well. \qed

\begin{algorithm}[htbp]
    \caption{Randomized policy framework for heterogeneous \masr}
    \label{alg:randomized-optimal-policy}
    \begin{algorithmic}[1]
        \Input number of agents \(M\), group cost \(G\), individual cost \(B\), probability density function \(f(\cdot)\)
        \Initial number of inactive agents \(\ell \gets 0\)
        \State Sample initial threshold \(T(0) \sim f(0)\)\label{line:randomized-initial-sample}
        \For{each active day \(t\) (with at least one active agent)}
        \If{\(t = T(\ell)\)}
        \State \textbf{Buy} \(\begin{cases}
            \text{Group pass}      & \text{if } T(\ell) < B
            \\
            \text{Individual pass} & \text{otherwise}
        \end{cases}\) \label{line:randomized-buy}
        \State \textbf{Terminate} \label{line:randomized-terminate}
        \Else
        \State \textbf{Rent} for all active agents \label{line:randomized-rent}
        \EndIf
        \State Reveal \(\mathcal{L}_t\): set of agents become inactive at the end of day \(t\) \label{line:randomized-begin-update}
        \If{\(\mathcal{L}_t \neq \emptyset\)}
        \State \(\ell \gets \ell + \abs{\mathcal{L}_t}\)
        \State \(N_m \gets t\) for all \(m \in \mathcal{L}_t\)
        \State Sample threshold \(T(\ell) \sim f(\{N_{n}\}_{n\le \ell})\) \label{line:randomized-sample}
        \EndIf\label{line:randomized-end-update}
        \EndFor
    \end{algorithmic}
\end{algorithm}

\section{Proofs for Heterogeneous Multi-Agent Ski-Rental: Randomized Policies}
\label{sec:appendix-heterogeneous-randomized}
\label{app:heterogeneous-randomized-policy-proof}

The detailed framework for the randomized policy is summarized in Algorithm~\ref{alg:randomized-optimal-policy}.

\subsection{Proof of Minimizing Competitive Ratio for Homogeneous \masr}\label{subapp:basic-proof-for-randomized-policy}

Before we present the proofs for general heterogeneous \masr, we first present a derivation (similar to that for the single-agent ski-rental problem) for the optimal randomized policy for the homogeneous \masr (where \(N_1 = N_2 = \dots = N_M = N\)) as a preliminary step.
This derivation will help us understand the key challenges for deriving optimal randomized policy for each of three types of competitive ratios in the following subsections.
For simplicity, we assume \(\frac{G}{M}\) is an integer, otherwise one can replace \(\frac{G}{M}\) with its ceiling \(\ceil{\frac{G}{M}}\) where the same proof and results hold.


Denote \(p_t\) for any \(t\in\{1,2,\dots\}\) as the probability of the picking day \(t\) as the threshold to buy the group pass.
Denote \(\ovcr(p_t, \ell = 0)\) as \(c\).
For each possible value of active days \(N \in \{1,2,\dots\}\), we have \begin{align}
    \sum_{t=1}^\infty p_t\cost(T = t, N)  \le c \cdot \ovopt(N\vert \ell = 0) = \min\{M\min\{N, B\}, G\} c.
\end{align}
By eliminating sub-dominated variables and inequalities as in standard ski-rental problem~\citep{mathieu_skirental},
we can derive the following group of inequalities---variables \(t\in \{1,2,\dots, \frac{G}{M}\}\) and inequalities \(N\in\{1,2,\dots, \frac{G}{M}\}\) remain---for the optimal randomized policy for the homogeneous \masr as follows,
\begin{align}
    \begin{blockarray}{cccccccc}
        N & p_{1} & p_{2} & p_{3} & \ldots & p_{\frac G M - 1} & p_{\frac G M}
        & c
        \\
        \begin{block}{c[cccccc]c}
            1 & G & M & M   & \ldots & M & M  &  M  \\
            2 & G & G+M & 2M   & \ldots & 2M &  2M    & 2M\\
            3 & G & G+M & G+2M   & \ldots & 3M & 3M    & 3M \\
            \vdots & \vdots & \vdots & \vdots   & \ddots & \vdots  & \vdots  & \vdots \\
            \frac{G}{M}-1  & G & G+M & G+2M & \ldots & 2G - 2M & G-M  & G-M\\
            \frac{G}{M} & G & G+M & G+2M   & \ldots  & 2G - 2M & 2G-M & G\\
        \end{block}
    \end{blockarray},
\end{align}
where we simplified the group of inequalities by their coefficients as entries in the matrix.
For example, the first row in the above matrix corresponds to the first inequality \(G\cdot p_1 + M \cdot p_2  + \dots + M \cdot p_{\frac{G}{M}} \leq M \cdot c\).
Then, we use the inequality in \(i^{\text{th}}\) row to subtract the inequality in \((i-1)^{\text{th}}\) for all \(i \in \{1,2,\dots, \frac{G}{M}-1\}\), yielding the following group of inequalities,
\begin{align}
    \begin{blockarray}{cccccccc}
        N & p_{1} & p_{2} & p_{3} & \ldots & p_{\frac G M - 1} & p_{\frac G M}
        & c
        \\
        \begin{block}{c[cccccc]c}
            1 & G & M & M   & \ldots & M & M & Mc \\
            2 & 0 & G & M   & \ldots & M &  M & Mc\\
            3 & 0 & 0 & G   & \ldots & M & M   & Mc \\
            \vdots & \vdots & \vdots & \vdots   & \ddots & \vdots  & \vdots  & \vdots \\
            \frac{G}{M}-1  & 0 & 0 &0 & \ldots & G & M & Mc\\
            \frac{G}{M} & 0 & 0 & 0   & \ldots  & 0 & G & Mc\\
        \end{block}
    \end{blockarray}.
\end{align}
Lastly, we use the \(i^{\text{th}}\) row to subtract the \((i-1)^{\text{th}}\) row for all \(i \in \{2,\dots, \frac{G}{M}\}\), yielding the following group of inequalities,
\begin{align}
    \begin{blockarray}{cccccccc}
        N & p_{1} & p_{2} & p_{3} & \ldots & p_{\frac G M - 1} & p_{\frac G M}
        & c
        \\
        \begin{block}{c[cccccc]c}
            1 & G & M-G & 0   & \ldots & 0 & 0  & 0 \\
            2 & 0 & G & M-G   & \ldots & 0 &  0  & 0\\
            3 & 0 & 0 & G   & \ldots & 0 & 0    & 0 \\
            \vdots & \vdots & \vdots & \vdots   & \ddots & \vdots  & \vdots  & \vdots \\
            \frac{G}{M}-1  & 0 & 0 &0 & \ldots & G & M-G  & 0\\
            \frac{G}{M} & 0 & 0 & 0   & \ldots  & 0 & G  & Mc\\
        \end{block}
    \end{blockarray}.
\end{align}
Solving the above inequalities and noticing that \(\sum_{t=1}^{\frac{G}{M}} p_t = 1\) yields the randomized policy in Theorem~\ref{thm:homogeneous-randomized-policy} as follows,
\begin{align}
    p_t=  \frac{M}{G}\left( 1 - \frac{M}{G} \right)^{\frac{G}{M} - t}, \forall t\in \left\{1,2,\dots, {\frac{G}{M}}\right\},
    \text{ and } c = \frac{1}{1 - \left( 1 - \frac{M}{G} \right)^{\frac{G}{M}}}.
\end{align}\qed

\subsection{Symmetric policies are better than asymmetric policies}\label{sec:symmetric-better-than-asymmetric-randomized}

\begin{proof}[Proof of Lemma~\ref{lemma:symmetric-better-than-asymmetric-randomized}]
    The proof of Lemma~\ref{lemma:symmetric-better-than-asymmetric-randomized} consists of two key steps via the application of the second principle of induction. This is illustrated in Figure~\ref{fig:randomized-policy-induction}.

    All diagonal arrows in Figure~\ref{fig:randomized-policy-induction} connects the states with the same number of active agents. For example, the diagonal arrows connecting \((1, T(0))\) to \((2, T(1))\), \((2, T(0))\) to \((3, T(1))\), and so on, all refer to the scenario with one active agent.
    Therefore, the optimal threshold for one of them, say \((1, T(0))\), can be used (with minor modification) for all of them.
    With this in mind, we can focus on the horizontal arrows in Figure~\ref{fig:randomized-policy-induction} to prove the lemma.

    \paragraph{Step 1: Base case for \(2\) agents} This step corresponds to prove the optimal policy of \((2, T(0))\), the red scenario in Figure~\ref{fig:randomized-policy-induction}, is symmetric.
    Following the LP formulation in Section~\ref{subapp:basic-proof-for-randomized-policy}, we can write the problem as follows (one LP split to two parts to fit with text width spacing),
    {\small
            \begin{align}
                \begin{blockarray}{cccccccc|c}
                    (N_1, N_2) & (1,1) & (2,2,2) & (2,2,3) & (2,2,4) & (3,3,3) & (3,3,4) & (4,4)
                    & \text{OPT}
                    \\
                    \begin{block}{c[ccccccc|]c}
                        (1, 1) & G & 2 & 2 & 2 & 2 & 2    & 2
                        & 2c
                        \\
                        (2, 2) & G & G+2 & G+2 & G+2 & 4 & 4 & 4
                        & 4c
                        \\
                        (3, 3) & G & G+2 & G+2 & G+2 & G+4 & G+4 & 6
                        & 6c
                        \\
                        (4, 4) & G & G+2 & G+2 & G+2 & G+4 & G+4 & G+6
                        & 6c
                        \\
                    \end{block}
                    \begin{block}{c[ccccccc|]c}
                        (1, 2) & G & B+2 & 3 & 3 & 3 & 3   & 3  & 3c
                        \\
                        (1, 3) & G & B+2 & B+3 & 4 & B+3 & 4   & 4      & 4c
                        \\
                        (1, 4) & G & B+2 & B+3 & B+4 & B+3 & B+4   & 5
                        & 5c
                        \\
                        (2, 3) & G & G+2 & B+3 & G+2 & B+3 & 5   & 5      & 5c
                        \\
                        (2, 4) & G & G+2 & G+2 & G+2 & G+2 & B+5 & B+5
                        & 6c
                        \\
                        (3, 4) & G & G+2 & G+2 & G+2 & G+4 & G+4 & B+6
                        & 6c
                        \\
                    \end{block}
                \end{blockarray}
            \end{align}
            \begin{align}
                \begin{blockarray}{c|ccccccc}
                    (N_1, N_2) & (1,2) & (1,3) & (1,4) & (2,3) & (2,4) & (3,4)
                    & \text{OPT}
                    \\
                    \begin{block}{c[|cccccc]c}
                        (1, 1)
                        & B+1 & B+1 & B+1 & 2 & 2    & 2 & 2c
                        \\
                        (2, 2)
                        & 2B+1 & B+2 & B+2 & B+3 & B+3 & 4 & 4c
                        \\
                        (3, 3)
                        & 2B+1 & 2B+2 & B+3 & 2B+3 & B+4 & B+5    & 6c
                        \\
                        (4, 4)
                        & 2B+1 & 2B+2 & 2B+3 & 2B+3 & 2B+4 & 2B+5    & 6c
                        \\
                    \end{block}
                    \begin{block}{c[|cccccc]c}
                        (1, 2) & 2B+1 / B+1 & B+2 / B+1 & B+2 / B+1 & 3 / B+2 & 3/B+2    & 3/3 & 3c
                        \\
                        (1, 3) &  2B+1 / B+1 & 2B+2 / B+1 & B+3 / B+1 & B+3 / B+2 & 4 / B+2 & 4 / B+3     & 4c
                        \\
                        (1, 4)
                        & 2B+1 / B+1 & 2B+2 / B+1 & 2B+3 / B+1 & B+3 / B+2 & B+4 / B+2 & B+4 / B+3     & 5c
                        \\
                        (2, 3) & 2B+1 / 2B+1 & 2B+2 / B+2 & B+3 / B+2 & 2B+3 / B+3 & B+4 / B+3 & 5 / B+4     & 5c
                        \\
                        (2, 4)
                        & 2B+1 / 2B+1 & 2B+2 / B+2 & 2B+3 / B+2 & 2B+3 / B+3 & 2B+4 / B+3 & B+5 / B+4     & 6c
                        \\
                        (3, 4)
                        & 2B+1 / 2B+1 & 2B+2 / 2B+2 & 2B+3 / B+2 & 2B+3 / 2B+3 & 2B+4 / B+4 & 2B+5 / B+5     & 6c
                        \\
                    \end{block}
                \end{blockarray}
            \end{align}
        }
    where the column index \((i,i,j)\) refers to the policy that both agents buy the group pass in the same time slot \(i\), and if one of the agent becomes inactive before the time slot \(i\), then the other agent buys the individual pass at time slot \(j\), and the column index \((i,j)\) refers to the policy that agent \(1\) buys the individual pass at time slot \(i\) and agent \(2\) buys the individual pass at time slot \(j\).
    The row index refers to the active days of the two agents.

    Then, we can eliminate the dominated variables and inequalities as in standard ski-rental problem~\citep{mathieu_skirental}.
    The detail elimination consists of the following steps,
    \begin{itemize}
        \item Use the row of \((2,2)\) to eliminate the row of \((1,3)\),
        \item Use the row of \((4,4)\) to eliminate the rows of \((2,4)\) and \((3,4)\),
        \item Used the first parts (before slash) of columns of \((1,2), (1,3), (1,4) (3,4)\) to eliminate the second parts (after slash) of these columns respectively,
        \item Use the linear combination of columns \((1,1)\) and \((4,4)\) to eliminate the full columns of \((1,2), (1,3), (1,4)\), \begin{itemize}
                  \item with condition \(G\le 2B - 3\),
              \end{itemize}
        \item Use the linear combination of columns \((2,2,4)\) and \((4,4)\) to eliminate the full columns of \((2,3), (2,4)\),
        \item Use column \((4,4)\) to eliminate the full column of \((3,4)\).
    \end{itemize}
    With the above eliminations, we remove all columns of the asymmetric policies (after the vertical line) and the remaining columns are all symmetric policies, which can be expressed as follows,\begin{align}
        \begin{blockarray}{ccccccccc}
            (N_1, N_2) & (1,1) & (2,2,2) & (2,2,3) & (2,2,4) & (3,3,3) & (3,3,4) & (4,4)
            & \text{OPT}
            \\
            \begin{block}{c[ccccccc]c}
                (1, 1) & G & 2 & 2 & 2 & 2 & 2    & 2
                & 2c
                \\
                (2, 2) & G & G+2 & G+2 & G+2 & 4 & 4 & 4
                & 4c
                \\
                (3, 3) & G & G+2 & G+2 & G+2 & G+4 & G+4 & 6
                & 6c
                \\
                (4, 4) & G & G+2 & G+2 & G+2 & G+4 & G+4 & G+6
                & 6c
                \\
            \end{block}
            \begin{block}{c[ccccccc]c}
                (1, 2) & G & B+2 & 3 & 3 & 3 & 3   & 3  & 3c
                \\
                (1, 4) & G & B+2 & B+3 & B+4 & B+3 & B+4   & 5
                & 5c
                \\
                (2, 3) & G & G+2 & B+3 & G+2 & B+3 & 5   & 5   & 5c
                \\
            \end{block}
        \end{blockarray}
    \end{align}

    Notice that for each row of \((1,1), (2,2), (3,3), (4,4)\), the entries in columns of \((2,2,2), (2,2,3), (2,2,4)\) are the sames, and same statement holds for the columns of \((3,3,3), (3,3,4)\).
    Hence, we can solve the LP formulation by only considering the rows of \((1,1), (2,2), (3,3), (4,4)\) for the symmetric policies to get \(p_{1,1},p_{2,2}, p_{3,3}, p_{4,4}\) such that \(p_{1,1} + p_{2,2} + p_{3,3} + p_{4,4}=1\).
    Then, for \(p_{2,2}\) and \(p_{3,3}\), we can use the rows of \((1,2), (1,4)\) and \((2,3)\) to further get the probabilities \(p_{(2,2,2)}, p_{(2,2,3)}, p_{(2,2,4)}, p_{(3,3,3)}, p_{(3,3,4)}\) such that \(p_{(2,2,2)}+ p_{(2,2,3)}+ p_{(2,2,4)}=p_{2,2}\) and \(p_{(3,3,3)} + p_{(3,3,4)} = p_{3,3}\).

    \paragraph{Step 2: Induction for \(M\) agents}
    This corresponds to the horizontal arrows in Figure~\ref{fig:randomized-policy-induction}. This part can be proved by the same ``group split'' argument as the induction proof for the homogeneous \masr for Theorem~\ref{thm:homogeneous-randomized-policy}.
\end{proof}

\subsection{Proof of Minimizing Overall Competitive Ratio: Theorem~\ref{thm:randomized-heterogeneous-policy-overall}}\label{subapp:proof-randomized-heterogeneous-policy-overall}
\label{subsec:randomized-heterogeneous-policy-overall-proof}
\label{subapp:proof-randomized-overall-competitive-ratio}

After events \(N_1 < T(0)\), \(N_2 < T(1)\), \(\ldots\), \(N_{\ell} < T(\ell-1)\) happened, one needs to decide the randomized policy (probability distribution) to sample the next threshold \(T(\ell)\) for the remaining \(M-\ell\) agents (see Figure~\ref{fig:algorithmic-idea} for an illustration).

Denote \(p_t\) for any \(t\in\{N_{\ell}+1,N_{\ell}+2,\dots\}\) as the probability of picking day \(t\) as the threshold to buy the group pass.
Denote \(\ovcr(p_t\vert \ell)\) as \(c\) for simplicity here.
For each possible value of active days \(N \in \{N_{\ell}+1,N_{\ell}+2,\dots\}\), we have \begin{align}
     & \underbrace{\sum_{m=1}^{\ell} N_m + (M-\ell) N_{\ell}}_{\eqqcolon D(\ell)}  + \sum_{t=N_{\ell}+1}^\infty  p_t \sum_{m=\ell+1}^M (\cost_m(T = t, N) - N_\ell)
    \\
     & \qquad\qquad\qquad\qquad \le c \cdot \ovopt(N) = c\cdot \left( \min\left\{(M-\ell) \min\{N, B\} + \sum_{m=1}^{\ell} N_m, G\right\} \right).
\end{align}
With eliminating sub-dominated variables and inequalities as in standard ski-rental problem,
we can derive the following group of inequalities---only variables \(t\in \{N_{\ell + 1}, N_{\ell + 2}, \dots, T\overall(\ell)\}\) and inequalities \(N\in \{N_{\ell + 1}, N_{\ell + 2}, \dots, T\overall(\ell)\}\) remain and recall that \(T\overall(\ell) = \min\left\{ \frac{G - \sum_{m=1}^\ell N_m}{M-\ell}, B \right\}\) (assumed to be an integer without loss of generality)---for the optimal randomized policy for minimizing overall competitive ratio in \masr as follows,
{\scriptsize
        \begin{align}
            \begin{blockarray}{cccccccc}
                N & p_{N_{\ell} + 1} & p_{N_{\ell} + 2} & p_{N_{\ell} + 3} & \ldots & p_{T\overall(\ell)-1} & p_{T\overall(\ell)}
                & c
                \\
                \begin{block}{c[cccccc]c}
                    N_{\ell} + 1 & H(\ell) & M-\ell & M-\ell   & \ldots & M-\ell & M-\ell  &  (N_\ell + 1) (M-\ell)  \\
                    N_{\ell} + 2 & H(\ell) & H(\ell)+M-\ell & 2(M-\ell)   & \ldots & 2(M-\ell) &  2(M-\ell)    & (N_\ell + 2)(M-\ell)\\
                    N_{\ell} + 3 & H(\ell) & H(\ell)+M-\ell & H(\ell)+2(M-\ell)   & \ldots & 3(M-\ell) & 3(M-\ell)    & (N_\ell + 3)(M-\ell) \\
                    \vdots & \vdots & \vdots & \vdots   & \ddots & \vdots  & \vdots  & \vdots \\
                    T\overall(\ell)-1  & H(\ell) & H(\ell)+M-\ell & H(\ell)+2(M-\ell) & \ldots & 2H(\ell) - 2(M-\ell) & H(\ell)-(M-\ell)  & G-(M-\ell)\\
                    T\overall(\ell) & H(\ell) & H(\ell)+M-\ell & H(\ell)+2(M-\ell)   & \ldots  & 2H(\ell) - 2(M-\ell) & 2H(\ell)-(M-\ell) & G\\
                \end{block}
            \end{blockarray},
        \end{align}
    }
where \(H(\ell) \coloneqq (M-\ell)T\overall(\ell)\) is the rent cost of the remaining \(M-\ell\) agents after day \(N_{\ell}\),
and the entries in the above matrix, \emph{plus} \(D(\ell)\)---the total cost on and before day \(N_{\ell}\), correspond to the coefficients of the inequalities.

Then, we use the inequality in \(i^{\text{th}}\) row to subtract the inequality in \((i-1)^{\text{th}}\) for all \(i \in \{N_{\ell} +1, N_{\ell} + 2,\dots, T\overall(\ell)-1\}\), yielding the following group of inequalities,
{\scriptsize
        \begin{align}
            \begin{blockarray}{cccccccc}
                N & p_{N_{\ell} + 1} & p_{N_{\ell} + 2} & p_{N_{\ell} + 3} & \ldots & p_{T\overall(\ell)-1} & p_{T\overall(\ell)}
                & c
                \\
                \begin{block}{c[cccccc]c}
                    N_{\ell} + 1 & D(\ell) + H(\ell) & D(\ell) + M-\ell & D(\ell) + M-\ell   & \ldots & D(\ell) + M-\ell & D(\ell) + M-\ell  &  (N_\ell + 1)(M-\ell)  \\
                    N_{\ell} + 2 & 0 & H(\ell) & M-\ell   & \ldots & M-\ell &  M-\ell    & M-\ell\\
                    N_{\ell} + 3 & 0 & 0 & H(\ell)   & \ldots & M-\ell & M-\ell    & M-\ell \\
                    \vdots & \vdots & \vdots & \vdots   & \ddots & \vdots  & \vdots  & \vdots \\
                    T\overall(\ell)-1  & 0 & 0 & 0 & \ldots & H(\ell) & M-\ell  & M-\ell\\
                    T\overall(\ell) & 0 & 0 & 0   & \ldots  & 0 & H(\ell) & M-\ell\\
                \end{block}
            \end{blockarray},
        \end{align}
    }
where the entries in the above matrix \emph{directly} correspond to the coefficients of the inequalities. We note that after the subtraction, only the entries in the first row contain the \(D(\ell)\) term.

Lastly, we use the \(i^{\text{th}}\) row to subtract the \((i-1)^{\text{th}}\) row for all \(i \in \{N_{\ell} + 2,\dots, T\overall(\ell)\}\), yielding the following group of inequalities,
{\scriptsize
        \begin{align}
            \begin{blockarray}{cccccccc}
                N & p_{N_{\ell} + 1} & p_{N_{\ell} + 2} & p_{N_{\ell} + 3} & \ldots & p_{T\overall(\ell)-1} & p_{T\overall(\ell)}
                & c
                \\
                \begin{block}{c[cccccc]c}
                    N_{\ell} + 1 & D(\ell) + H(\ell) & D(\ell) + M-\ell - H(\ell) & D(\ell)    & \ldots & D(\ell)  & D(\ell)   &  N_\ell (M-\ell)  \\
                    N_{\ell} + 2 & 0 & H(\ell) & M-\ell - H(\ell)   & \ldots & 0 &  0    & 0\\
                    N_{\ell} + 3 & 0 & 0 & H(\ell)   & \ldots & M-\ell - H(\ell) & 0    & 0 \\
                    \vdots & \vdots & \vdots & \vdots   & \ddots & \vdots  & \vdots  & \vdots \\
                    T\overall(\ell)-1  & 0 & 0 & 0 & \ldots & H(\ell) & M-\ell - H(\ell)  & 0\\
                    T\overall(\ell) & 0 & 0 & 0   & \ldots  & 0 & H(\ell) & M-\ell\\
                \end{block}
            \end{blockarray},
        \end{align}
    }
Solving the above inequalities and noticing that \(\sum_{t=N_{\ell}+1}^{T\overall(\ell)} p_t = 1\) yield the randomized policy in Theorem~\ref{thm:randomized-heterogeneous-policy-overall}. \qed

\subsection{Proof of Minimizing State-Dependent Competitive Ratio: Theorem~\ref{thm:randomized-heterogeneous-policy-state-dependent}}\label{subapp:proof-randomized-heterogeneous-policy-state-dependent}
\label{subapp:proof-randomized-state-dependent-competitive-ratio}

The proof of minimizing state-dependent competitive ratio is similar to that of minimizing overall competitive ratio in Appendix~\ref{subapp:proof-randomized-heterogeneous-policy-overall}, where the corresponding inequalities are as follows,
\begin{align}
     & \sum_{m=1}^{\ell} N_m + (M-\ell) N_{\ell}  + \sum_{t=N_{\ell}+1}^\infty  p_t \sum_{m=\ell+1}^M (\cost_m(T = t, N) - N_\ell)
    \\
     & \qquad\qquad\qquad\qquad \le c \cdot \sdopt(N\vert \ell) = c\cdot \left( \sum_{m=1}^{\ell} N_m + \min\left\{(M-\ell) \min\{N, B\}, G\right\} \right),
\end{align}
and the threshold \(T\overall(\ell)\) is replaced with \(T\statedep(\ell) = \min\left\{ \frac{G}{M-\ell}, B \right\}\).
The rest of the proof follows the same steps as in Appendix~\ref{subapp:proof-randomized-heterogeneous-policy-overall} and yields the randomized policy in Theorem~\ref{thm:randomized-heterogeneous-policy-state-dependent}. \qed

\subsection{Proof of Minimizing Individual Rational Competitive Ratio: Theorem~\ref{thm:randomized-heterogeneous-policy-individual}}\label{subapp:proof-randomized-heterogeneous-policy-individual}
\label{subapp:proof-randomized-individual-rational-competitive-ratio}

After events \(N_1 < T(0)\), \(N_2 < T(1)\), \(\ldots\), \(N_{\ell} < T(\ell-1)\) happened, one needs to decide the randomized policy (probability distribution) to sample the next threshold \(T(\ell)\) for the remaining \(M-\ell\) agents (see Figure~\ref{fig:algorithmic-idea} for an illustration).
Denote \(p_t\) for any \(t\in\{N_{\ell}+1,N_{\ell}+2,\dots\}\) as the probability of the picking day \(t\) as the threshold to buy the group pass.
Denote \(\ratio\individual_{\ell}(p_t\vert \ell)\) as \(c\).
For each possible value of active days \(N \in \{N_{\ell}+1,N_{\ell}+2,\dots\}\), for any \emph{single} agent \(m > \ell\), we have \begin{align}
     & N_{\ell}  + \sum_{t=N_{\ell}+1}^\infty  p_t (\cost_m(T = t, N) - N_\ell)
    \le c \cdot \opt\individual_{m} (N) = c\cdot \left( \min\left\{\min\{N, B\}, \frac{G}{M-\ell}\right\} \right).
\end{align}

With eliminating sub-dominated variables and inequalities as in standard ski-rental problem,
we can derive the following group of inequalities---variables \(t\in \{N_{\ell + 1}, N_{\ell + 2}, \dots, T\statedep(\ell)\}\) and inequalities \(N\in \{N_{\ell + 1}, N_{\ell + 2}, \dots, T\statedep(\ell)\}\) remain and recall that \(T\statedep(\ell) = \min\left\{ \frac{G}{M-\ell}, B \right\}\)---for the optimal randomized policy of the single agent \(m\) for its minimizing individual rational competitive ratio in \masr as follows,
{\scriptsize
        \begin{align}
            \begin{blockarray}{cccccccc}
                N & p_{N_{\ell} + 1} & p_{N_{\ell} + 2} & p_{N_{\ell} + 3} & \ldots & p_{T\statedep(\ell)-1} & p_{T\statedep(\ell)}
                & c
                \\
                \begin{block}{c[cccccc]c}
                    N_{\ell} + 1 & T\statedep(\ell) & 1 & 1   & \ldots & 1 & 1  &  N_\ell + 1  \\
                    N_{\ell} + 2 & T\statedep(\ell) & T\statedep(\ell)+1 & 2   & \ldots & 2 &  2    & N_\ell + 2\\
                    N_{\ell} + 3 & T\statedep(\ell) & T\statedep(\ell)+1 & T\statedep(\ell)+2   & \ldots & 3 & 3    & N_\ell + 3 \\
                    \vdots & \vdots & \vdots & \vdots   & \ddots & \vdots  & \vdots  & \vdots \\
                    T\statedep(\ell)-1  & T\statedep(\ell) & T\statedep(\ell)+1 & T\statedep(\ell)+2 & \ldots & 2T\statedep(\ell) - 2 & T\statedep(\ell)-1  & T\statedep(\ell)-1\\
                    T\statedep(\ell) & T\statedep(\ell) & T\statedep(\ell)+1 & T\statedep(\ell)+2   & \ldots  & 2T\statedep(\ell) - 2 & 2T\statedep(\ell)-1 & T\statedep(\ell)\\
                \end{block}
            \end{blockarray},
        \end{align}
    }
where the entries in the above matrix, \emph{plus} \(N_\ell\)---the cost of agent \(m\) spent on and before day \(N_{\ell}\), correspond to the coefficients of the inequalities.

Then, we use the inequality in \(i^{\text{th}}\) row to subtract the inequality in \((i-1)^{\text{th}}\) for all \(i \in \{N_{\ell} +1, N_{\ell} + 2,\dots, T\statedep(\ell)-1\}\), yielding the following group of inequalities,
{\footnotesize
        \begin{align}
            \begin{blockarray}{cccccccc}
                N & p_{N_{\ell} + 1} & p_{N_{\ell} + 2} & p_{N_{\ell} + 3} & \ldots & p_{T\statedep(\ell)-1} & p_{T\statedep(\ell)}
                & c
                \\
                \begin{block}{c[cccccc]c}
                    N_{\ell} + 1 & N_\ell + T\statedep(\ell) & N_\ell + 1 & N_\ell + 1   & \ldots & N_\ell + 1 & N_\ell + 1  &  N_\ell + 1  \\
                    N_{\ell} + 2 & 0 & T\statedep(\ell) & 1   & \ldots & 1 &  1    & 1\\
                    N_{\ell} + 3 & 0 & 0 & T\statedep(\ell)  & \ldots & 1 & 1    & 1 \\
                    \vdots & \vdots & \vdots & \vdots   & \ddots & \vdots  & \vdots  & \vdots \\
                    T\statedep(\ell)-1  & 0 & 0 & 0 & \ldots & T\statedep(\ell) & 1  & 1\\
                    T\statedep(\ell) & 0 & 0 & 0   & \ldots  & 0 & T\statedep(\ell) & 1\\
                \end{block}
            \end{blockarray},
        \end{align}
    }
where the entries in the above matrix \emph{directly} correspond to the coefficients of the inequalities. After the subtraction, only the entries in the first row contain the \(N_\ell\) term.

Lastly, we use the \(i^{\text{th}}\) row to subtract the \((i-1)^{\text{th}}\) row for all \(i \in \{N_{\ell} + 2,\dots, T\statedep(\ell)\}\), yielding the following group of inequalities,
{\footnotesize
        \begin{align}
            \begin{blockarray}{cccccccc}
                N & p_{N_{\ell} + 1} & p_{N_{\ell} + 2} & p_{N_{\ell} + 3} & \ldots & p_{T\statedep(\ell)-1} & p_{T\statedep(\ell)}
                & c
                \\
                \begin{block}{c[cccccc]c}
                    N_{\ell} + 1 & N_\ell + T\statedep(\ell) & N_\ell + 1 - T\statedep(\ell)  & N_\ell    & \ldots & N_\ell  & N_\ell   &  N_\ell   \\
                    N_{\ell} + 2 & 0 & T\statedep(\ell) & 1 - T\statedep(\ell)   & \ldots & 0 &  0    & 0\\
                    N_{\ell} + 3 & 0 & 0 & T\statedep(\ell)  & \ldots & 0 & 0    & 0 \\
                    \vdots & \vdots & \vdots & \vdots   & \ddots & \vdots  & \vdots  & \vdots \\
                    T\statedep(\ell)-1  & 0 & 0 & 0 & \ldots & T\statedep(\ell) & 1- T\statedep(\ell)  & 0\\
                    T\statedep(\ell) & 0 & 0 & 0   & \ldots  & 0 & T\statedep(\ell) & 1\\
                \end{block}
            \end{blockarray},
        \end{align}
    }
Solving the above inequalities and noticing that \(\sum_{t=N_{\ell}+1}^{T\statedep(\ell)} p_t = 1\) yield the randomized policy in Theorem~\ref{thm:randomized-heterogeneous-policy-individual-rational}. \qed

\section{Competitive Ratio Lower Bounds for Heterogeneous \masr Problems}
\label{app:lower-bound}

\subsection{Lower Bounds for Deterministic Policies}\label{app:deterministic-lower-bound}

With the claim that the optimal deterministic policy is symmetric in Lemma~\ref{lemma:symmetric-better-than-asymmetric-deterministic}, we only needs to prove the lower bound for the class of the symmetric deterministic policies.
For any symmetric policy (i.e., same thresholds \(T\) for all active agents), the optimal choice of the adversary is to set the active days of agents the same.
Specifically,
the best option of the adversary (for maximizing the competitive ratio) is to set the active days of agents the same as for the symmetric policy, i.e., \(N_m = T\) for every active agent \(m\in \{1, \dots, M\}\), where \(T\) is the threshold of the symmetric policy.
Any other choices of the adversary will lead to a smaller competitive ratio, contradicting that the adversary is trying to maximize the competitive ratio.

Therefore, for the overall competitive ratio, given the state of \(\ell\) inactive agents, the lower bound is exactly the same as that presented in~\eqref{eq:deterministic-overall-competitive-ratio}, highlighting that optimality of state-aware algorithm design. Similarly, the lower bound for the state-dependent and individual rational competitive ratios are the same as those presented in~\eqref{eq:deterministic-state-dependent-competitive-ratio} and~\eqref{eq:deterministic-individual-rational-competitive-ratio}, respectively.

\subsection{Lower Bounds for Randomized Policies: Proof for Lemma~\ref{lemma:randomized-competitive-ratio-lower-bound}}\label{app:randomized-lower-bound}\label{subsec:randomized-policy-lower-bound-proof}
\label{sec:randomized-competitive-ratio-lower-bound-proof}

Without loss of generality, we prove a competitive ratio lower bound for the randomized policies for the case that all agents are active (i.e., state \(\ell\) is zero), and with the asymptotic nature of the lower bound, one can easily extend the proof to other states.

\paragraph{Step 1: use Yao's priniciple to transfer the task }
\begin{align}
    \min_{\text{randomized policy } f(\cdot)} \,\, \max_{\text{instance }\mathcal{I}=(N_1,\dots, N_M)} \ratio
     & = \min_{f} \max_{\mathcal{I}} \frac{\mathbb{E}_{\pi\sim f}[\cost(\pi, \mathcal{I})]}{\opt(\mathcal{I})}
    \\
     & \ge \max_{\mathcal I \sim \mathcal{D}} \min_{\pi} \frac{\mathbb{E}_{\mathcal I \sim \mathcal{D}}[\cost(\pi, \mathcal{I})]}{\mathbb{E}_{\mathcal I \sim \mathcal{D}}[\opt(\mathcal{I})]},
\end{align}
where \(\mathcal{D}\) is a distribution over all possible instances \(\mathcal{I}=(N_1,\dots, N_M)\), and the inequality follows from Yao's principle~\cite{yao1977probabilistic}.
In the following, we will construct a distribution \(\mathcal{D}\) over all possible instances, such that the competitive ratio is lower bounded.

\paragraph{Step 2: find a ``hard'' distribution over all possible instances.}
As the adversary already knows that the final optimal randomized policy is symmetric (proved in Appendix~\ref{sec:symmetric-better-than-asymmetric-randomized}), its best strategy is to set the active days of agents the same, i.e., \(N_m = n\) for all \(m\in \{1, \dots, M\}\)---the same argument as in the deterministic case.
Then, we consider the ``hard'' distribution over instances as follows,
\begin{align}
    g(n) = \frac{M}{G} e^{-\frac{nM}{G}} \quad \text{for } n\ge 0.
\end{align}
For simplicity, we normalize over the cost \(\frac{G}{M}\) as unit \(1\), which simplifies the distribution to
\begin{align}
    g(n) = e^{-n}.
\end{align}
Below, we use \(n\sim g(\cdot)\) to denote the random instance with active days \(N_m = n\) for all \(m\in \{1, \dots, M\}\), sampled from the distribution \(g(\cdot)\), which corresponds to \(\mathcal{I}\sim \mathcal{D}\) in the previous step.

\paragraph{Step 3: caclulate the competitive ratio under the ``hard'' distribution.}
We first calculate the optimal cost \(\opt\) under the instance \(N_m = n\) for all \(m\in \{1, \dots, M\}\):
\begin{align}
    \mathbb{E}_{n\sim g(\cdot)}[\opt(\mathcal{I}=(n, \dots, n))]
     & = \int_{n=0}^1 Mn e^{-n} dn + M\int_{n=1}^\infty e^{-n} dn
    \\
     & = M( 1 - e^{-1}),
\end{align}
where the first term corresponds to the cost of renting for all active days, while the second term corresponds to the cost of group buying for all active agents.

Then, we calculate the expected cost of the randomized policy \(\pi\) with buying threshold \(T=t\) under the instance \(N_m = n\) for all \(m\in \{1, \dots, M\}\):
\begin{align}
    \mathbb{E}_{n\sim g(\cdot)}[\cost(T=t, \mathcal{I}=(n, \dots, n))]
     & = \int_{n=0}^t M n e^{-n} dn + (Mx+M)\int_{n=t}^\infty e^{-n} dn
    \\
     & = M( 1 - e^{-t}(t+1)) + M (t+1) e^{-t} = M.
\end{align}

Putting everything together, we have
\begin{align}
    \ratio
     & \ge \frac{\mathbb{E}_{n\sim g(\cdot)}[\cost(T=t, \mathcal{I}=(n, \dots, n))]}{\mathbb{E}_{n\sim g(\cdot)}[\opt]}
    \\
     & = \frac{M}{M(1 - e^{-1})} = \frac{e}{e - 1} \approx 1.58.
\end{align}


\section{More Discussions on Figure~\ref{fig:deterministic-CR-trend}}
\label{app:deterministic-CR-trend}


In this section, we provide more discussions on the numerical results in Figure~\ref{fig:deterministic-CR-trend}.
The detailed numerical results (with three decimal digits) are summarized in Table~\ref{tab:threshold-M=10}.

We derive the overall competitive ratio of the deterministic policy \(\pi_*\statedep\) as follows,
\begin{align}
        \ovcr(\pi_*\upbra{\text{SD}}\vert \{N_n\}_{n\le \ell})
         & =
        \begin{cases}
                2 + \frac{\sum_{n=1}^\ell N_n - (M - \ell)}{G}
                 & \text{ if } G \le (M-\ell)B
                \\
                \frac{\sum_{n=1}^\ell N_n + (M-\ell)(2B-1)}{G}
                 & \text{ if }(M-\ell)B < G \le \sum_{n=1}^{\ell} N_n + (M-\ell) B
                \\
                1 + \frac{(M-\ell)(B-1)}{\sum_{n=1}^{\ell} N_n + (M-\ell)B}
                 & \text{ if } G > \sum_{n=1}^{\ell} N_n + (M-\ell) B
        \end{cases},
\end{align}
where the third case have the same ratio as the second case of the state-dependent \(\ratio\) in~\eqref{eq:deterministic-state-dependent-competitive-ratio}, as well as the third case of \(\ovcr(\pi_*\upbra{\text{OV}}\vert \{N_n\}_{n\le \ell})\) in~\eqref{eq:deterministic-overall-competitive-ratio}.
This is because when \(G\) is very large, (1) both overall and state-dependent optimal offline policies are the same (i.e., either rent or buy the individual pass), and (2) the thresholds \(T(\ell)\) for both overall and state-dependent policies are the same, i.e., \(T\overall(\ell) = T\statedep(\ell) =  B\).
We also calculate the state-dependent competitive ratio of the deterministic policy \(\pi_*\overall\), that is, \(\sdcr(\pi_*\overall\vert \ell)\), which turns out to be the same as the overall competitive ratio, that is, \(\sdcr(\pi_*\overall\vert \ell) = \ovcr(\pi_*\overall\vert \ell)\) for all \(\ell\) (see the two \text{red}{red} lines in both subfigures of Figures~\ref{subfig:deterministic-CR-trend-overall} and~\ref{subfig:deterministic-CR-trend-state-dependent}).
This counterintuitive coincidence emerges from the fact that the two competitive ratios are worst-case ones defined in~\eqref{eq:state-dependent-cr} and~\eqref{eq:overall-cr},
and their coincident ratios correspond to different worst-case instances for the remaining active agents.

\begin{table}[H]
        \centering
        \caption{Details for Figure~\ref{fig:deterministic-CR-trend}}\label{tab:threshold-M=10}
        \resizebox{\textwidth}{!}{
                \begin{tabular}{c|cccccccccc}
                        \hline
                        \# inactive agents \(\ell\)
                         & 0
                         & 1
                         & 2
                         & 3
                         & 4
                         & 5
                         & 6
                         & 7
                         & 8
                         & 9
                        \\
                        \hline
                        \(\sdcr(\pi\statedep_*\vert \ell)\)
                         & 1.833
                         & {1.836}
                         & 1.825
                         & 1.803
                         & 1.771
                         & 1.692
                         & 1.590
                         & 1.466
                         & 1.321
                         & 1.164
                        \\
                        \(\ovcr(\pi\statedep_*\vert \ell)\)
                         & 1.833
                         & 1.867
                         & 1.917
                         & 1.983
                         & {2.067}
                         & 1.833
                         & 1.617
                         & 1.466
                         & 1.321
                         & 1.164
                        \\
                        \(\sdcr(\pi\overall_*\vert \ell)\)
                         & 1.833
                         & 1.850
                         & 1.867
                         & 1.883
                         & {1.900}
                         & 1.750
                         & 1.600
                         & 1.466
                         & 1.321
                         & 1.164
                        \\
                        \(\ovcr(\pi\overall_*\vert \ell)\)
                         & 1.833
                         & 1.850
                         & 1.867
                         & 1.883
                         & {1.900}
                         & 1.750
                         & 1.600
                         & 1.466
                         & 1.321
                         & 1.164   \\
                        \hline
                        $\sdcr(f_*\statedep\vert\ell)$
                         & 1.504
                         & 1.518
                         & {1.528}
                         & 1.528
                         & 1.513
                         & 1.488
                         & 1.445
                         & 1.382
                         & 1.290
                         & 1.163
                        \\
                        $\ovcr(f_*\overall\vert\ell)$
                         & 1.504
                         & 1.520
                         & 1.527
                         & {1.531}
                         & 1.520
                         & 1.494
                         & 1.449
                         & 1.384
                         & 1.290
                         & 1.164
                        \\
                        \hline
                \end{tabular}}
\end{table}

\section{Additional Experiments}\label{app:additional-experiments}

In this section, we provide additional experiments to further validate the effectiveness of our proposed algorithms, where we vary the normal distribution (mean \(160\) with deviation \(30\)) used in the main paper for generating the active days of agents.
Specifically, we consider the following four normal distributions:
\begin{itemize}
        \item Figure~\ref{fig:hist-comparison-group-mean-160-std-10}: Normal distribution with mean \(160\) and standard deviation \(10\)
        \item Figure~\ref{fig:hist-comparison-group-mean-160-std-50}: Normal distribution with mean \(160\) and standard deviation \(50\)
        \item Figure~\ref{fig:hist-comparison-group-mean-140-std-30}: Normal distribution with mean \(140\) and standard deviation \(30\)
        \item Figure~\ref{fig:hist-comparison-group-mean-180-std-30}: Normal distribution with mean \(180\) and standard deviation \(30\)
\end{itemize}
Except for the normal distributions, all of the experiments are conducted with the same setup as in the main paper.
From these figures, we can see that the our proposed algorithms achieves the lowest worst-case (Max) competitive ratios in terms of its corresponding competitive ratios respectively, which is the same as the figures reported in the main paper and further validates the effectiveness of our proposed algorithms.

\begin{figure}[t]
        \centering
        \begin{subfigure}{0.945\textwidth}
                \centering
                \includegraphics[width=\linewidth]{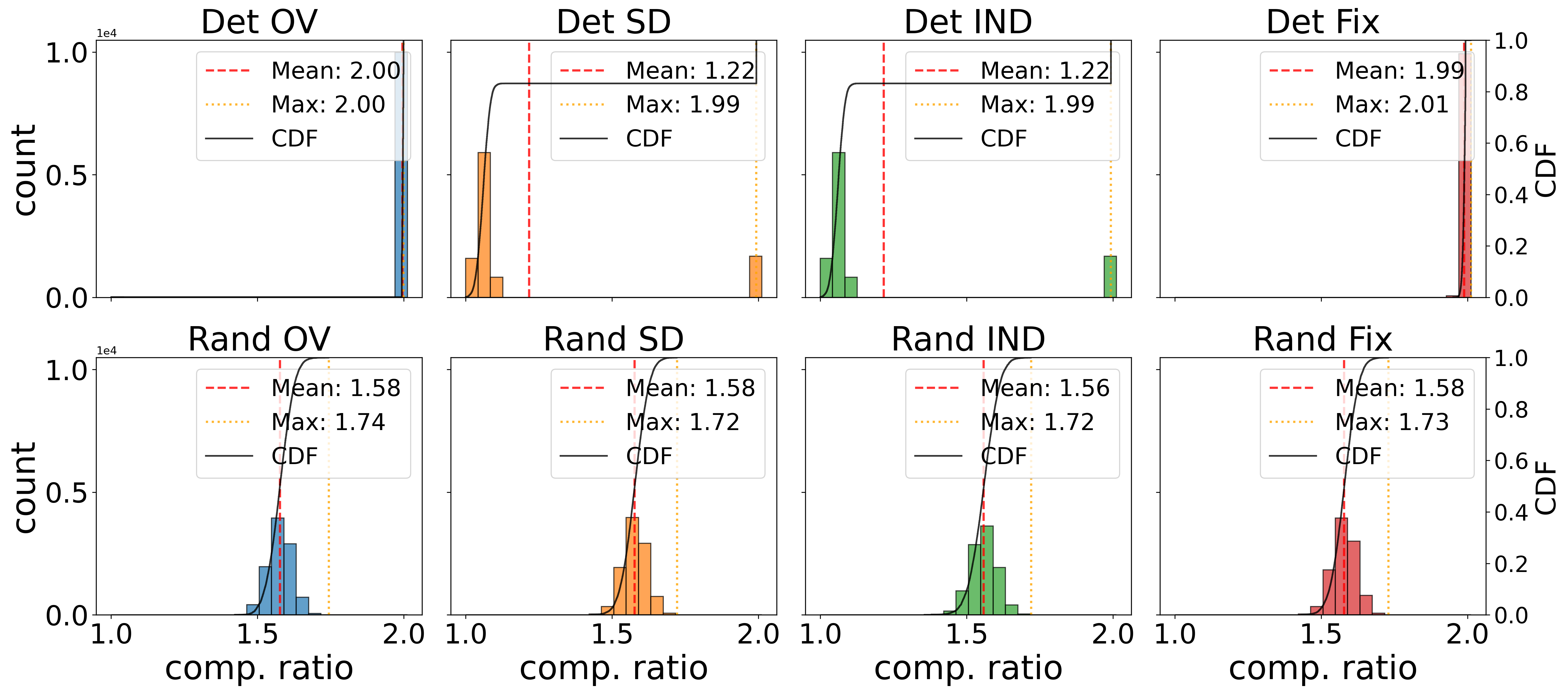}
                \caption{For overall competitive ratio \(\ovcr\)}
        \end{subfigure}
        \hspace{0.2em}
        \\
        \begin{subfigure}{0.945\textwidth}
                \centering
                \includegraphics[width=\linewidth]{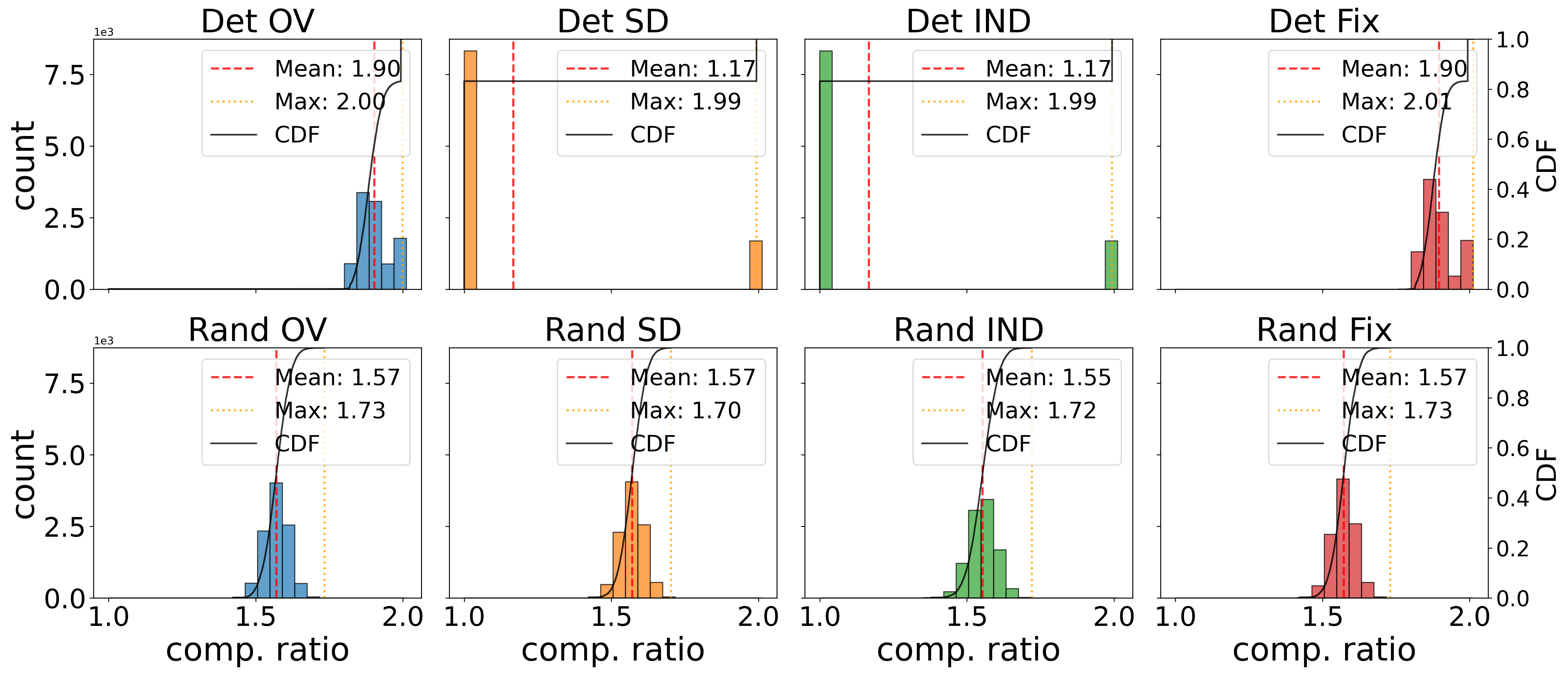}
                \caption{For state-dependent competitive ratio \(\sdcr\)}
        \end{subfigure}
        \\
        \begin{subfigure}{0.945\textwidth}
                \centering
                \includegraphics[width=\linewidth]{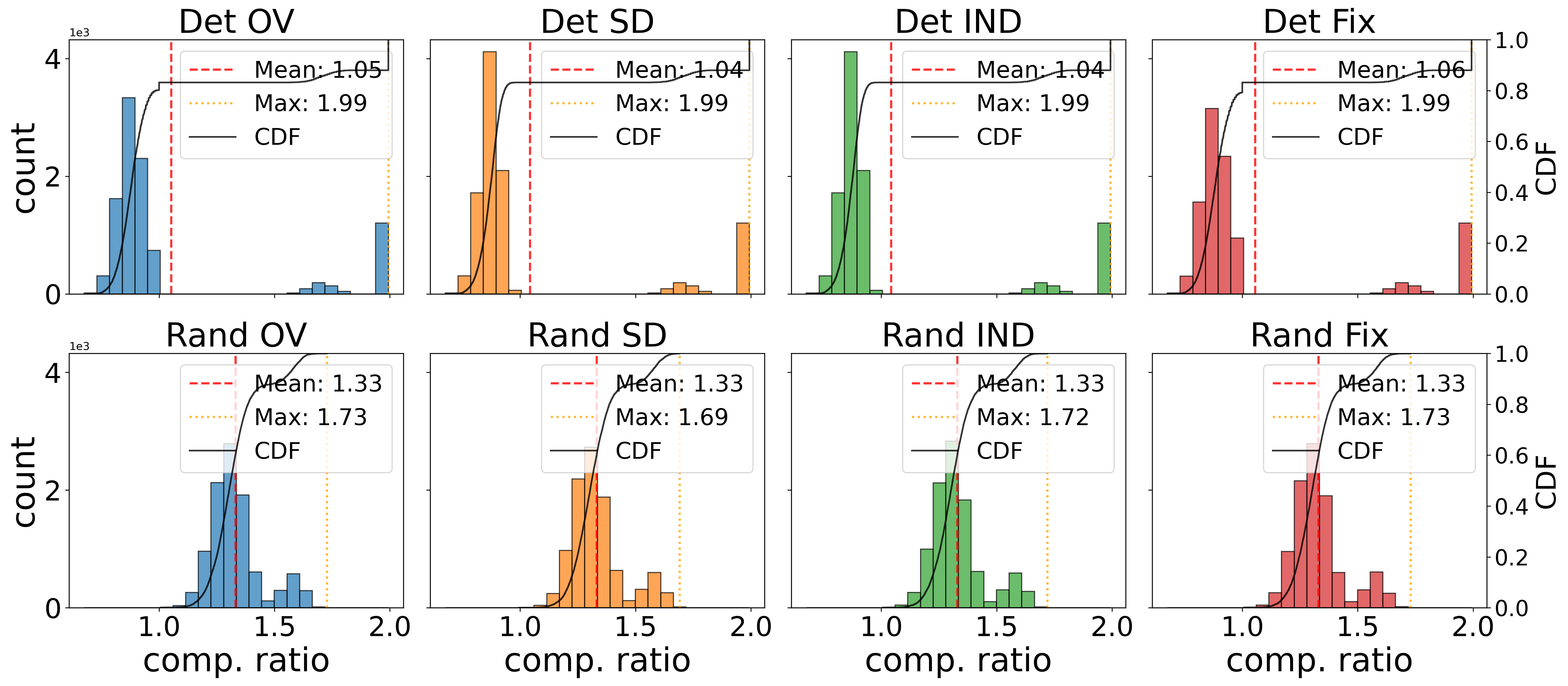}
                \caption{For minimum individual competitive ratio \(\indcr\)}
        \end{subfigure}
        \caption{Histogram and empirical cumulative distribution functions (CDFs) of competitive ratios for normal distribution with mean \(160\) and standard deviation \(10\). ``Max'' in the legend refers to the worst-case competitive ratio, while ``Mean'' refers to average performance.}\label{fig:hist-comparison-group-mean-160-std-10}
\end{figure}

\begin{figure}[t]
        \centering
        \begin{subfigure}{0.945\textwidth}
                \centering
                \includegraphics[width=\linewidth]{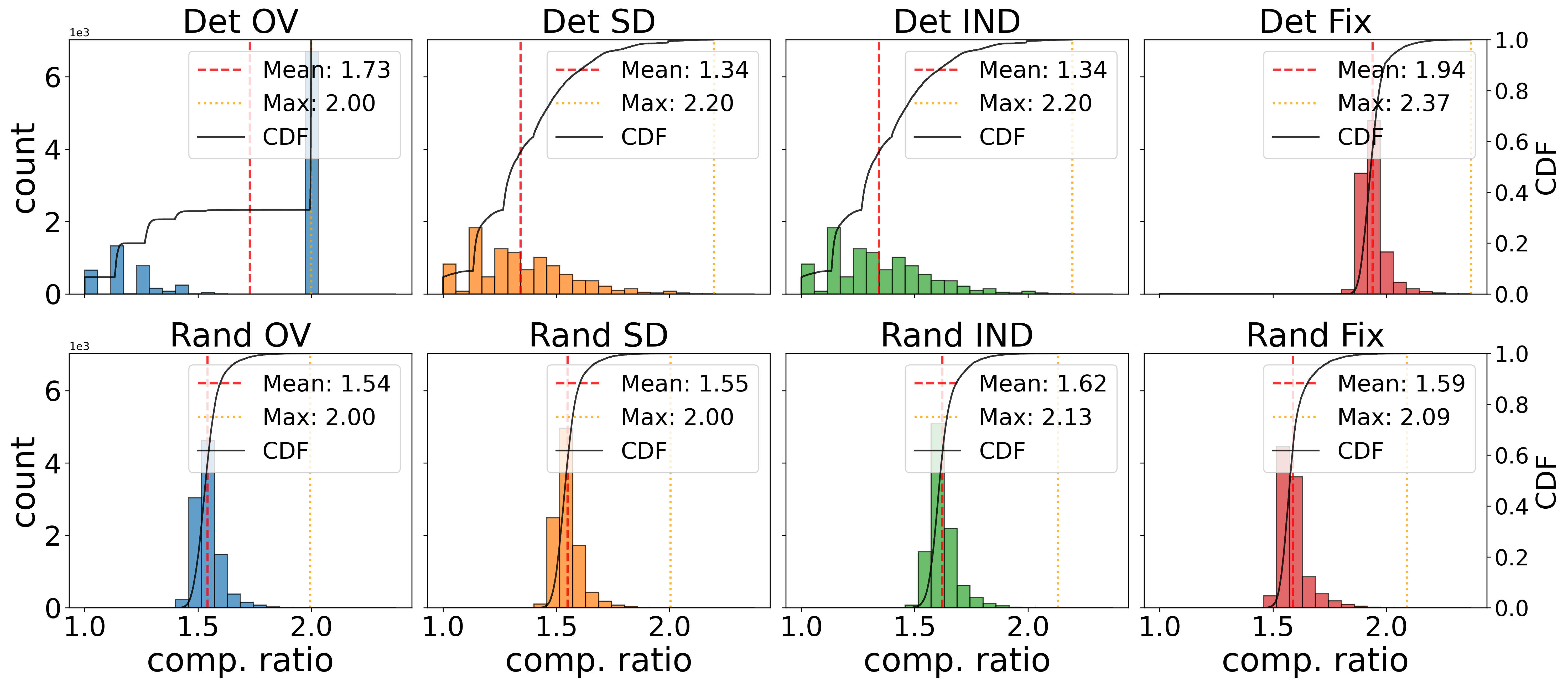}
                \caption{For overall competitive ratio \(\ovcr\)}
        \end{subfigure}
        \hspace{0.2em}
        \\
        \begin{subfigure}{0.945\textwidth}
                \centering
                \includegraphics[width=\linewidth]{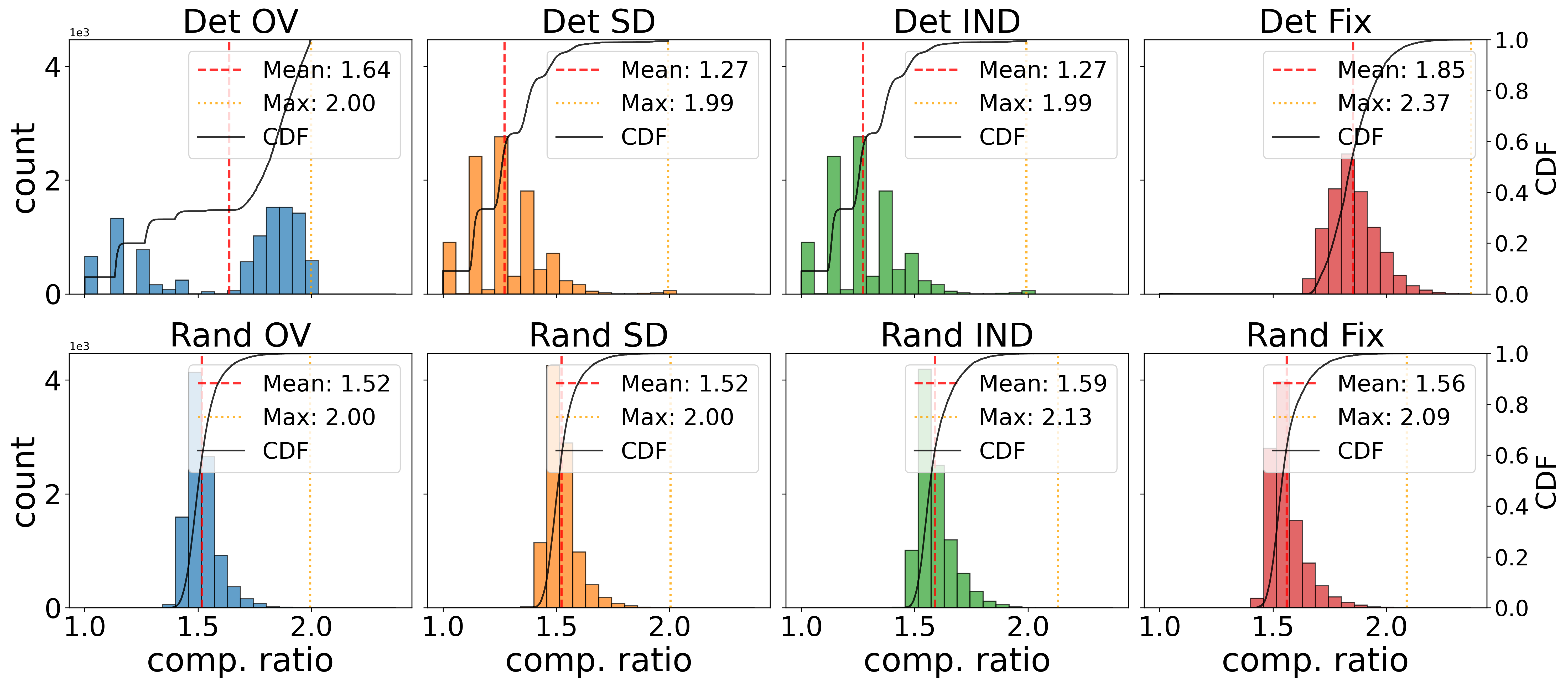}
                \caption{For state-dependent competitive ratio \(\sdcr\)}
        \end{subfigure}
        \\
        \begin{subfigure}{0.945\textwidth}
                \centering
                \includegraphics[width=\linewidth]{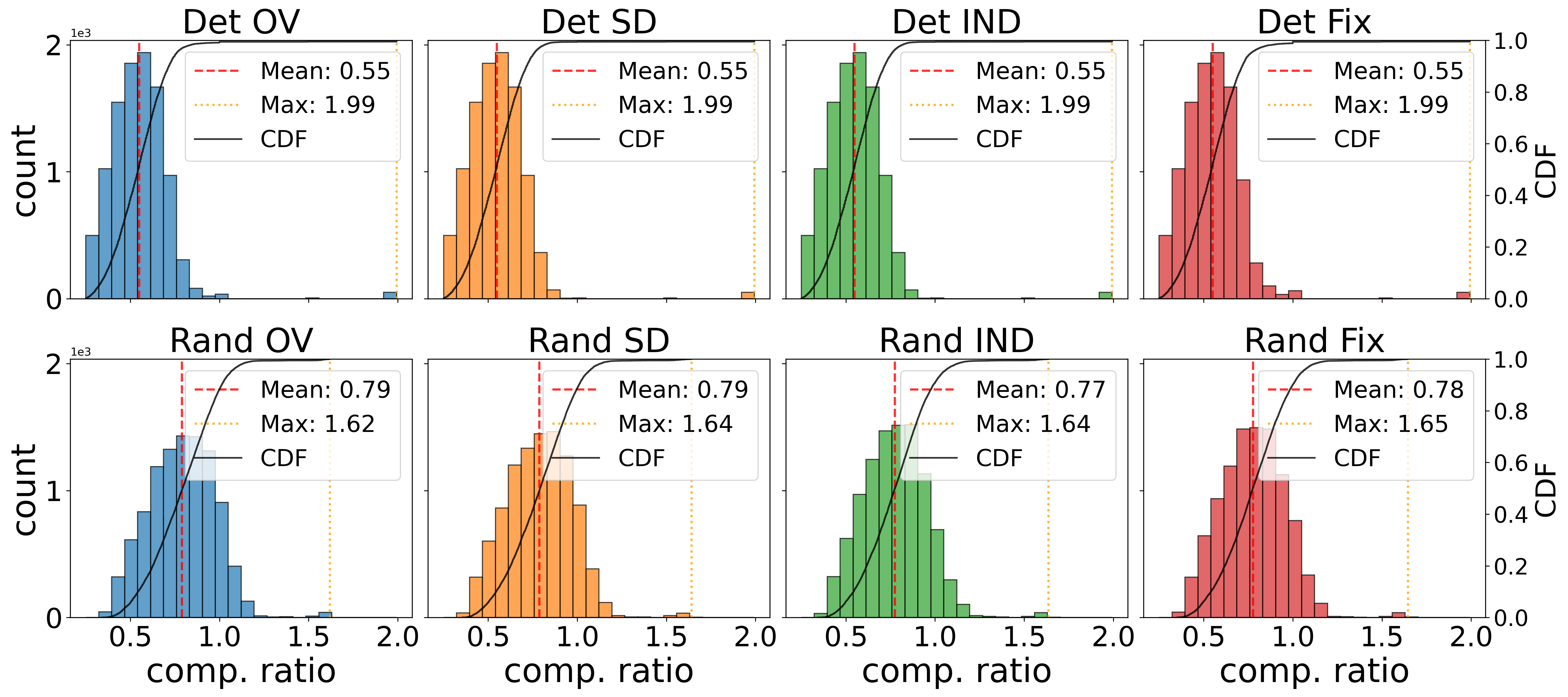}
                \caption{For minimum individual competitive ratio \(\indcr_{\min}\)}
        \end{subfigure}
        \caption{Histogram and empirical cumulative distribution functions (CDFs) of competitive ratios for normal distribution with mean \(160\) and standard deviation \(50\). ``Max'' in the legend refers to the worst-case competitive ratio, while ``Mean'' refers to average performance.}\label{fig:hist-comparison-group-mean-160-std-50}
\end{figure}

\begin{figure}[t]
        \centering
        \begin{subfigure}{0.945\textwidth}
                \centering
                \includegraphics[width=\linewidth]{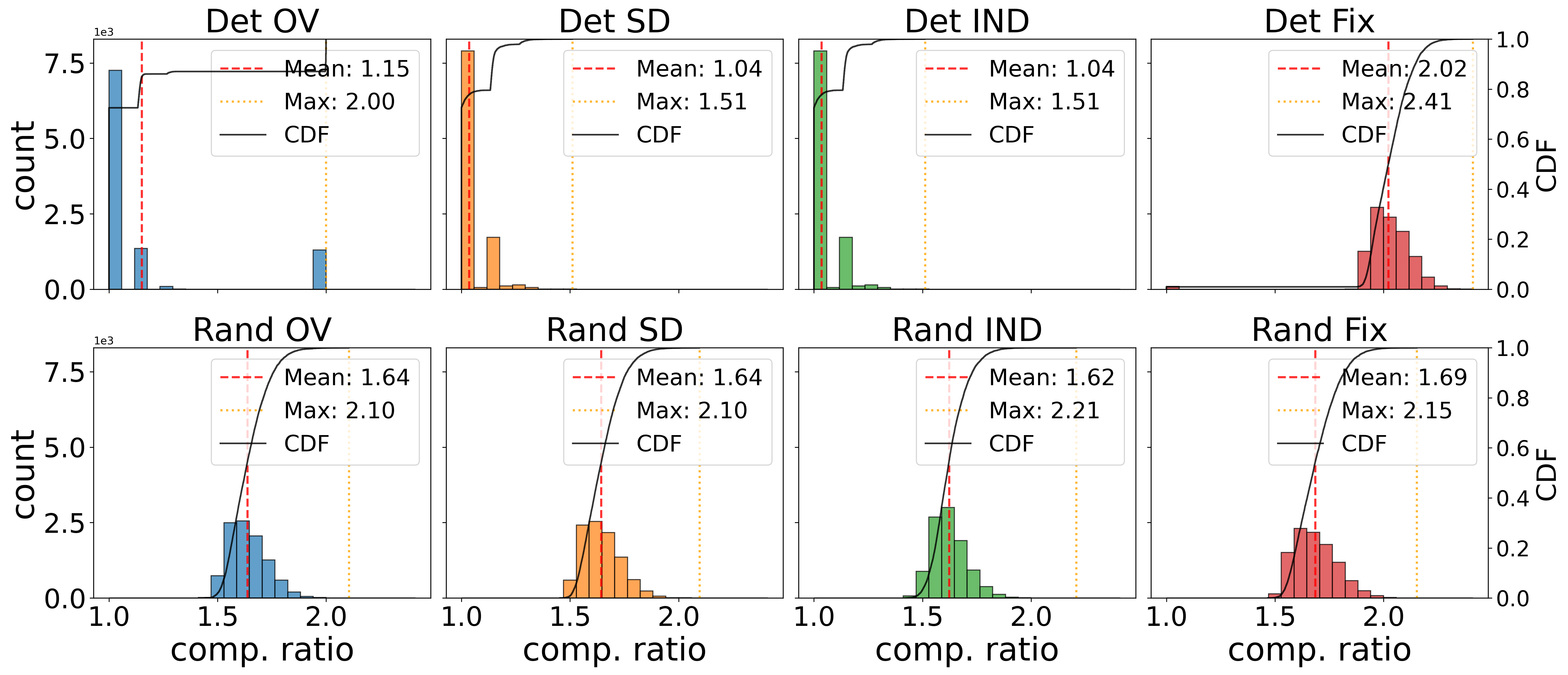}
                \caption{For overall competitive ratio \(\ovcr\)}
        \end{subfigure}
        \hspace{0.2em}
        \\
        \begin{subfigure}{0.945\textwidth}
                \centering
                \includegraphics[width=\linewidth]{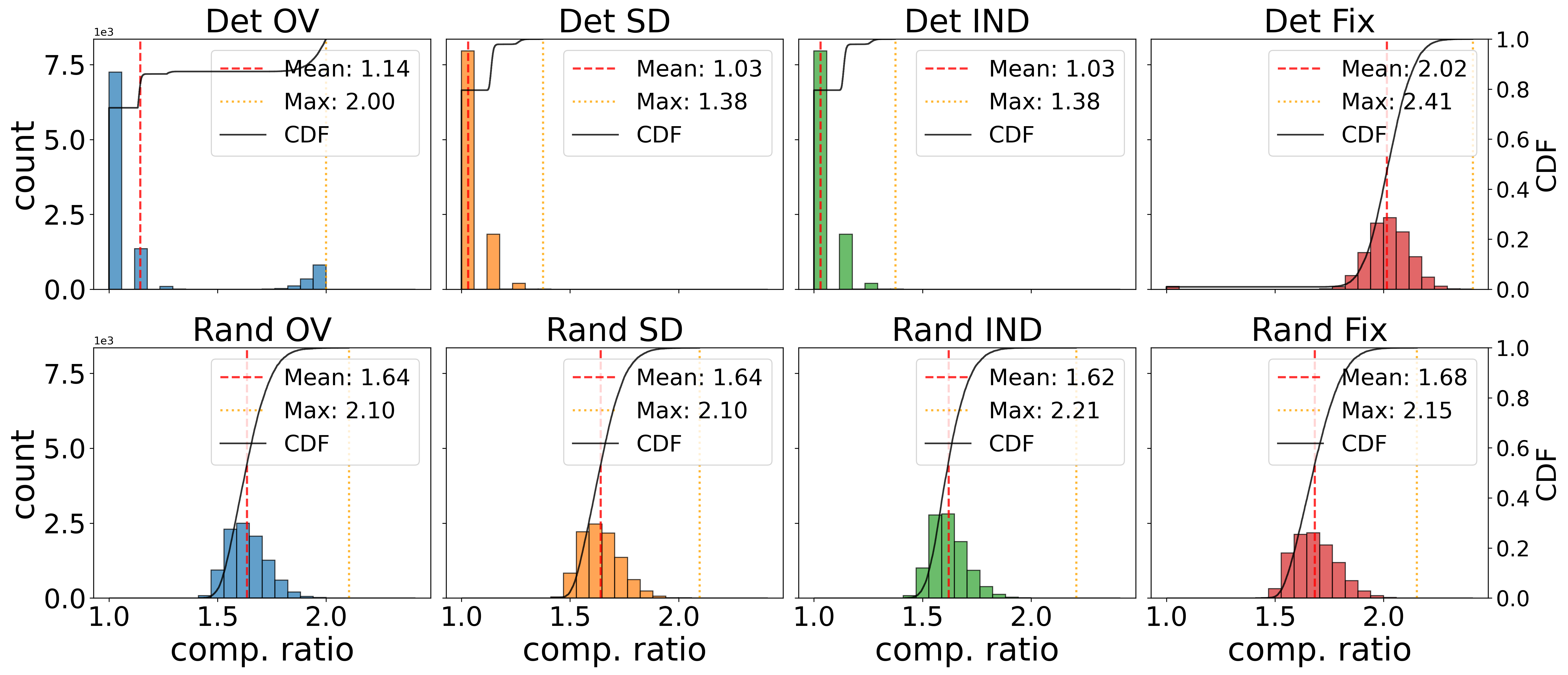}
                \caption{For state-dependent competitive ratio \(\sdcr\)}
        \end{subfigure}
        \\
        \begin{subfigure}{0.945\textwidth}
                \centering
                \includegraphics[width=\linewidth]{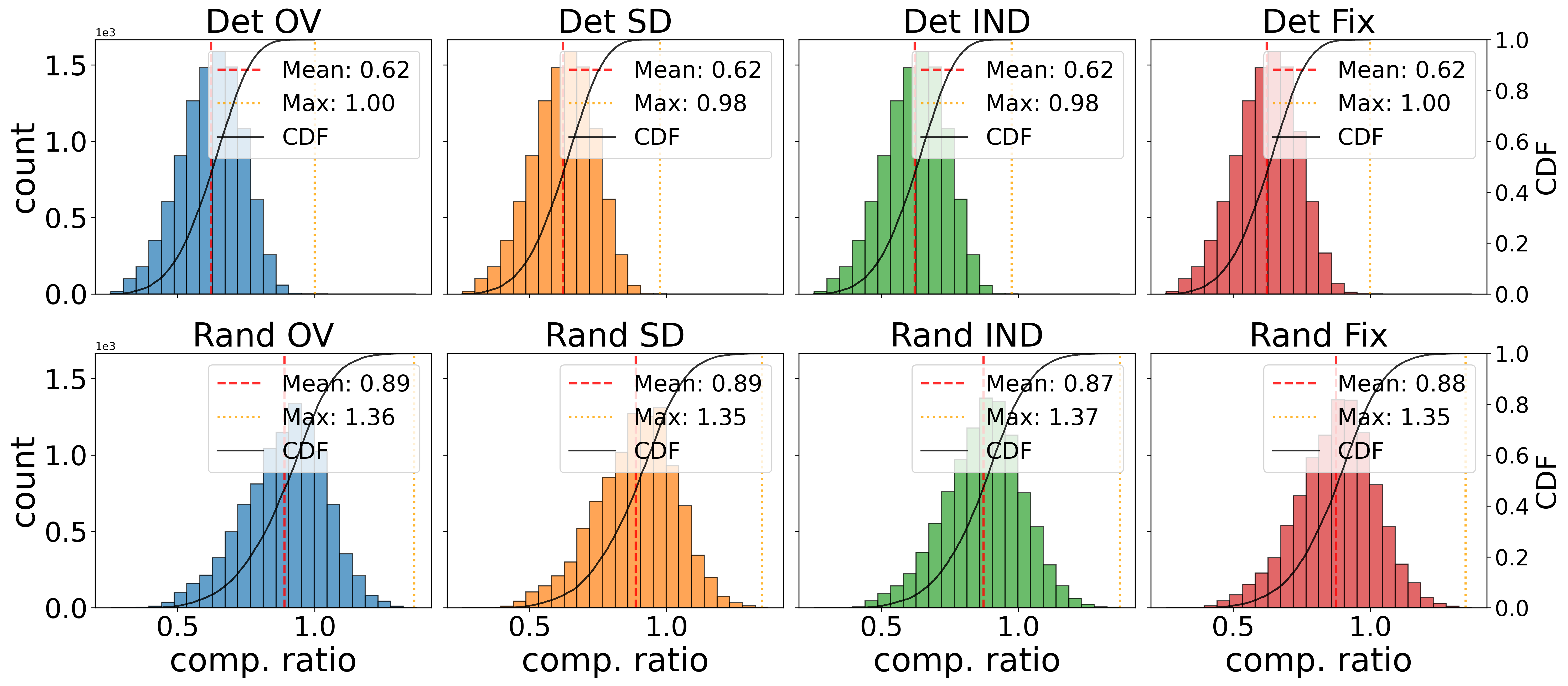}
                \caption{For minimum individual competitive ratio \(\indcr_{\min}\)}
        \end{subfigure}
        \caption{Histogram and empirical cumulative distribution functions (CDFs) of competitive ratios for normal distribution with mean \(140\) and standard deviation \(30\). ``Max'' in the legend refers to the worst-case competitive ratio, while ``Mean'' refers to average performance.}\label{fig:hist-comparison-group-mean-140-std-30}
\end{figure}

\begin{figure}[t]
        \centering
        \begin{subfigure}{0.945\textwidth}
                \centering
                \includegraphics[width=\linewidth]{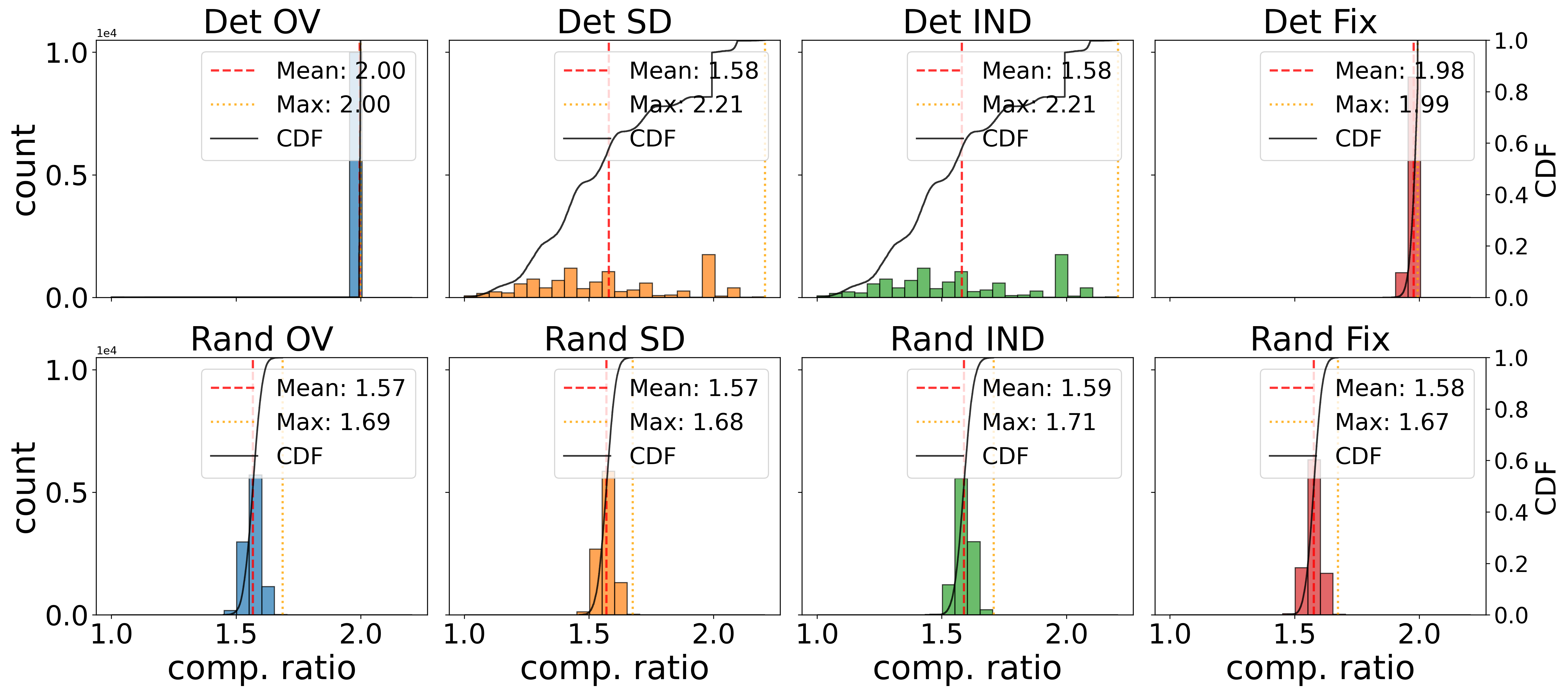}
                \caption{For overall competitive ratio \(\ovcr\)}
        \end{subfigure}
        \hspace{0.2em}
        \\
        \begin{subfigure}{0.945\textwidth}
                \centering
                \includegraphics[width=\linewidth]{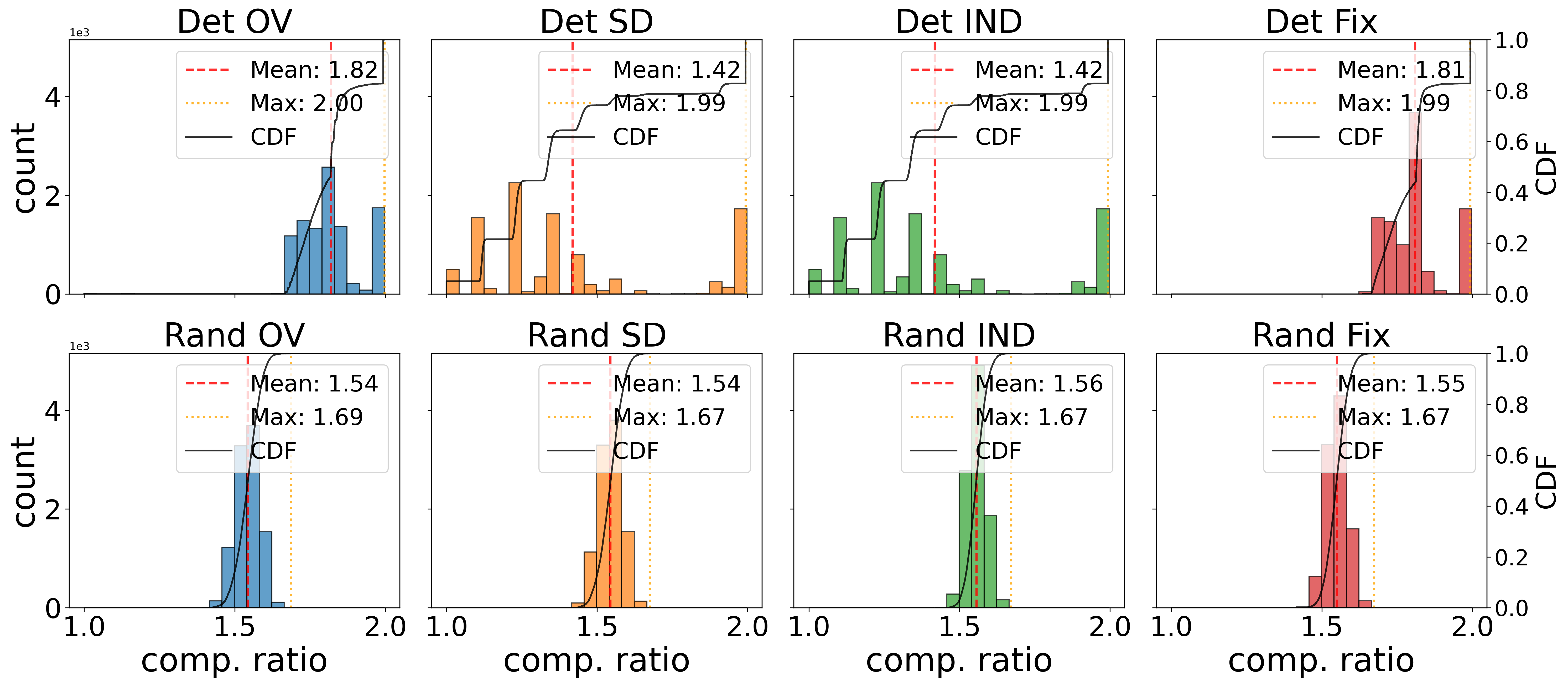}
                \caption{For state-dependent competitive ratio \(\sdcr\)}
        \end{subfigure}
        \\
        \begin{subfigure}{0.945\textwidth}
                \centering
                \includegraphics[width=\linewidth]{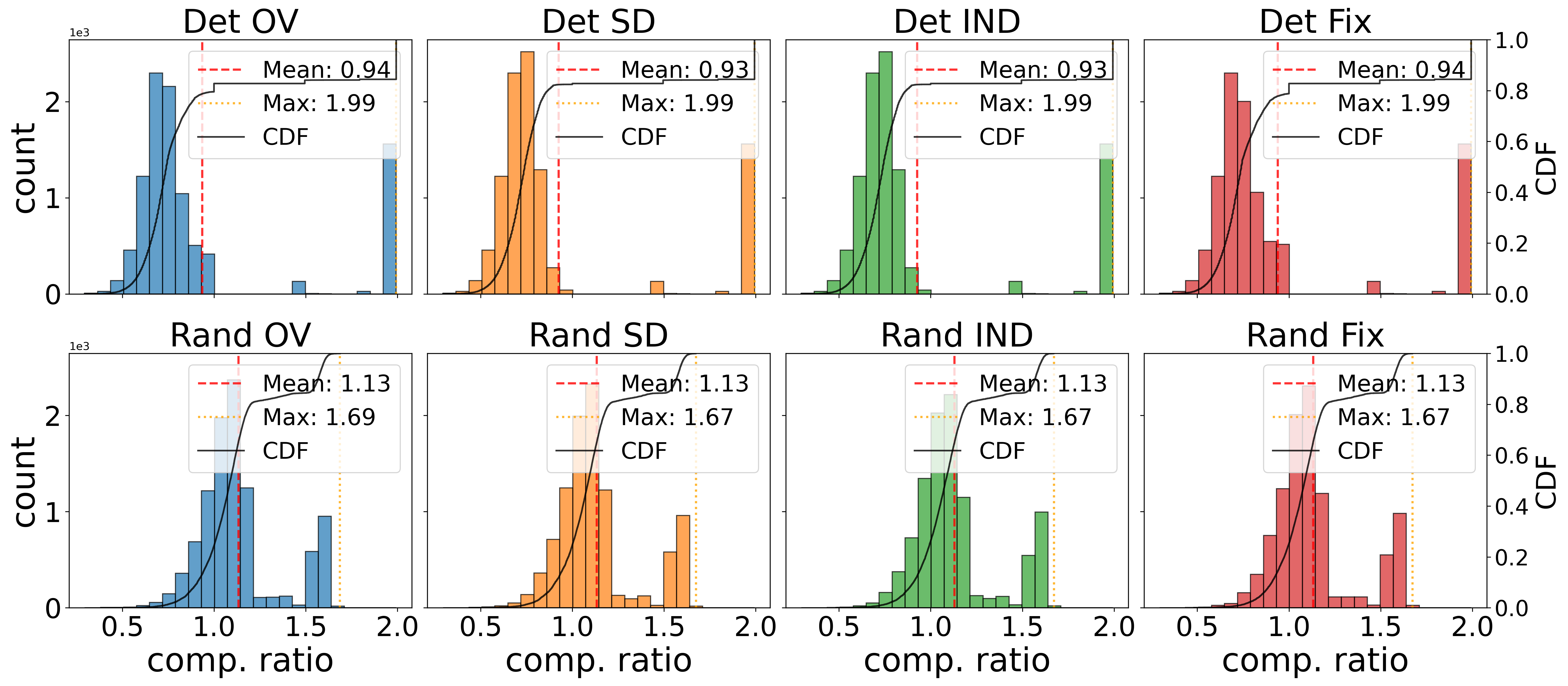}
                \caption{For minimum individual competitive ratio \(\indcr_{\min}\)}
        \end{subfigure}
        \caption{Histogram and empirical cumulative distribution functions (CDFs) of competitive ratios for normal distribution with mean \(180\) and standard deviation \(30\). ``Max'' in the legend refers to the worst-case competitive ratio, while ``Mean'' refers to average performance.}\label{fig:hist-comparison-group-mean-180-std-30}
\end{figure}

\end{document}